\documentclass{article}

\usepackage[final]{neurips_2021}

\usepackage[utf8]{inputenc} 
\usepackage[T1]{fontenc}    
\usepackage[hidelinks]{hyperref}       
\usepackage{url}            
\usepackage{booktabs}       
\usepackage{amsfonts}       
\usepackage{nicefrac}       
\usepackage{microtype}      
\usepackage{xcolor}         

\hypersetup{
    colorlinks=false,
}
\usepackage{graphicx}
\usepackage{amsmath,bm,bbm}
\usepackage{amsfonts}
\usepackage{xcolor}         
\usepackage{amsthm}        
\usepackage[algoruled,lined,linesnumbered]{algorithm2e} 
\usepackage[capitalize,nameinlink,noabbrev]{cleveref}
\usepackage[normalem]{ulem} 
\usepackage{tikz}
\usetikzlibrary{chains,arrows}
\usepackage{booktabs}
\usepackage[skip=0pt,belowskip=-2pt]{caption}
\usepackage{subcaption}
\usepackage{wrapfig}
\usepackage{chngcntr}
\usepackage{dblfloatfix}    
\usepackage{ulem}   

\newcommand{\new}[1]{\textcolor{blue}{#1}}    

\newcommand{\statespace}{\mathcal{S}}

\newcommand{\optionspace}{\mathcal{O}}
\newcommand{\dynamicsmodeltarget}{m^p}
\newcommand{\rewardmodeltarget}{m^r}
\newcommand{\durationmodeltarget}{m^l}
\newcommand{\dynamicsmodel}{M^p}
\newcommand{\rewardmodel}{M^r}
\newcommand{\durationmodel}{M^l}

\usepackage{wrapfig}

\DeclareMathOperator*{\argmax}{\arg\!\max}

\newtheorem{theorem}{Theorem}

\newtheorem{lemma}{Lemma}

\newtheorem{assumption}{Assumption}

\newcommand{\bbE}{\mathbb{E}}
\newcommand{\bbR}{\mathbb{R}}

\newcommand{\bbI}{\mathbb{I}}
\newcommand{\calS}{\mathcal{S}}
\newcommand{\calA}{\mathcal{A}}
\newcommand{\calR}{\mathcal{R}}
\newcommand{\calO}{\mathcal{O}}
\newcommand{\calI}{\mathcal{I}}
\newcommand{\calL}{\mathcal{L}}

\newcommand{\calF}{\mathcal{F}}
\newcommand{\calH}{\mathcal{H}}
\newcommand{\calM}{\mathcal{M}}

\newcommand{\norm}[1]{\left\lVert#1\right\rVert}
\newcommand{\abs}[1]{\left\lvert#1\right\rvert}
\newcommand{\cardS}{\abs{\calS}}

\newcommand{\cardO}{\abs{\calO}}

\usepackage{xargs}

\usepackage{xspace}

\title{Average-Reward Learning and Planning with Options}

\author{%
  Yi Wan$^\dagger$, Abhishek Naik$^\dagger$, Richard S. Sutton$^{\dagger\ddagger}$\\
  \texttt{\{wan6,anaik1,rsutton\}@ualberta.ca}
}

\begin{document}

\maketitle

\vspace{-2.8\baselineskip}
\begin{center}
  \begin{tabular}{c@{\qquad}c}
    \llap{$^\dagger$}University of Alberta, Amii
    & \llap{$^\ddagger$}DeepMind\\
    Edmonton, Canada
    & Edmonton, Canada
  \end{tabular}
\end{center}
\vspace{1\baselineskip}

\begin{abstract}
We extend the options framework for temporal abstraction in reinforcement learning from discounted Markov decision processes (MDPs) to average-reward MDPs.
Our contributions include general convergent off-policy inter-option learning algorithms, intra-option algorithms for learning values and models, as well as sample-based planning variants of our learning algorithms.
Our algorithms and convergence proofs extend those recently developed by Wan, Naik, and Sutton.  
We also extend the notion of option-interrupting behavior from the discounted to the average-reward formulation.
We show the efficacy of the proposed algorithms with experiments on a continuing version of the Four-Room domain. 

\end{abstract}

\section{Introduction}

Reinforcement learning (RL) is a formalism of trial-and-error learning in which an agent interacts with an environment to learn a behavioral strategy that maximizes a notion of reward. 
In many problems of interest, a learning agent may need to predict the consequences of its actions over multiple levels of temporal abstraction. 
The \textit{options} framework provides a way for defining courses of actions over extended time scales, and for learning, planning, and representing knowledge with them (Sutton, Precup, \& Singh 1999, 
Sutton \& Barto 2018). 
The options framework was originally proposed within the \textit{discounted} formulation of RL in which the agent tries to maximize the expected discounted return from each state. 
We extend the options framework from the discounted formulation to the \textit{average-reward} formulation in which the goal is to find a policy that maximizes the rate of reward.

The average-reward formulation is of interest because, once genuine function approximation is introduced, there is no longer a well-defined discounted formulation of the continuing RL problem (see Sutton \& Barto 2018, Section 10.4; Naik et al. 2019).
If we want to take advantage of options in acting, learning, and planning in the continuing (non-episodic) RL setting, then we must extend options to the average-reward formulation.

Given a Markov decision process (MDP) and a fixed set of options, learning and planning algorithms can be divided into two classes.
The first class consists of \textit{inter}-option algorithms, which enable an agent to learn or plan with options instead of primitive actions. 
Given an option, the learning and planning updates for this option in these algorithms occur only \textit{after} the option's actual or simulated execution. 
Algorithms in this class are also called semi-MDP (SMDP) algorithms because given an MDP, the decision process that selects among a set of options, executing each to termination, is an SMDP (Sutton et al. 1999).
The second class consists of algorithms in which learning or planning updates occur after each state-action transition \textit{within} options' execution --- these are called \textit{intra}-option algorithms. From a single state-action transition, these algorithms can learn or plan to improve the values or policies for \emph{all} options that may generate that transition, and are therefore potentially more efficient than SMDP algorithms.

Several inter-option (SMDP) learning algorithms have been proposed for the average-reward formulation (see, e.g., Das et al.~1999, Gosavi 2004, Vien \& Chung 2008). 
To the best of our knowledge, Gosavi's (2004) algorithm is the only proven-convergent \textit{off-policy} inter-option learning algorithm. 
However, its convergence proof requires the underlying SMDP to have a special state that is recurrent under all stationary policies. 
Recently, Wan, Naik, and Sutton (2021) proposed Differential Q-learning, an off-policy control learning algorithm for average-reward MDPs that is proved to converge without requiring any special state. 
We extend this algorithm and its convergence proof from primitive actions to options and highlight some challenges we faced in developing \textit{inter-option Differential Q-learning}.
For planning, we propose \textit{inter-option Differential Q-planning}, which is the first convergent \textit{incremental} (sampled-based) planning algorithm. 
The existing proven-convergent inter-option planning algorithms (e.g., Schweitzer 1971, Puterman 1994, Li \& Cao 2010) are not incremental because they perform a full sweep over states for each planning step. 

Additionally, the literature lacks intra-option learning and planning algorithms within the average-reward formulation for both values and models.
We fill this gap by proposing such algorithms in the average-reward formulation and provide their convergence results. 
These algorithms are stochastic approximation algorithms solving the average-reward intra-option value and model equations, which are also introduced in this paper for the first time.

Sutton et al.~(1999) also introduced an algorithm to improve an agent's behavior given estimated option values. 
Instead of letting an option execute to termination, this algorithm involves potentially interrupting an option's execution to check if starting a new option might yield a better expected outcome. 
If so, then the currently-executing option is terminated, and the new option is executed. 
Our final contribution involves extending this notion of an \textit{interruption} algorithm from the discounted to the average-reward formulation.

\section{Problem Setting}
\label{sec: background}

We formalize an agent's interaction with its environment by a finite Markov decision process (MDP) $\calM$ and a finite set of options $\calO$. The MDP is defined by the tuple $\calM \doteq (\calS, \calA, \calR, p)$, where $\calS$ is a set of states, $\calA$ is a set of actions, $\calR$ is a set of rewards, and $p : \calS \times \calR \times \calS \times \calA \to [0, 1]$ is the dynamics of the environment. 
Each option $o$ in $\calO$ has two components: the option's \textit{policy} $\pi^o: \calA \times \calS \to [0, 1]$, and a probability distribution of the option's \textit{termination} $\beta^o: \calS \to [0, 1]$.
For simplicity, for any $s \in \calS, o \in \calO$, we use $\pi(a \mid s, o)$ to denote $\pi^o(a, s)$ and $\beta(s, o)$ to denote $\beta^o(s)$. Sutton et al.'s (1999) options additionally have an \textit{initiation} set that consists of the states at which the option can be initiated. To simplify the presentation in this paper, we allow all options to be initiated in all states of the state space; the algorithms and theoretical results can easily be extended to incorporate initiation from specific states. 

In the continuing (non-episodic) setting, the agent-environment interactions go on forever without any resets.
If an option $o$ is initiated at time $t$ in state $S_t$, then the action $A_t$ is chosen according to the option's policy $\pi(\cdot \mid S_t, o)$. The agent then observes the next state $S_{t+1}$ and reward $R_{t+1}$ according to $p$. 
The option terminates at $S_{t+1}$ with probability $\beta(S_{t+1}, o)$ or continues with action $A_{t+1}$ chosen according to $\pi(\cdot \mid S_{t+1}, o)$. It then possibly terminates in $S_{t+2}$ according to $\beta(S_{t+2}, o)$, and so on. 
At an option termination, one way to govern an agent's behavior is to choose a new option according to a hierarchical policy
$\mu_b : \mathcal{S} \times \mathcal{O} \mapsto [0, 1]$. In this case, when an option terminates at time $t$, the next option is selected stochastically according to $\mu_b (\cdot | S_t)$. 
The option initiates at $S_t$ and terminates at $S_{t+K}$, where $K$ is a random variable denoting the number of time steps the option executed. 
At $S_{t+K}$, a new option is again chosen according to $\mu_b (\cdot | S_{t+K})$, and so on. 
We use the notation $O_t$ to denote whatever option is being executed at time step $t$. 
Note that $O_t$ will remain the same for as many steps as the option executes.
Also note that actions are a special case of options: every action $a$ is an option $o$ that terminates after exactly one step ($\beta(s,o)=1,\ \forall s$) and whose policy is to pick $a$ in every state ($\pi(a \mid s,o)=1,\ \forall s$).

Let $T_n$ denote the time step when the ${n-1}\textsuperscript{th}$ option terminates and the $n\textsuperscript{th}$ option is chosen. 
Denote the $n\textsuperscript{th}$ option by $\hat O_n \doteq O_{T_n}$, its starting state by $\hat S_n \doteq S_{T_n}$, the cumulative reward during its execution by $\hat R_n \doteq \sum_{t = T_n + 1}^{T_{n+1}} R_{t}$, the state it terminates in by $\hat S_{n+1} \doteq S_{T_{n+1}}$, and its length by $\hat L_n \doteq T_{n+1} - T_n$. 
Note that every option's length is a random variable taking values among positive integers. 
The option's transition probability is then defined as $\hat p(s', r, l \mid s, o) \doteq \Pr(\hat S_{n+1} = s', \hat R_{n} = r, \hat L_n = l \mid \hat S_{n} = s, \hat O_{n} = o)$. 
Throughout the paper, we assume that the expected execution time of every option starting from any state is finite.

An MDP $\calM$ and a set of options $\calO$ results in an SMDP $\hat \calM = (\calS, \calO, \hat \calL, \hat \calR, \hat p)$, where $\hat \calL$ is the set of all possible lengths of options and $\hat \calR$ is the set of all possible options' cumulative rewards. 
For this SMDP, the \textit{reward rate} of a policy $\mu$ given a starting state $s$ and option $o$ can be defined as $r^C(\mu)(s, o) \doteq \lim_{t \to \infty} \bbE_\mu [\sum_{i = 1}^t R_i \mid S_0 = s, O_0 = o ] / t$. 
Alternatively, at the level of option transitions, $r(\mu)(s, o) \doteq \lim_{n \to \infty} \bbE_\mu [\sum_{i = 0}^n \hat R_i \mid \hat S_0 = s, \hat O_0 = o ] / \bbE_\mu [\sum_{i = 0}^n \hat L_i \mid \hat S_0 = s, \hat O_0 = o ]$. Both the limits exist and are equivalent (Puterman's (1994) propositions 11.4.1 and 11.4.7) under the following assumption: 
\begin{assumption}\label{assu: unichain}
The Markov chain induced by any stationary policy in the MDP $(\calS, \calO, \hat \calR, p')$ is unichain, where $p'(s', r \mid s, o) \doteq \sum_{l} \hat p (s', r,\,l \mid s, o)\ \forall\ s', r, s, o$.
\end{assumption}

\textbf{Note:} In a unichain MDP, there could be some states that only occur a finite number of times in a single stream of experience. In other words, these states are \textit{transient} under all stationary policies. 
Thus, their values can not be correctly estimated by \textit{any} learning algorithm. 
However, this inaccurate value estimation is not a problem because the decisions made in these transient states do not affect the reward rate. 
We refer to the non-transient states as \emph{recurrent} states and denote their set by $\calS' \subseteq \calS$.

Under Assumption~\ref{assu: unichain}, the reward rate does not depend on the starting state-option pair and hence we can denote it by just $r(\mu)$. 
The optimal reward rate can then be defined as $r_* \doteq \sup_{\mu \in \Pi} r(\mu)$,
where $\Pi$ denotes the set of all policies.
The differential option-value function for a policy $\mu$ is defined for all $s \in \calS, o \in \calO$ as $q_\mu(s, o) \doteq \bbE_\mu[R_{t+1} - r(\mu) + R_{t+2} - r(\mu) + \cdots \mid S_t = s, O_t = o\,].$
The \textit{evaluation} and \emph{optimality} equations for SMDPs, as given by Puterman (1994), are: 
\begin{align}
    q(s, o) & = \sum_{s', r,\,l} \hat p(s', r,\,l \mid s, o) \big( r - \bar r \cdot l +  \sum_{o'} \mu(o' | s') q (s', o') \big), \label{eq: SMDP Bellman evaluation equation} \\
    q(s, o) & = \sum_{s', r,\,l} \hat p(s', r,\,l \mid s, o) \big( r - \bar r \cdot l  + \max_{o'} q (s', o') \big), \label{eq: SMDP Bellman optimality equation}
\end{align}
where $q$ and $\bar{r}$ denote estimates of the option-value function and the reward rate respectively. If Assumption~\ref{assu: unichain} holds, the SMDP Bellman equations have a unique solution for $\bar r$ — $r(\mu)$ for evaluation and $r_*$ for control — and a unique solution for $q$ only up to a constant (Schweitzer \& Federgruen 1978).
Given an MDP and a set of options, the goal of the \textit{prediction} problem is, for a given policy $\mu$, to find the reward rate $r(\mu)$ and the differential value function (possibly with some constant offset). 
The goal of the \textit{control} problem is to find a policy that achieves the optimal reward rate $r_*$. 


\section{Inter-Option Learning and Planning Algorithms} \label{sec: inter-option}

In this section, we present our inter-option learning and planning, prediction and control algorithms, which extend Wan et al.'s (2021) differential learning and planning algorithms for average-reward MDPs from actions to options. 
We begin with the control learning algorithm and then move on to the prediction and planning algorithms.

Consider Wan et al.'s (2021) control learning algorithm:
\begin{align*}
    Q_{t+1}(S_t, A_t) \doteq Q_{t}(S_t, A_t) + \alpha_{t} \delta_{t}, \quad \bar R_{t+1} \doteq \bar R_{t} + \eta \alpha_{t} \delta_{t},
\end{align*}
where $Q$ is a vector of size $\abs{\calS \times \calA}$ that approximates a solution of $q$ in the Bellman optimality equation for MDPs, $\bar R$ is a scalar estimate of the optimal reward rate, $\alpha_t$ is a step-size sequence, $\eta$ is a positive constant, and $\delta_t$ is the temporal-difference (TD) error: $\delta_t \doteq R_t - \bar R_t + \max_a Q_t (S_{t+1}, a) - Q_t(S_{t}, A_t).$ The most straightforward inter-option extension of Differential Q-learning is:
\begin{align}
    Q_{n+1}(\hat S_n, \hat O_n) & \doteq Q_{n}(\hat S_n, \hat O_n) + \alpha_{n} \delta_{n}, \label{eq: possible extension of Diff Q} \\
    \bar R_{n+1} & \doteq \bar R_{n} + \eta \alpha_{n} \delta_{n}, \label{eq: possible extension of Diff bar R}
\end{align}
where $Q$ is a vector of size $\abs{\calS \times \calO}$ that approximates a solution of $q$ in \eqref{eq: SMDP Bellman optimality equation}, $\bar R$ is a scalar estimate of $r_*$, $\alpha_n$ is \emph{a} step-size sequence, and $\delta_n$ is the TD error:
\begin{align} \label{eq: Inter-option Differential Q-learning attempt TD error}
    \delta_n & \doteq \hat R_n - \hat L_n \bar R_n + \max_o Q_n (\hat S_{n+1}, o) - Q_n(\hat S_{n}, \hat O_n).
\end{align}
Such an algorithm is prone to instability because the \textit{sampled} option length $\hat L_{n}$ can be quite large, and any error in the reward-rate estimate $\bar R_n$ gets multiplied with the potentially-large option length. 
Using small step sizes might make the updates relatively stable, but at the cost of slowing down learning for options of shorter lengths. 
This could make the choice of step size quite critical, especially when the range of the options' lengths is large and unknown. 
Alternatively, inspired by Schweitzer (1971), we propose scaling the updates by the \textit{estimated} length of the option being executed:
\begin{align}
    Q_{n+1}(\hat S_n, \hat O_n) & \doteq Q_{n}(\hat S_n, \hat O_n) + \alpha_{n} \delta_{n} / L_n(\hat S_n, \hat O_n), \label{eq: Inter-option Differential TD-learning Q}\\
    \bar R_{n+1} & \doteq \bar R_{n} + \eta \alpha_{n} \delta_{n}/L_n(\hat S_n, \hat O_n), \label{eq: Inter-option Differential TD-learning R bar}
\end{align}
where $\alpha_n$ is a step-size sequence, $L_n(\cdot,\cdot)$ comes from an additional vector of estimates $L: \calS \times \calO \to \bbR$ that approximates the expected lengths of state-option pairs, updated from experience by:
\begin{align}
    L_{n+1}(\hat S_n, \hat O_n) \doteq L_{n}(\hat S_n, \hat O_n) + \beta_n (\hat L_n - L_{n}(\hat S_n, \hat O_n)), \label{eq: Inter-option Differential TD-learning L}
\end{align}
where $\beta_n$ is an another step-size sequence. 
The TD-error $\delta_n$ in \eqref{eq: Inter-option Differential TD-learning Q} and \eqref{eq: Inter-option Differential TD-learning R bar} is
\begin{align} \label{eq: Inter-option Differential Q-learning TD error}
    \delta_n & \doteq \hat R_n - L_n(\hat S_n, \hat O_n) \bar R_n + \max_o Q_n (\hat S_{n+1}, o) - Q_n(\hat S_{n}, \hat O_n),
\end{align}
which is different from \eqref{eq: Inter-option Differential Q-learning attempt TD error} with the estimated expected option length $L_n(\hat S_n, \hat O_n)$ being used instead of the sampled option length $\hat L_n$.
(\ref{eq: Inter-option Differential TD-learning Q}–\ref{eq: Inter-option Differential Q-learning TD error}) make up our \textit{inter-option Differential Q-learning} algorithm. 

Similarly, our prediction learning algorithm, called \textit{inter-option Differential Q-evaluation}, 
also has update rules (\ref{eq: Inter-option Differential TD-learning Q}–\ref{eq: Inter-option Differential TD-learning L}) with the TD error:
\begin{align} \label{eq: Inter-option Differential TD-learning TD error}
    \delta_n \doteq \hat R_n - L_n(\hat S_n, \hat O_n) \bar R_n + \sum_o \mu(o \mid \hat S_{n+1}) Q_n (\hat S_{n+1}, o) - Q_n(\hat S_{n}, \hat O_n).
\end{align}

\begin{theorem}[Convergence of inter-option algorithms; informal]\label{thm: Inter-option Differential Methods}
If Assumption~\ref{assu: unichain} holds, step sizes are decreased appropriately, all state-option pairs $(s, o)$ in $ \calS'$ and $\calO$ are visited for an infinite number of times, and the relative visitation frequency between any two pairs is finite: 
\vspace{-2mm}
\begin{enumerate}\itemsep-1mm
    \item inter-option Differential Q-learning (\ref{eq: Inter-option Differential TD-learning Q}–\ref{eq: Inter-option Differential Q-learning TD error}) converges almost surely, $\bar R_n$ to $r_*$ and $Q_n(s, o)$ to a solution of $q(s, o)$ in (\ref{eq: SMDP Bellman optimality equation}) for all $s \in \calS', o \in \calO$, and $r(\mu_n)$ to $r_*$, where $\mu_n$ is a greedy policy w.r.t.\ $Q_n$, 
    \item inter-option Differential Q-evaluation (\ref{eq: Inter-option Differential TD-learning Q}–\ref{eq: Inter-option Differential TD-learning L}, \ref{eq: Inter-option Differential TD-learning TD error}) converges almost surely, $\bar R_n$ to $r(\mu)$ and $Q_n(s, o)$ to a solution of of $q(s, o)$ in (\ref{eq: SMDP Bellman evaluation equation}) for all $s \in \calS', o \in \calO$.
\end{enumerate}
\end{theorem}

\begin{wrapfigure}{r}{.28\textwidth}
    \vspace{-5mm}
    \centering
    \includegraphics[width=0.28\textwidth]{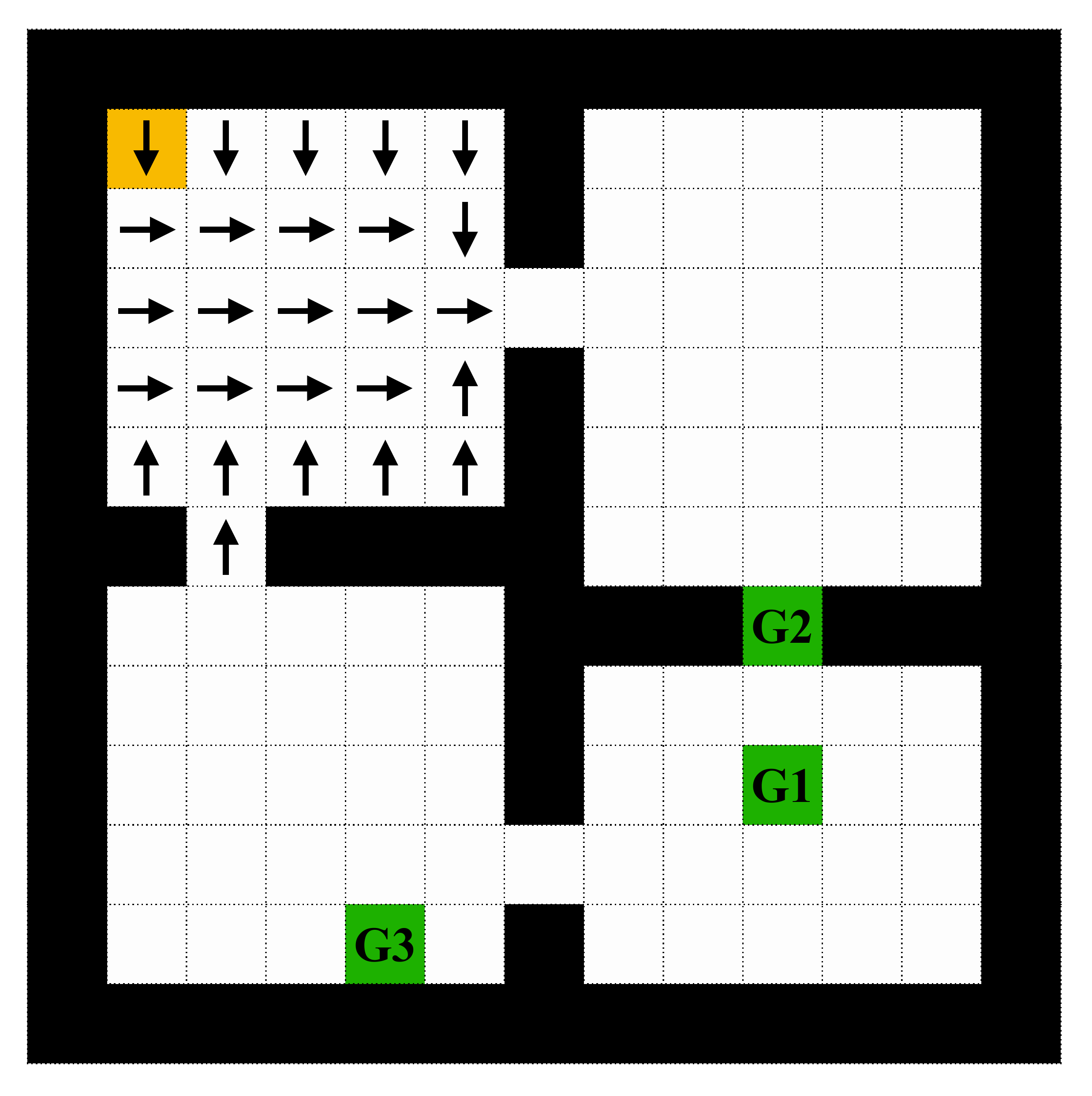}
    \caption{A continuing variant of the Four-Room domain where the task is to repeatedly go from the yellow start state to one of the three green goal states. There is one goal state per experiment, chosen to demonstrate particular aspects of the proposed algorithms. Also shown is an option policy to go to the upper hallway cell; more details in-text.}
    \label{fig: four rooms}
    \vspace{-8mm}
\end{wrapfigure}

The convergence proofs for the inter-option (as well as the subsequent intra-option) algorithms are based on a result that generalizes Wan et al.'s (2021) and Abounadi et al.'s (2001) proof techniques from primitive actions to options. We present this result in \cref{app: general rvi q}; the formal theorem statements and proofs in \cref{app: inter option}. 

\textbf{Remark}: It is important for the scaling factor in the algorithm to be the expected option length $L_n(\hat S_n, \hat O_n)$ and not the sampled option length $\hat{L}_n$. 
Scaling the updates by the expected option lengths ensures that fixed points of the updates are the solutions of \eqref{eq: SMDP Bellman optimality equation}. 
This is not guaranteed to be true when using the sampled option length. 
We discuss this in more detail in \cref{app: additional discussion: Other Attempts on Extending Differential Q-learning}.

The inter-option planning algorithms for prediction and control are similar to the learning algorithms except that they use simulated experience generated by a (given or learned) model instead of real experience.
In addition, they only have two update rules, \eqref{eq: Inter-option Differential TD-learning Q} and \eqref{eq: Inter-option Differential TD-learning R bar}, not \eqref{eq: Inter-option Differential TD-learning L}, because the model provides the expected length of a given option from a given state (see \cref{sec: intra-option model} for a complete specification of option models). 
The planning algorithms and their convergence results are presented in \cref{app: inter option}.


\textbf{Empirical Evaluation.} We tested our inter-option Differential Q-learning with Gosavi's (2004) algorithm as a baseline in a variant of Sutton et al.'s (1999) Four-Room domain (shown in \cref{fig: four rooms}). 
The agent starts in the yellow cell.
The goal states are indicated by green cells. Every experiment in this paper uses only one of the green cells as a goal state; the other two are considered as empty cells.
There are four primitive actions of moving \texttt{up}, \texttt{down}, \texttt{left}, \texttt{right}. 
The agent receives a reward of +1 when it moves into the goal cell, 0 otherwise.

In addition to the four primitive actions, the agent has eight options that take it from a given room to the hallways adjoining the room. 
The arrows in \cref{fig: four rooms} illustrate the policy of one of the eight options. 
For this option, the policy in the empty cells (not marked with arrows) is to uniformly-randomly pick among the four primitive actions. 
The termination probability is 0 for all the cells with arrows and 1 for the empty cells. 
The other seven options are defined in a similar way. 
Denote the set of primitive actions as $\calA$ and the set of hallway options as $\calH$. 
Including the primitive actions, the agent has 
12 options in total. 

In the first experiment, we tested inter-option Differential Q-learning with three different sets of options, $\calO \in \{\calA, \calH, \calA + \calH\}$. 
The task was to reach the green cell \texttt{G1}, which the agent can achieve with a combination of options and primitive actions. 
The shortest path to \texttt{G1} from the starting state takes 16 time steps, hence the best possible reward rate for this task is 1/16 $=$ 0.0625. 
The agent used an $\epsilon$-greedy policy with $\epsilon = 0.1$. 
For each of the two step-sizes $\alpha_n$ and $\beta_n$, we tested five choices: $2^{-x}, x \in \{1,3,5,7,9\}$.
In addition, we tested five choices of $\eta: 10^{-x}, x \in \{0,1,2,3,4\}$.
$Q$ and $\bar R$ were initialized to 0, $L$ to 1.
Each parameter setting was run for 200,000 steps and repeated 30 times.
The left subfigure of \cref{fig: learning curves and sensitivity curves for inter-option learning} shows a typical learning curve for each of the three sets of options, with $\alpha=2^{-3}$, $\beta=2^{-1}$, and $\eta = 10^{-1}$. 
The parameter study for $\calO = \calA + \calH$ w.r.t.\ $\alpha$ and $\eta$, with $\beta=2^{-1}$, is presented in the right subfigure of \cref{fig: learning curves and sensitivity curves for inter-option learning}. The metric is the average reward obtained over the entire training period.
Complete parameter studies for 
all the three sets of options 
are presented in \cref{app: additional empirical results: Inter-option Learning}.

\begin{figure}[b!]
    \begin{subfigure}{.5\textwidth}
    \centering
    \includegraphics[width=0.97\textwidth]{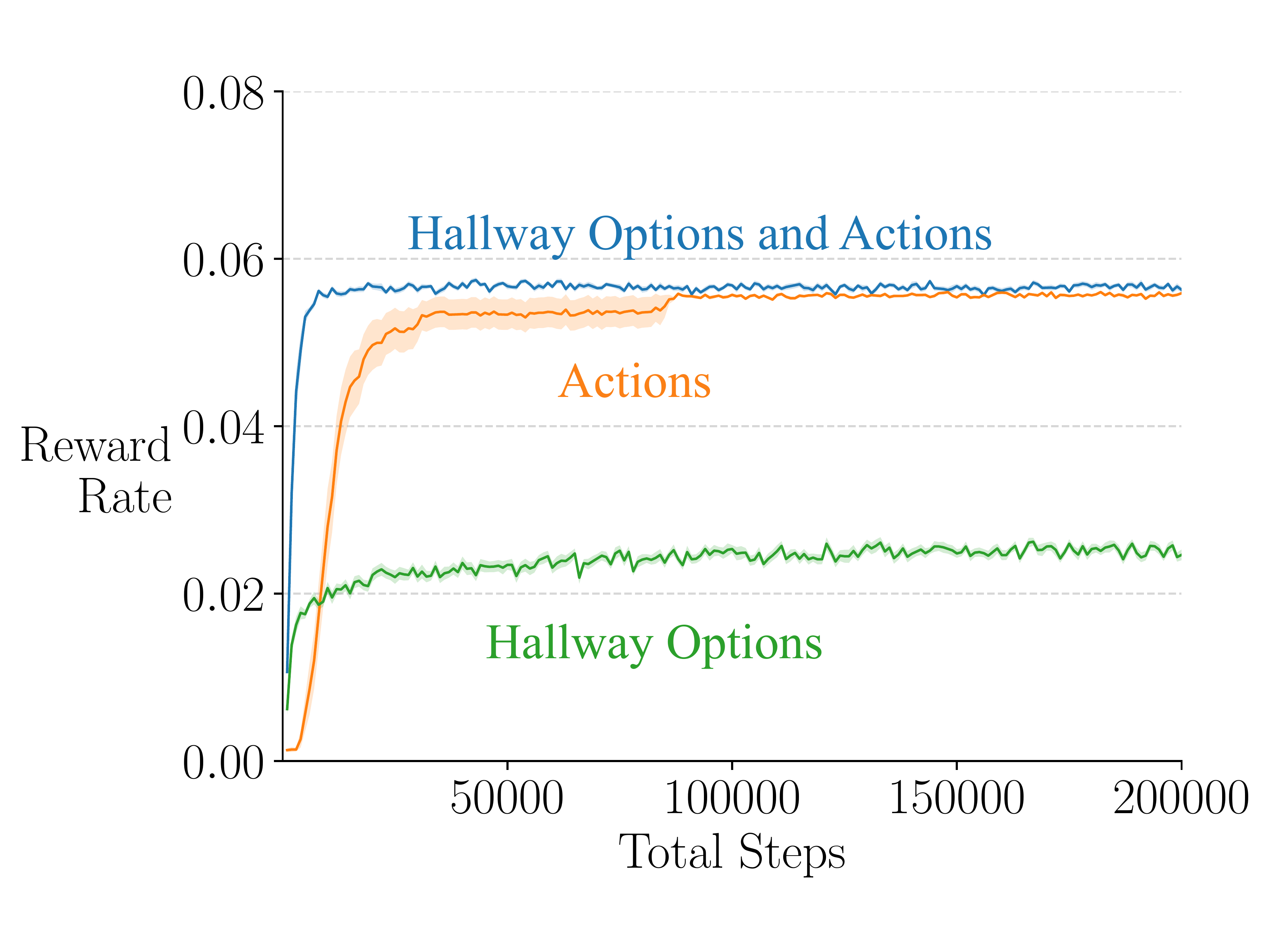}
    \end{subfigure}
        \begin{subfigure}{.5\textwidth}
    \centering
    \includegraphics[width=0.97\textwidth]{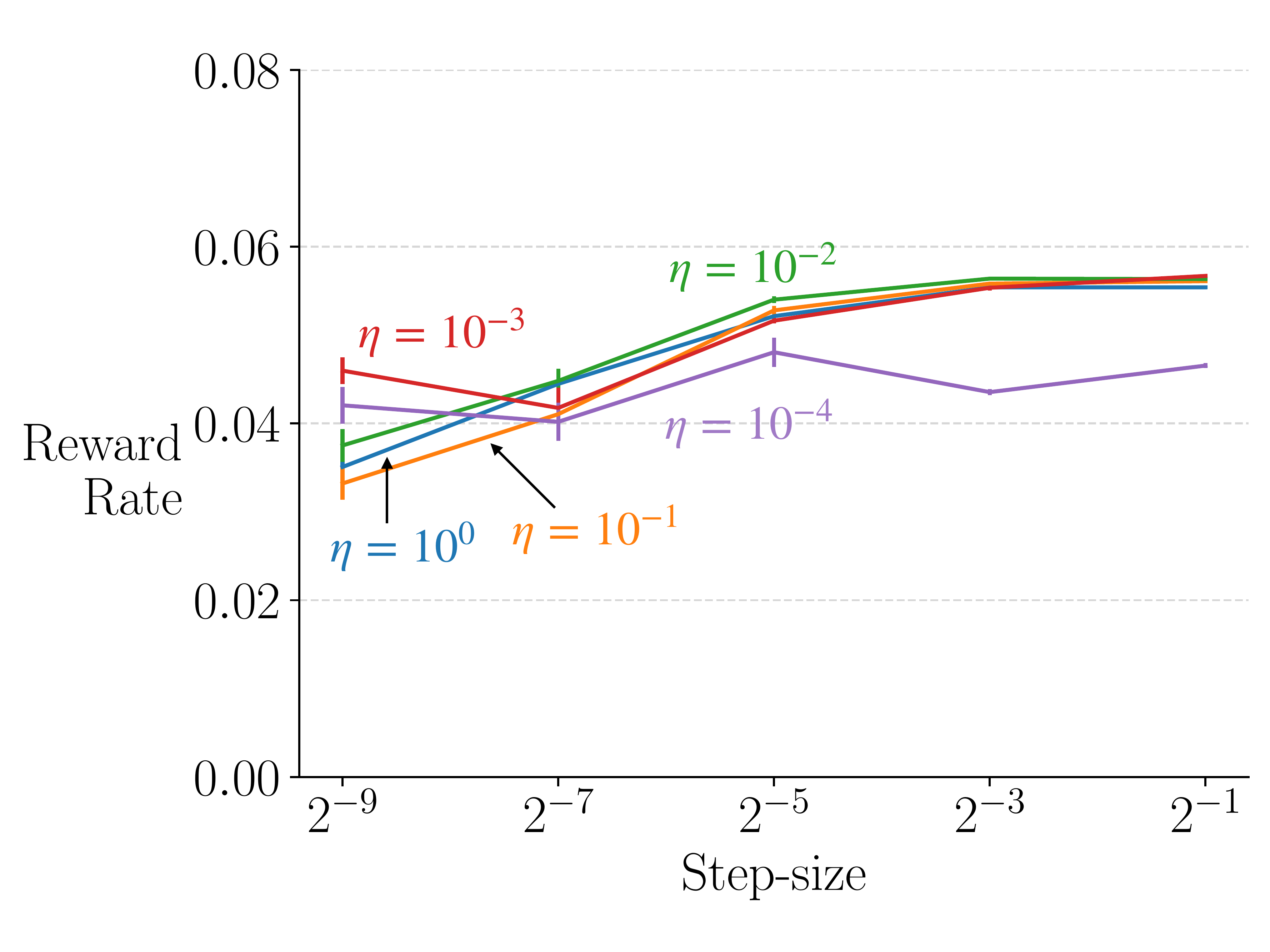}
    \end{subfigure}%
    \caption{Plots showing some learning curves and the parameter study of inter-option Differential Q-learning on the continuing Four-Room domain when the goal was to go to \texttt{G1}. \textit{Left}: 
    A point on the solid line denotes reward rate over the last 1000 time steps and the shaded region indicates one standard error.
    The behavior using the three different sets of options was as expected.
    \textit{Right}: Sensitivity of performance to $\alpha$ and $\eta$ when using $\calO = \calA + \calH$ and $\beta = 2^{-1}$. 
    The x-axis denotes step size $\alpha$; the y-axis denotes the rate of the rewards averaged over all 200,000 steps of training, reflecting the rate of learning. The error bars denote one standard error. The algorithm's rate of learning varied little over a broad range of 
    $\eta$.
    }
    \label{fig: learning curves and sensitivity curves for inter-option learning}
\end{figure}

The learning curves in the left panel of \cref{fig: learning curves and sensitivity curves for inter-option learning} show that the agent achieved a relatively stable reward rate after 100,000 steps in all three cases. 
Using just primitive actions $\cal A$, 
the learning curve rises the slowest, indicating that hallway options indeed help the agent reach the goal faster. But solely using the hallway options $\cal H$ 
is not very useful in the long run as the goal $\texttt{G1}$ is not a hallway state. 
Note that because of the $\epsilon$-greedy behavior policy, the learning curves do not reach the optimal reward rate of $0.0625$.
These observations mirror those by Sutton et al.~(1999) in the discounted formulation.

The sensitivity curves of inter-option Differential Q-learning (right panel of \cref{fig: learning curves and sensitivity curves for inter-option learning}) 
indicate that, in this Four-Room domain, the algorithm was not sensitive to parameter $\eta$, performed well for a wide range of step sizes $\alpha$, and showed low variance across different runs. 
We also found that the algorithm was not sensitive to $\beta$ either; this parameter study is also presented in \cref{app: additional empirical results: Inter-option Learning}. 

\begin{wrapfigure}{r}{.4\textwidth}
    \vspace{4mm}
    \centering
    \includegraphics[width=0.4\textwidth]{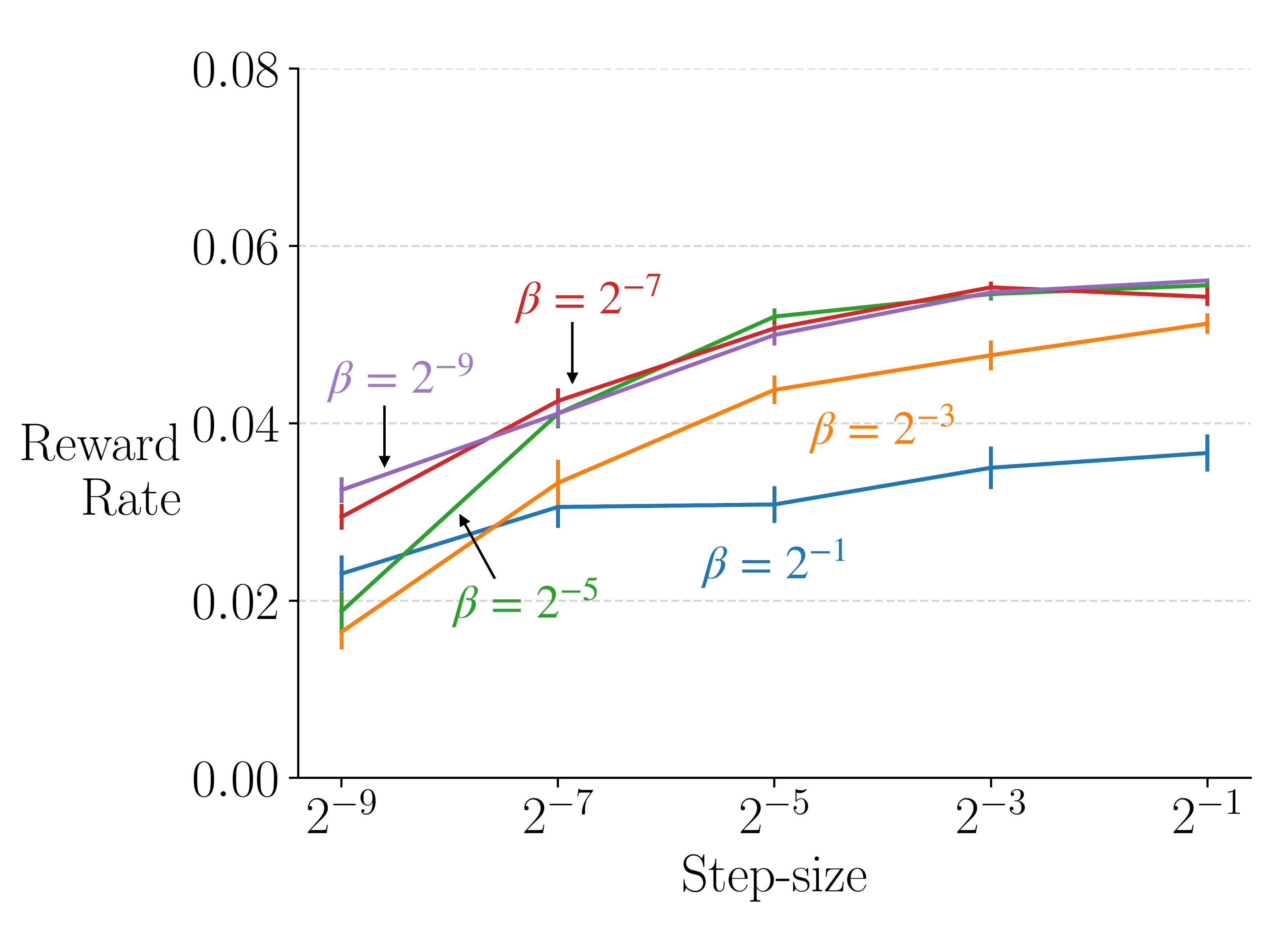}
    \caption{Parameter studies showing  the baseline algorithm's (Gosavi 2004) rate of learning is relatively more sensitive to the choices of its two parameters compared to our inter-option Differential Q-learning. The experimental setting and the plot axes are the same as mentioned in \cref{fig: learning curves and sensitivity curves for inter-option learning}'s caption.
    }
    \label{fig: sensitivity curves for inter-option and gosavi's algorithm}
    \vspace{-8mm}
\end{wrapfigure}

We also tested Gosavi's (2004) algorithm as a baseline. 
We chose not to compare the proposed algorithms in this paper with Sutton et al.'s (1999) discounted versions because the discounted and average-reward problem formulations are different; comparing the performance of their respective solution methods would be inappropriate and difficult to interpret. 
We have proposed new solution methods for the average-reward formulation, hence in this case Gosavi's (2004) algorithm is the most appropriate baseline.
Recall it is the only proven-convergent average-reward SMDP off-policy control learning algorithm prior to our work.
It estimates the reward rate by tracking the cumulative reward $\bar C$ obtained by the options and dividing it by the another estimate $\bar T$ the tracks the length of the options. If the $n\textsuperscript{th}$ option executed is a greedy choice, then these estimates are updated using:
\begin{align*}
    \bar C_{n+1} &\doteq \bar C_n + \beta_n (\hat R_n - \bar C_n),\\
    \bar T_{n+1} &\doteq \bar T_n + \beta_n (\hat L_n - \bar T_n),\\
    \bar R_{n+1} &\doteq \bar C_{n+1} / \bar T_{n+1}.
    \vspace{-1mm}
\end{align*}

When $\hat O_n$ is not greedy, $\bar R_{n+1} = \bar R_n$. 
The option-value function is updated with \eqref{eq: possible extension of Diff Q} with $\delta_n$ as defined in \eqref{eq: Inter-option Differential Q-learning attempt TD error}.
$\alpha_n$ and $\beta_n$ are two step-size sequences. 
The sensitivity of this algorithm with $\calO = \calA + \calH$ is shown in \cref{fig: sensitivity curves for inter-option and gosavi's algorithm}. Compared to inter-option Differential Q-learning, this baseline has one less parameter, but its performance 
was found to be more sensitive to the values of both its step-size parameters. 
In addition, the error bars were generally larger, suggesting that the variance across different runs was also higher. 

To conclude, our experiments with the continuing Four-Room domain show that our 
inter-option Differential Q-learning indeed finds the optimal policy given a set of options, in accordance with \cref{thm: Inter-option Differential Methods}. 
In addition, its performance seems more robust to the choices of parameters compared to the baseline.
Experiments with the prediction algorithm, inter-option Differential Q-evaluation, are presented in \cref{app: additional empirical results: Prediction}.


\section{Intra-Option Value Learning and Planning Algorithms} \label{sec: intra-option value}

In this section, we introduce intra-option value learning and planning algorithms. 
The objectives are same as that of inter-option value learning algorithms.
As mentioned earlier, intra-option algorithms learn from every transition $S_t, A_t, R_{t+1}, S_{t+1}$ during the execution of a given option $O_t$. 
Moreover, intra-option algorithms also make updates for \textit{every} option $o \in \cal O$, including ones that may potentially never be executed.

To develop our algorithms, we first establish the intra-option evaluation and optimality equations in the average-reward case.
The general form of the intra-option Bellman equation is:
\begin{align} 
    q(s, o) &= \sum_{a} \pi(a \,|\, s, o)\sum_{s', r} p(s', r \mid s, a) \Big( r - \bar r + u^q(s', o)
    \Big) \label{eq: Intra-option Equations}
\end{align}
\normalsize
where $q \in \bbR^{\cardS \times |\calO|}$ and $\bar r \in \bbR$ are free variables. 
The optimality and evaluation equations use $u^q = u^q_*$ and $u^q = u^q_\mu$ respectively, defined $\forall\,s'\in\calS, o\in\calO$ as:
\begin{align}
    u^q(s', o) = u^q_*(s', o) &\doteq \big( 1 - \beta(s', o) \big) q(s', o) + \beta(s', o) \max_{o'} q(s', o'), \label{eq: Intra-option Optimality Equations} \\
    u^q(s', o) = u^q_\mu(s', o) &\doteq \big( 1 - \beta(s', o) \big) q(s', o) + \beta(s', o) \sum_{o'} \mu(o' | s') q(s', o'). \label{eq: Intra-option Evaluation Equations}
\end{align}
Intuitively, the $u^q$ term accounts for the two possibilities of an option terminating or continuing in the next state. These equations generalize the average-reward Bellman equations given by Puterman (1994). The following theorem characterizes the solutions to the intra-option Bellman equations. 

\vspace{3mm}
\begin{theorem}[Solutions to intra-option Bellman equations]\label{thm: intra-option bellman equations}
If \cref{assu: unichain} holds, then: 
\vspace{-2mm}
\begin{enumerate}\itemsep-1mm
    \item a) there exists a $\bar r \in \bbR$ and a $q \in \bbR^{\abs{\calS} \times \abs{\calO}}$ for which \eqref{eq: Intra-option Equations} and \eqref{eq: Intra-option Optimality Equations} hold, \\
    b) the solution of $\bar r$ is unique and is equal to $r_*$, let $q_1$ be one solution of $q$, the solutions of $q$ form a set $\{q_1 + c\,e \mid c \in \bbR \}$ where $e$ is an all-one vector of size $\abs{\calS} \times \abs{\calO}$,\\ 
    c) a greedy policy w.r.t. a solution of $q$ achieves the optimal reward rate $r_*$.
    \item a) there exists a $\bar r \in \bbR$ and a $q \in \bbR^{\abs{\calS} \times \abs{\calO}}$ for which \eqref{eq: Intra-option Equations} and \eqref{eq: Intra-option Evaluation Equations} hold, \\
    b) the solution of $\bar r$ is unique and is equal to $r(\mu)$, the solutions of $q$ form a set $\{q_\mu + c\,e \mid c \in \bbR \}$.
\end{enumerate}
\end{theorem}

The proof extends those of Corollary 8.2.7, Theorem 8.4.3, Theorem 8.4.4 by Puterman (1994) and is presented in \cref{app: proof of intra-option value equations}.

Our intra-option control and prediction algorithms are stochastic approximation algorithms solving the intra-option optimality and evaluation equations respectively. Both the algorithms maintain a vector of estimates $Q(s, o)$ and a scalar estimate $\bar R$, just like our inter-option algorithms. However, unlike inter-option algorithms, intra-option algorithms need not maintain an estimator for option lengths ($L$) because they make updates after every transition.
Our control algorithm, called \emph{intra-option Differential Q-learning}, updates the estimates $Q$ and $\bar R$ by:
\begin{align}
    Q_{t+1}(S_t, o) &\doteq Q_{t}(S_t, o) + \alpha_t \rho_t(o) \delta_t(o), \quad \forall\ o \in \optionspace, \label{eq: Intra-option Differential TD-learning Q}\\
    \bar R_{t+1} &\doteq \bar R_t + \eta \alpha_t \sum_{o \in \optionspace} \rho_t(o) \delta_t(o), \label{eq: Intra-option Differential TD-learning R bar}
\end{align}
where $\alpha_t$ is a step-size sequence, $\rho_t(o) \doteq \frac{\pi(A_t | S_t, o)}{\pi(A_t | S_t, O_t)}$ is the importance sampling ratio, and:
\begin{align}
\label{eq: Intra-option Differential Q-learning TD error}
    \delta_t(o) \doteq R_{t+1} - \bar R_t + u^{Q_t}_*(S_{t+1}, o) - Q_t(S_t, o).
\end{align}
\normalsize

Our prediction algorithm, called \emph{intra-option Differential Q-evaluation}, also updates $Q$ and $\bar R$ by
\eqref{eq: Intra-option Differential TD-learning Q} and \eqref{eq: Intra-option Differential TD-learning R bar} but with the TD error:
\begin{align}
\label{eq: Intra-option Differential TD-learning TD error}
    \delta_t(o) & \doteq R_{t+1} - \bar R_t + u^{Q_t}_\mu (S_{t+1}, o)
    - Q_t(S_t, o).
\end{align}
\normalsize

\vspace{2mm}
\begin{theorem}[Convergence of intra-option algorithms; informal] \label{thm: Intra-option Differential Methods}
Under the conditions of \cref{thm: Inter-option Differential Methods}: 
\vspace{-2mm}
\begin{enumerate}\itemsep-1mm
    \item intra-option Differential Q-learning algorithm (\ref{eq: Intra-option Differential TD-learning Q}–\ref{eq: Intra-option Differential Q-learning TD error}) converges almost surely, $\bar R_t$ to $r_*$, $Q_t(s, o)$ to a solution of $q(s, o)$ in \eqref{eq: Intra-option Equations} and \eqref{eq: Intra-option Optimality Equations} for all $s \in \calS', o \in \calO$, and $r(\mu_t)$ to $r_*$, where $\mu_t$ is a greedy policy w.r.t.~$Q_t$, 
    \item intra-option Differential Q-evaluation algorithm (\ref{eq: Intra-option Differential TD-learning Q},\ref{eq: Intra-option Differential TD-learning R bar},\ref{eq: Intra-option Differential TD-learning TD error}) converges almost surely, $\bar R_t$ to $r(\mu)$, $Q_t(s, o)$ to a solution of $q(s, o)$ in \eqref{eq: Intra-option Equations} and \eqref{eq: Intra-option Evaluation Equations} for all $s \in \calS', o \in \calO$.
\end{enumerate}
\end{theorem}

\begin{wrapfigure}{r}{0.4\textwidth}
    \vspace{-5mm}
    \centering
    \includegraphics[width=0.4\textwidth]{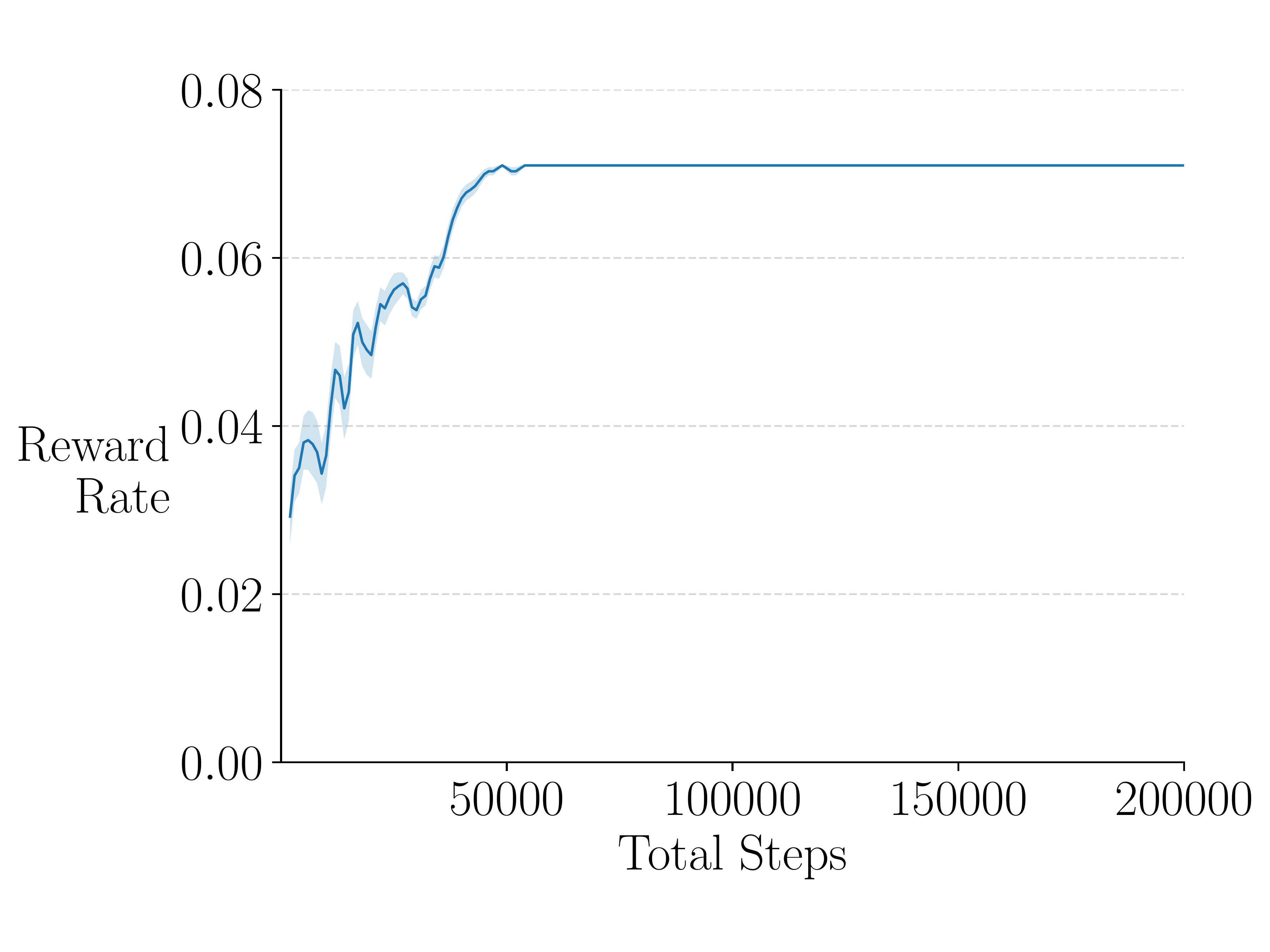}
    \caption{Learning curve showing that the greedy policy corresponding to the hallway options' option-value function achieves the optimal reward rate on the continuing Four-Room domain. The value function was learned via intra-option Differential Q-learning using a behavior policy consisting only of primitive actions; the hallway options were never executed.
    }
    \label{fig: learning curve for intra-option learning with random actions}
    \vspace{-12mm}
\end{wrapfigure}

\textbf{Remark:} The intra-option learning methods introduced in this section can be used with options having stochastic policies. This is possible with the use of the important sampling ratios as described above. Sutton et al.'s (1999) discounted intra-option learning methods were restricted to options having deterministic policies. 

Again, the intra-option value planning algorithms are similar to the learning algorithms except that they use simulated experience generated by a given or learned model instead of real experience. 
The planning algorithms and their convergence results are presented in \cref{app: intra-option value}.

\textbf{Empirical Evaluation}. 
We conducted another experiment in the Four-Room domain to show that intra-option Differential Q-learning can learn the values of hallway options $\calH$ using only primitive actions $\calA$. As mentioned earlier, there are no baseline intra-option average-reward algorithms, so this is a proof-of-concept experiment.

The goal state for this experiment was \texttt{G2}, which can be reached using the option that leads to the lower hallway.
The optimal reward rate in this case is $1/14 \approx 0.714$ with both $\calO = \calH$ and $\calO = \calA$. 
We applied intra-option Differential Q-learning 
using a behavior policy that chose the four primitive actions with equal probability in all states. 
This choice of behavior policy and goal \texttt{G2} would test if the intra-option algorithm leads to a policy consisting exclusively of options by interacting with the environment using only primitive actions.
Each parameter setting was run for 200,000 steps and repeated 30 times.
For evaluation, we saved the learned option value function after every 1000 steps and computed the average reward of the corresponding greedy policy over 1000 steps. 

\cref{fig: learning curve for intra-option learning with random actions} shows the learning curve of this average reward across the 30 independent runs for parameters $\alpha=0.125, \eta=0.1$. 
The agent indeed succeeds in learning the option-value function corresponding to the hallway options using a behavior policy consisting only of primitive actions. 
The parameter study of intra-option Differential Q-learning is presented in \cref{app: additional empirical results: Intra-option Learning}. 
Experiments with the prediction algorithm, intra-option Differential Q-evaluation, are presented in \cref{app: additional empirical results: Prediction}. 


\section{Intra-Option Model Learning and Planning Algorithms}\label{sec: intra-option model}

In this section, we first describe option models within the average-reward formulation.
We then introduce
an algorithm to learn such models in an intra-option fashion. This option-model learning algorithm can be combined with the planning algorithms from the previous section to obtain a complete model-based average-reward options algorithm that learns option models and plans with them (we present this combined algorithm in \cref{app: additional discussion: pseudocodes}).

The average-reward option model is similar to the discounted options model but with key distinctions.
Sutton et al.'s (1999) discounted option model has two parts: the dynamics part and the reward part.
Given a state and an option, the dynamics part predicts the discounted occupancy of states upon termination, and the reward part predicts the expected (discounted) sum of rewards during the execution of the option.
In the average-reward setting, apart from the dynamics and the reward parts, an option model has a third part — the \textit{duration} part — that predicts the duration of execution of the option. In addition, the dynamics part predicts the state distribution upon termination without discounting and reward part predicts the undiscounted cumulative rewards during the execution of the option.

Formally, the dynamics part approximates $\dynamicsmodeltarget(s' | s, o) \doteq \sum_{r, l} \hat p(s', r, l\,| s, o)$, the probability that option $o$ terminates in state $s'$ when starting from state $s$. 
The reward part approximates $\rewardmodeltarget(s, o) \doteq \sum_{\new{s'}, r, l} \hat p(s', r, l\,| s, o)\, r$, the expected cumulative reward during the execution of option $o$ when starting from state $s$. 
Finally, the duration part approximates $\durationmodeltarget (s, o) \doteq \sum_{ s',r,\new{l}} \hat p(s', r, l\,| s, o)\, l$, the expected duration of option $o$ when starting from state $s$. 

We now present a set of recursive equations that are key to our model-learning algorithms. These equations extend the discounted Bellman equations for option models (Sutton et al.~1999) to the average-reward formulation.
\fontsize{9.5}{6}
\begin{align}
    \bar m^p(x \mid s, o) &= \sum_{a} \pi(a\, |\, s, o)\sum_{\new{s',} r} p(s', r \mid s, a) \Big( \beta(s', o) \mathbb{I}(x = s') + \big(1 - \beta(s', o)\big) \bar m^p(x \mid s', o) \Big), \label{eq: Intra-option Dynamics Model Equation}\\
    \bar m^r(s, o) &= \sum_{a} \pi(a\, |\, s, o)\sum_{s', r} p(s', r \mid s, a) \big( r + (1 - \beta(s', o)) \bar m^r(s', o)\big) \label{eq: Intra-option Reward Model Equation}, \\
    \bar m^l(s, o) &= \sum_{a} \pi(a\, |\, s, o)\sum_{s', r} p(s', r \mid s, a) \big( 1 + (1 - \beta(s', o)) \bar  m^l (s', o)\big). \label{eq: Intra-option Duration Model Equation}
\end{align}
\normalsize
The first equation are different from the other two because the total reward and length of the option $o$ are incremented irrespective of whether the option terminates in $s'$ or not.
The following theorem shows that $(m^p, m^r, m^l)$ is the unique solution of (\ref{eq: Intra-option Dynamics Model Equation}–\ref{eq: Intra-option Duration Model Equation}) and therefore the models can be obtained by solving these equations (see \cref{app: Proof of thm: model equations} for the proof). 

\vspace{2mm}
\begin{theorem}[Solutions to Bellman equations for option models] \label{thm: model equations}
There exist unique $\bar m^p \in \bbR^{\cardS \times |\calO| \times \cardS}$, $\bar m^r \in \bbR^{\cardS \times |\calO|}$, and $\bar m^l \in \bbR^{\cardS \times |\calO|}$ for which \eqref{eq: Intra-option Dynamics Model Equation}, \eqref{eq: Intra-option Reward Model Equation}, and \eqref{eq: Intra-option Duration Model Equation} hold. Further, if $\bar m^p, \bar m^r, \bar m^l$ satisfy \eqref{eq: Intra-option Dynamics Model Equation}, \eqref{eq: Intra-option Reward Model Equation}, and \eqref{eq: Intra-option Duration Model Equation}, then $\bar m^p = \dynamicsmodeltarget, \bar m^r = \rewardmodeltarget, \bar m^l = \durationmodeltarget$.
\end{theorem}
Our \emph{intra-option model-learning} algorithm solves the above recursive equations using the following TD-like update rules for each option $o$:
\fontsize{9.5}{6}
\begin{align}
    \dynamicsmodel_{t+1}(x \mid S_t, o) &\doteq \dynamicsmodel_{t}(x \mid S_t, o) + \alpha_t \rho_t(o) \Big( \beta(S_{t+1}, o) \mathbb{I}(S_{t+1} = x) \nonumber \\
    &\quad + \big( 1 - \beta(S_{t+1}, o) \big) \dynamicsmodel_t(x \mid S_{t+1}, o) - \dynamicsmodel_t(x \mid S_t, o) \Big), \quad \forall\ x \in \statespace, \label{eq: Intra-option Dynamics Model learning}\\
    \rewardmodel_{t+1}(S_t, o) &\doteq \rewardmodel_{t}(S_t, o) + \alpha_t \rho_t(o) \Big( R_{t+1} + \big( 1 - \beta(S_{t+1}, o) \big) \rewardmodel_t(S_{t+1}, o) - \rewardmodel_{t}(S_t, o) \Big) \label{eq: Intra-option Reward Model Learning}\\
    \durationmodel_{t+1}(S_t, o) &\doteq \durationmodel_{t}(S_t, o) + \alpha_t \rho_t(o) \Big( 1 + \big( 1 - \beta(S_{t+1}, o) \big) \durationmodel_{t}(S_{t+1}, o) - \durationmodel_{t}(S_t, o) \Big) \label{eq: Intra-option Duration Model learning}
\end{align}
\normalsize
where $\dynamicsmodel$ is a $\cardS \times |\calO| \times \cardS$-sized vector of estimates, $\rewardmodel$ and $\durationmodel$ are both $\cardS \times |\calO|$-sized vectors of estimates, and $\alpha_t$ is a sequence of step sizes. Standard stochastic approximation results can be applied to show the algorithm's convergence (see \cref{app: thm: intra-option model learning} for details).

\vspace{2mm}
\begin{theorem}[Convergence of the intra-option model learning algorithm; informal] \label{thm: intra-option model learning}
If the step sizes are set appropriately and all the state-option pairs are updated an infinite number of times, then intra-option model-learning (\ref{eq: Intra-option Dynamics Model learning}–\ref{eq: Intra-option Duration Model learning}) converges almost surely, $\dynamicsmodel_t$ to $\dynamicsmodeltarget$, $\rewardmodel_t$ to $\rewardmodeltarget$, and $\durationmodel_t$ to $\durationmodeltarget$.
\end{theorem} 

Our intra-option model-learning algorithms (\ref{eq: Intra-option Dynamics Model learning}–\ref{eq: Intra-option Duration Model learning}) can be applied with simulated one-step transitions generated by a \textit{given action model}, resulting in a planning algorithm that produces an \textit{estimated option model}.
The planning algorithm and its convergence result are presented in \cref{app: thm: intra-option model learning}.


\section{Interruption to Improve Policy Over Options}
\label{sec: interruption}
In all the algorithms we considered so far, the policy over options would select an option, execute the option policy till termination, then select a new option. 
Sutton et al.~(1999) showed that the policy over options can be improved by allowing the \textit{interruption} of an option midway through its execution to start a seemingly better option. 
We now show that this interruption result applies for average-reward options as well (see \cref{app: interruption} for the proof).

\vspace{2mm}
\begin{theorem}[Interruption] \label{thm: interruption}
For any MDP, any set of options $\calO$, and any policy $\mu: \calS \times \calO \to [0, 1]$, define a new set of options, $\calO'$, with a one-to-one mapping between the two option sets as follows: for every $o = (\pi, \beta) \in \calO$, define a corresponding $o' = (\pi, \beta') \in \calO'$ where $\beta' = \beta$, but for any state $s$ in which $q_\mu(s, o) < v_\mu(s)$, $\beta'(s, o) = 1$ (where $v_\mu(s) \doteq \sum_o \mu(o \mid s) q_\mu(s,o)$).
Let the interrupted policy $\mu'$ be such that for all $s \in \calS$ and for all $o' \in \calO', \mu'(s, o') = \mu(s, o)$, where $o$ is the option in $\calO$ corresponding to $o'$. Then:
\vspace{-2mm}
\begin{enumerate}\itemsep-1mm
    \item the new policy over options $\mu'$ is not worse than the old one $\mu$, i.e., $r(\mu') \geq r(\mu)$,
    \item if there exists a state $s \in \calS$ from which there is a non-zero probability of encountering an interruption upon initiating $\mu'$ in $s$, then $r(\mu') > r(\mu)$.
\end{enumerate}
\end{theorem}

\begin{wrapfigure}{r}{0.4\textwidth}
    \vspace{-4mm}
    \centering
    \includegraphics[width=0.4\textwidth]{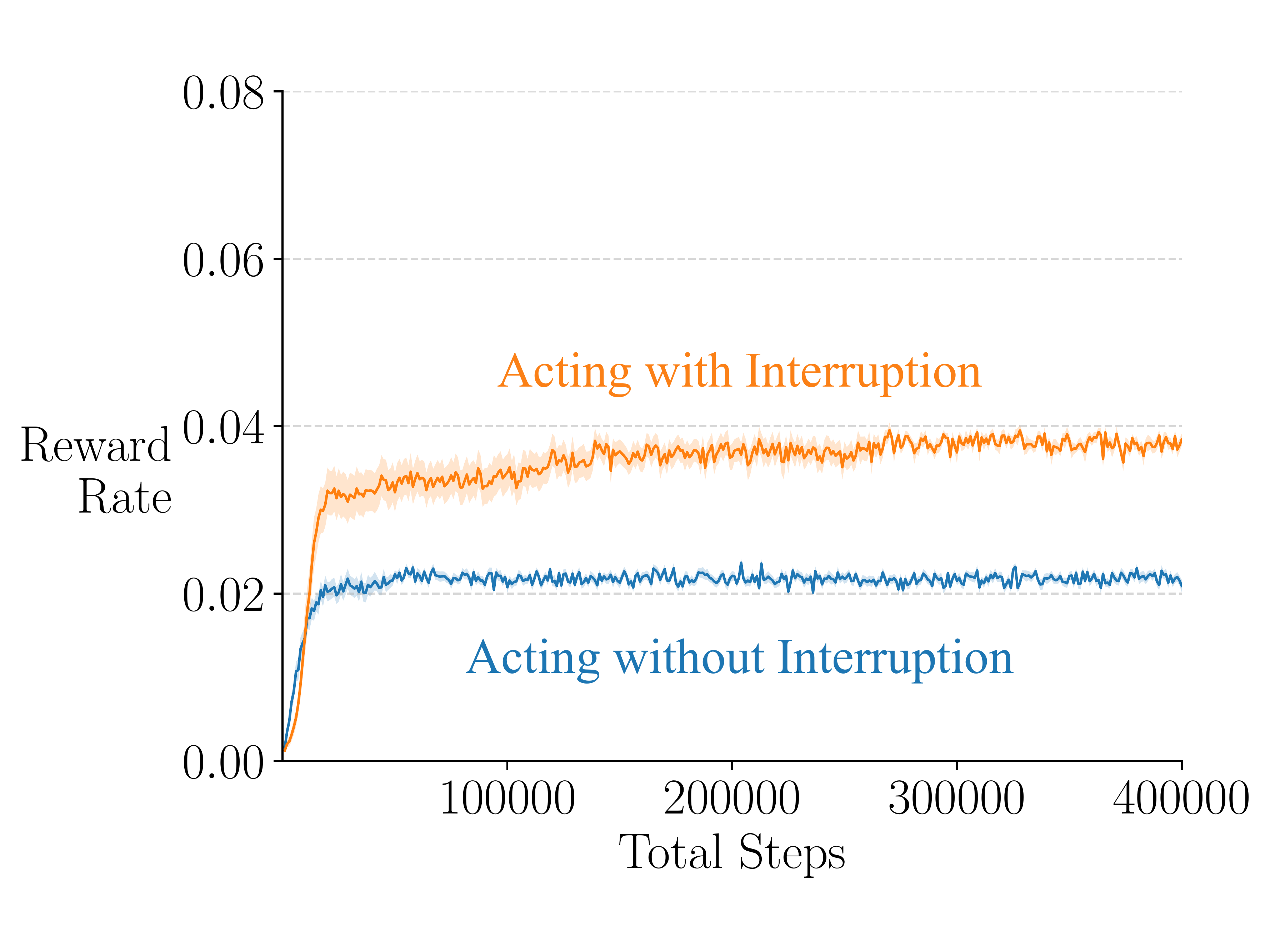}
    \caption{
    Learning curves showing that executing options with interruptions can achieve a higher reward rate than executing options till termination in the domain described in the adjoining text.
    }
    \label{fig: interruption learning curves}
    \vspace{-6mm}
\end{wrapfigure}

In short, the above theorem shows that interruption produces a behavior that achieves a higher reward rate than without interruption.
Note that interruption behavior is only applicable with intra-option algorithms; complete option transitions are needed in inter-option algorithms.

\textbf{Empirical Evaluation.} We tested the intra-option Differential Q-learning algorithm with and without interruption in the Four-Room domain. We set the goal as \texttt{G3} and allowed the agent to choose and learn only from the set of all hallway options $\calH$.
With just hallway options, 
without interruption, the best strategy is to first move to the lower hallway and then try to reach the goal by following options that pick random actions in the states near the hallway and goal. 
With interruption, the agent can first move to the left hallway, 
then take the option that moves the agent to the lower hallway but terminate when other options have higher option-values. 
This termination is most likely to occur in the cell just above \texttt{G3}. 
The agent then needs a fewer number of steps in expectation to reach the goal.

\Cref{fig: interruption learning curves} shows learning curves using intra-option Differential Q-learning with and without interruptions on this problem. 
Each parameter setting was run for 400,000 steps and repeated 30 times. 
The learning curves shown correspond to  $\alpha = 0.125$ and $\eta = 0.1$. 
As expected, the agent achieved a higher reward rate by using interruptions. The parameter study of the interruption algorithm along with the rest of the experimental details is presented in \cref{app: additional empirical results: Interruption}.

\section{Conclusions, Limitations, and Future Work}

In this paper, we extended learning and planning algorithms for the options framework --- originally proposed by Sutton et al.~(1999) for discounted-reward MDPs --- to average-reward MDPs. 
The inter-option learning algorithm presented in this paper is more general than previous work in that its convergence proof does not require existence of any special states in the MDP. 
We also derived the intra-option Bellman equations in average-reward MDPs and used them to propose the first intra-option learning algorithms for average-reward MDPs.
Finally, we extended the interruption algorithm and its related theory from the discounted to the average-reward setting.
Our experiments on a continuing version of the classic Four-Room domain demonstrate the efficacy of the proposed algorithms. 
We believe that our contributions will enable widespread use of options in the average-reward setting. 

We now briefly comment on the novelty of our theoretical and algorithmic contributions. 
Our primary theoretical contribution is to generalize Wan et al.’s (2021) proof techniques to obtain a unified convergence proof for actions and options. 
The same proof techniques then apply for both the inter- and intra-option algorithms. 
Our primary algorithmic contribution is the scaling of the updates by option lengths in the inter-option algorithms. 
The lack of scaling makes the algorithms unstable and prone to divergence. 
Furthermore, we show the correct way of scaling involves estimated option lengths, not sampled option lengths. 

The most immediate line of future work involves extending these ideas from the tabular case to the general case of function approximation, starting with linear function approximation. 
One way to incorporate function approximation is to extend algorithms presented in this paper to those using linear options (Sorg \& Singh 2010, Yao et al.\ 2014), perhaps by building on Zhang et al.'s (2021) work.
Using the results developed in this paper, we also foresee extensions to more ideas from the discounted formulation involving function approximation, such as Bacon et al.'s (2017) option-critic architecture, to the average-reward formulation.

This paper assumes that a fixed set of options is provided and the agent then learns or plans using them. One of the most important challenges in the options framework is the \textit{discovery} of options. 
We think the discovery problem is orthogonal to the problem formulation. Hence, another line of future work is to extend existing option-discovery algorithms developed for the discounted formulation to the average-reward formulation (e.g., algorithms by McGovern \& Barto 2001, Menache et al.~2002, Şimşek \& Barto 2004, Singh et al.~2004, Van Djik \& Polani 2011, Machado et al.~2017) . 
Relatively more work might be required in extending approaches that couple the problems of option discovery and learning (e.g., Gregor et al.~2016, 
Eysenbach et al.~2018, Achiam et al.~2018, Veeriah et al.~2021). 

Another limitation of this paper is that it deals with learning and planning separately. We also need combined methods that learn models as well as plan with them; we discuss some ideas in \cref{app: additional discussion}. 
Finally, we would like to get more empirical experience with the algorithms proposed in this paper, both in pedagogical tabular problems and challenging large-scale problems. 
Nevertheless, we believe this paper makes novel contributions that are significant for the use of temporal abstractions in average-reward reinforcement learning.


\section*{Acknowledgements}

The authors were supported by DeepMind, Amii, and CIFAR.
The authors wish to thank Huizhen Yu for extensive discussions on several related works; Benjamin Van Roy, Csaba Szepesv\'{a}ri, Dale Schuurmans, Martha White, and the anonymous reviewers for valuable feedback.


\section*{References}

\parskip=5pt
\parindent=0pt
\def\hangin{\hangindent=0.2in}

\hangin
Abounadi, J., Bertsekas, D., \& Borkar, V. S. (2001).  Learning Algorithms for Markov Decision Processes with Average Cost. \emph{SIAM Journal on Control and Optimization}.

\hangin
Achiam, J., Edwards, H., Amodei, D., \& Abbeel, P. (2018). Variational Option Discovery Algorithms. \emph{ArXiv:1807.10299}.

\hangin
Almezel, S., Ansari, Q. H., \& Khamsi, M. A. (2014). \emph{Topics in Fixed Point Theory (Vol. 5)}. Springer.

\hangin
Bacon, P. L., Harb, J., \& Precup, D. (2017). The Option-Critic Architecture. \emph{AAAI Conference on Artificial Intelligence}.

\hangin
Borkar, V. S. (1998). Asynchronous Stochastic Approximations. \emph{SIAM Journal on Control and Optimization}. 

\hangin
Borkar, V. S. (2009). \emph{Stochastic Approximation: A Dynamical Systems Viewpoint}. Springer.

\hangin
Borkar, V. S., \& Soumyanatha, K. (1997). An Analog Scheme for Fixed Point Computation. I. Theory. \emph{IEEE Transactions on Circuits and Systems I: Fundamental Theory and Applications}. 

\hangin
Brunskill, E., \& Li, L. (2014). PAC-inspired Option Discovery in Lifelong Reinforcement Learning. \emph{International Conference on Machine Learning}.

\hangin
Das, T. K., Gosavi, A., Mahadevan, S., \& Marchalleck, N. (1999). Solving Semi-Markov Decision Problems Using Average Reward Reinforcement Learning. \emph{Management Science}.

\hangin
Eysenbach, B., Gupta, A., Ibarz, J., \& Levine, S. (2018). Diversity is All You Need: Learning Skills without a Reward Function. \emph{ArXiv:1802.06070}.

\hangin
Fruit, R., \& Lazaric, A. (2017). Exploration-Exploitation in MDPs with Options. \emph{Artificial Intelligence and Statistics}. 

\hangin
Gosavi, A. (2004). Reinforcement Learning for Long-run Average Cost. \textit{European Journal of Operational Research}.

\hangin
Gregor, K., Rezende, D. J., \& Wierstra, D. (2016). Variational Intrinsic Control. \emph{ArXiv:1611.07507}.

\hangin 
Li, Y., \& Cao, F. (2010). RVI Reinforcement Learning for Semi-Markov Decision Processes with Average Reward. 
\emph{IEEE World Congress on Intelligent Control and Automation}.

\hangin
Machado, M. C., Bellemare, M. G., \& Bowling, M. (2017). A Laplacian Framework for Option Discovery in Reinforcement Learning. \emph{International Conference on Machine Learning}.

\hangin
McGovern, A., \& Barto, A. G. (2001). Automatic Discovery of Subgoals in Reinforcement Learning using Diverse Density. \emph{International Conference on Machine Learning}.

\hangin
Menache, I., Mannor, S., \& Shimkin, N. (2002). Q-Cut—Dynamic Discovery of Sub-goals in Reinforcement Learning. \emph{European Conference on Machine Learning}.

\hangin
Puterman, M. L. (1994). \emph{Markov Decision Processes: Discrete Stochastic Dynamic Programming.} John Wiley \& Sons.

\hangin
Schweitzer, P. J. (1971). Iterative Solution of the Functional Equations of Undiscounted Markov Renewal Programming. \emph{Journal of Mathematical Analysis and Applications}.

\hangin
Schweitzer, P. J., \& Federgruen, A. (1978). The Functional Equations of Undiscounted Markov Renewal Programming. \emph{Mathematics of Operations Research}.

\hangin
Şimşek, Ö., \& Barto, A. G. (2004). Using Relative Novelty to Identify Useful Temporal Abstractions in Reinforcement Learning. \emph{International Conference on Machine Learning}.

\hangin
Singh, S., Barto, A. G., \& Chentanez, N. (2004). Intrinsically Motivated Reinforcement Learning. \emph{Advances in Neural Information Processing Systems.}

\hangin
Sorg, J., \& Singh, S. (2010). Linear Options. \emph{International Conference on Autonomous Agents and Multiagent Systems}.

\hangin
Sutton, R. S., Precup, D., \& Singh, S. (1999). Between MDPs and Semi-MDPs: A Framework for Temporal Abstraction in Reinforcement Learning. \emph{Artificial Intelligence}.

\hangin
Sutton, R. S., \& Barto, A. G. (2018). \emph{Reinforcement Learning: An Introduction.} MIT Press.

\hangin
Tsitsiklis, J. N. (1994). Asynchronous StochasticAapproximation and Q-learning. \emph{Machine Learning}. 

\hangin
van Dijk, S. G., \& Polani, D. (2011). Grounding Subgoals in Information Transitions. \emph{IEEE Symposium on Adaptive Dynamic Programming and Reinforcement Learning}.

\hangin
Veeriah, V., Zahavy, T., Hessel, M., Xu, Z., Oh, J., Kemaev, I., van Hasselt, H., Silver, D., \& Singh S. (2021). Discovery of Options via Meta-Learned Subgoals. \emph{ArXiv:2102.06741}.

\hangin
Vien, N. A., \& Chung, T. (2008). Policy Gradient Semi-Markov Decision Process. \emph{IEEE International Conference on Tools with Artificial Intelligence}.

\hangin
Wan, Y., Naik, A., \& Sutton, R. S. (2021). Learning and Planning in Average-Reward Markov Decision Processes.  \textit{International Conference on Machine Learning}.

\hangin
Yao, H., Szepesvári, C., Sutton, R. S., Modayil, J., \& Bhatnagar, S. (2014). Universal Option Models. \emph{Advances in Neural Information Processing Systems}.

\hangin
Zhang, S., Wan, Y., Sutton, R. S., \& Whiteson, S. (2021). Average-Reward Off-Policy Policy Evaluation with Function Approximation. \textit{International Conference on Machine Learning}.


\newpage

\onecolumn

\appendix

\counterwithin{figure}{section}
\counterwithin{table}{section}
\counterwithin{theorem}{section}
\counterwithin{lemma}{section}
\counterwithin{assumption}{section}
\counterwithin{equation}{section}

\section{Formal Theoretical Results and Proofs}
\label{app:all proofs}
In this section, we provide formal statements of the theorems presented in the main text of the paper and show their proofs. This section has several subsections. The first subsection introduces General RVI Q, which will be used in later subsections. The other six subsections correspond to six theorems presented in the main text.

\subsection{General RVI Q} \label{app: general rvi q}
Wan et al.~ (2021) extended the family of RVI Q-learning algorithms (Abounadi, Bertsekas, and Borkar et al.~ 2001) to prove the convergence of their Differential Q-learning algorithm. 
Unlike RVI Q-learning, Differential Q-learning does not require a reference function. 
We further extend Wan et al.'s extended family of RVI Q-learning algorithms 
to a more general family of algorithms, called \emph{General RVI Q}. We then prove convergence for this family of algorithms and show that
inter-option algorithms and intra-option value learning algorithms are all members of this family.

We first need the following definitions:
\begin{enumerate}
    \item a set-valued process $\{Y_n\}$ taking values in the set of nonempty subsets of $\calI$ with the interpretation: $Y_n = \{i: i\textsuperscript{th}$ component of $Q$ was updated at time $n\}$,
    \item $\nu(n, i) \doteq \sum_{k=0}^n I\{i \in Y_k\}$, where $I$ is the indicator function. Thus $\nu(n, i) =$ the number of times the $i$ component was updated up to step $n$,
    \item i.i.d. random vectors $R_n$, $G_n$ and $F_n$ for all $n \geq 0$ satisfying $\bbE \left [R_n(i) \right] = r(i)$, where $r$ is a fixed real vector, $ \bbE[G_n(Q)(i)] = g(Q)(i)$ for any $Q \in \bbR^{\abs{\calI}}$ where $g: \calI \to \calI$ is a function satisfying \cref{assu: g} and $ \bbE[F_n(Q)(i)] = f(Q)$ for any $i \in \calI$ and $Q \in \bbR^{\abs{\calI}}$ where $f: \calI \to \bbR$ is a function satisfying \cref{assu: f}.
\end{enumerate}

\begin{assumption}\label{assu: g}
1) $g$ is a max-norm non-expansion, 2) $g$ is a span-norm non-expansion, 3) $g(x + ce) = g(x) + ce$ for any $c \in \bbR, x \in \bbR^{\abs{\calI}}$, 4) $g(cx) = cg(x)$ for any $c \in \bbR, x \in \bbR^{\abs{\calI}}$.
\end{assumption}

\begin{assumption}\label{assu: f}
1) $f$ is L-Lipschitz, 2) there exists a positive scalar $u$ s.t. $f(e) = u$ and $f(x + ce) = f(x) + cu$, 3) $f(cx) = cf(x)$.
\end{assumption}

\begin{assumption}\label{assu: variance of Martingale difference}
For $n \in \{0, 1, 2, \dots \}$, $\bbE [\norm{R_{n} - r}^2] \leq K$, $\bbE [\norm{G_{n}(Q) - g(Q)}^2] \leq K(1 + \norm{Q}^2)$ for any $Q \in \bbR^{\abs{\calI}}$, and $\bbE [\norm{F_{n}(Q) - f(Q)e}^2] \leq K(1 + \norm{Q}^2)$ for any $Q \in \bbR^{\abs{\calI}}$ for a suitable constant $K > 0$.
\end{assumption}
The above assumption means that the variances of $R_n$, $G_n(Q)$, and $F_n(Q)$ for any $Q$ are bounded.

General RVI Q's update rule is
\begin{align}
    Q_{n+1}(i) &\doteq Q_n(i) + \alpha_{\nu(n, i)} \big( R_n(i) - F_n(Q_n)(i) + G_n(Q_n)(i) - Q_n(i) + \epsilon_n(i) \big) I\{i \in Y_n\} \label{eq: General RVI Q async update},
\end{align}
where $\alpha_{\nu(n, i)}$ is the stepsize and $\epsilon_n$ is a sequence of random vectors of size $\abs{\calI}$.

We make following assumption on $\epsilon_n$.
\begin{assumption}[Noise Assumption] \label{assu: epsilon} $\norm{\epsilon_n}_\infty \leq  K (1 + \norm{Q_n}_\infty)$ for some scalar $K$. Further, $\epsilon_n$ converges in probability to 0.
\end{assumption}

We make following assumptions on $\alpha_{\nu(n, i)}$.
\begin{assumption}[Stepsize Assumption] \label{assu: stepsize} For all $n \geq 0$, $\alpha_n > 0$, $\sum_{n = 0}^\infty \alpha_n = \infty$, and $\sum_{n = 0}^\infty \alpha_n^2 < \infty$.
\end{assumption}

\begin{assumption}[Asynchronous Stepsize Assumption A] \label{assu: asynchronous stepsize 1}
Let $[\cdot]$ denote the integer part of $(\cdot)$, for $x \in (0, 1)$, 
\begin{align*}
    \sup_i \frac{\alpha_{[xi]}}{\alpha_i} < \infty
\end{align*}
and 
\begin{align*}
    \frac{\sum_{j=0}^{[yi]} \alpha_j}{\sum_{j=0}^i \alpha_j} \to 1
\end{align*} 
uniformly in $y \in [x, 1]$.
\end{assumption}

\begin{assumption}[Asynchronous Stepsize Assumption B] \label{assu: asynchronous stepsize 2}
There exists $\Delta > 0$ such that 
\begin{align*}
    \liminf_{n \to \infty} \frac{\nu(n, i)}{n+1} \geq \Delta,
\end{align*}
a.s., for all $s \in \calS, o \in \calO$.
Furthermore, for all $x > 0$, let 
\begin{align*}
    N(n, x) = \min \Bigg \{m > n: \sum_{i = n+1}^m \alpha_i \geq x \Bigg \},
\end{align*}
the limit 
\begin{align*}
    \lim_{n \to \infty} \frac{\sum_{i = \nu(n, i)}^{\nu(N(n, x), i)} \alpha_i}{\sum_{i = \nu(n, i')}^{\nu(N(n, x), i')} \alpha_i}
\end{align*}
exists a.s. for all $s, s', o, o'$.
\end{assumption}

\begin{assumption}\label{assu: General RVI Q unique solution}
$r(i) - \bar r + g(q)(i) - q(i) = 0, \forall i \in \calI$ has a unique solution for $\bar r$ and a unique for $q$ only up to a constant. 
\end{assumption}

Denoted the unique solution of $\bar r$ by $r_\infty$. Further, it can be seen that the solution of $q$ satisfying both $r - \bar r e - g(q) - q = 0$ and $f(q) = r_\infty$ is unique because our assumption on $f$ (\cref{assu: f}). Denote the unique solution as $q_\infty$. We have,
\begin{align}
    f(q_\infty) = r_\infty \label{eq: General RVI Q relation between f(q_infty) and r_infty}.
\end{align}

\begin{theorem}\label{thm: General RVI Q}
Under Assumptions~\ref{assu: g}-\ref{assu: General RVI Q unique solution}, General RVI Q converges, almost surely, $Q_n$ to $q_\infty$ and $f(Q_n)$ to $r_\infty$.
\end{theorem}
\begin{proof}
Because \eqref{eq: General RVI Q async update} is in the same form as the asynchronous update (Equation 7.1.2) by Borkar (2009), we apply the result in Section 7.4 of the same text (Borkar 2009) (see also Theorem 3.2 by Borkar (1998)) which shows convergence for Equation 7.1.2, to show the convergence of \eqref{eq: General RVI Q async update}. This result, given Assumption \ref{assu: asynchronous stepsize 1} and \ref{assu: asynchronous stepsize 2}, only requires showing the convergence of the following \emph{synchronous} version of the General RVI Q algorithm:
\begin{align}
    & Q_{n+1}(i) \doteq Q_{n}(i) + \alpha_n \big( R_n(i) - F_n(Q_n)(i) + g(Q_n)(i) - Q_n(i) \big) \quad \forall i \in \calI. \label{eq: General RVI Q sync update}
\end{align}

Define operators $T_1, T_2$:
\begin{align*}
    T_1 (Q)(i) & \doteq r(i) + g(Q)(i) - r_\infty, \\
    T_2 (Q)(i) & \doteq r(i) + g(Q)(i) - f(Q)\\
    & = T_1 (Q)(i) +  \left (r_\infty - f(Q) \right).
\end{align*}

Consider two ordinary differential equations (ODEs):
\begin{align}
    \dot y_t & \doteq T_1 (y_t) - y_t , \label{eq: General RVI Q aux ode}\\
    \dot x_t & \doteq T_2 (x_t) - x_t = T_1(x_t) - x_t +  \left ( r_\infty - f(x_t) \right)e. \label{eq: General RVI Q original ode}
\end{align}

Note that because $g$ is a non-expansion by \cref{assu: g}, both \eqref{eq: General RVI Q aux ode} and \eqref{eq: General RVI Q original ode} have Lipschitz R.H.S.’s and thus are well-posed.

Because $g$ is a non-expansion, $T_1$ is also a non-expansion. Therefore we have the next lemma, which restates Theorem 3.1 and Lemma 3.2 by Borkar and Soumyanath (1997).
\begin{lemma}\label{lemma: General RVI Q aux ode convergence}
Let $\bar y$ be an equilibrium point of \eqref{eq: General RVI Q aux ode}. Then $\norm{y_t - \bar y}_\infty$ is nonincreasing, and $y_t \to y_*$ for some equilibrium point $y_*$ of \eqref{eq: General RVI Q aux ode} that may depend on $y_0$.
\end{lemma}

\begin{lemma} \label{lemma: General RVI Q unique equilibrium}
\eqref{eq: General RVI Q original ode} has a unique equilibrium at $q_\infty$.
\end{lemma}

\begin{proof}
Because $f(q_\infty) = r_\infty$, we have that $q_\infty = T_1(q_\infty) = T_2(q_\infty)$,
thus $q_\infty$ is a equilibrium point for \eqref{eq: General RVI Q original ode}.
Conversely, if $T_2 (Q) - Q = 0$, then $T_1 Q + (r_\infty - f(Q)) e = Q$. But the equation $T_1Q + ce = Q$ only has a solution when $c = 0$ because of Assumption~\ref{assu: g}. We have $c = 0$ and thus $f(Q) = r_\infty$, which along with $T_1 Q = Q$, implies $Q = q_\infty$.
\end{proof}

\begin{lemma}\label{lemma: General RVI Q connection between original and aux ode}
Let $x_0 = y_0$, then $x_t= y_t + z_t e$, where $z_t$ satisfies the ODE $\dot z_t= - u z_t + (r_\infty - f(y_t))$, and $k \doteq \abs{ \calI}$.
\end{lemma}
\begin{proof}
From \eqref{eq: General RVI Q aux ode}, \eqref{eq: General RVI Q original ode},  by the variation of parameters formula, 
\begin{align*}
    x_t &= \exp(-t) x_0 + \int_0^t \exp(\tau - t) T_1 (x_\tau) d \tau + \left [ \int_0^t \exp(\tau - t) \left(r_\infty - f(x_\tau) \right) d\tau \right ] e, \\
    y_t &= \exp(-t) y_0 + \int_0^t \exp(\tau - t) T_1 (y_\tau) d\tau .
\end{align*}
Then we have
\begin{align*}
    & \max_{s, o} (x_t(s, o) - y_t(s, o)) \\
    & \leq \int_0^t \exp(\tau - t) \max_{s, o}(T_1 (x_\tau)(s, o) - T_1 (y_\tau) (s, o)) d\tau + \left[\int_0^t \exp(\tau - t) \left (r_\infty - f(x_\tau) \right) d\tau \right], \\
    & \min_{s, o} (x_t(s, o) - y_t(s, o)) \\
    & \geq \int_0^t \exp(\tau - t) \min_{s, o} (T_1 (x_\tau)(s, o) - T_1 (y_\tau) (s, o)) d\tau +  \left[\int_0^t \exp(\tau - t) \left(r_\infty - f(x_\tau) \right) d\tau \right].
\end{align*}
Subtracting, we have
\begin{align*}
    sp(x_t - y_t) & \leq \int_0^t \exp(\tau - t) sp (T_1 (x_\tau) - T_1 (y_\tau)) d\tau,
\end{align*}
where $sp(x)$ denotes the span of vector $x$.

Because we assumed that $g$ is span-norm non-expansion, $T_1$ is also a span-norm non-expansion and thus
\begin{align*}
    sp(x_t - y_t) \leq \int_0^t \exp(\tau - t) sp (T_1 (x_\tau) - T_1 (y_\tau)) d\tau \leq \int_0^t \exp(\tau - t) sp(x_\tau - y_\tau) d\tau.
\end{align*}

By Gronwall’s inequality, $sp(x_t - y_t) = 0$ for all $t \geq 0$. Because $sp(x) = 0$ if and only if $x = ce$ for some $c \in \bbR$, we have 
\begin{align*}
    x_t = y_t + z_t e, \quad t \geq 0.
\end{align*}
for some $z_t$. Also $x_0 = y_0 \implies z_0 = 0$.

Now we show that $\dot z_t= - u z_t + (r_\infty - f(y_t) )$. Note that $f(x_t) = f(y_t + z_t e) = f(y_t) + u z_t$. In addition, $T_1(x_t) - T_1(y_t) = T_1(y_t + z_t e) - T_1(y_t) = T_1(y_t) + z_t e - T_1(y_t) = z_t e$, therefore we have, for $z_t \in \bbR$:
\begin{align*}
    \dot z_t e & = \dot x_t - \dot y_t\\
    & = \left(T_1 (x_t) - x_t + \left(r_\infty - f(x_t) \right) e \right) - ( T_1 (y_t) - y_t) \quad \text{(from \eqref{eq: General RVI Q aux ode} and \eqref{eq: General RVI Q original ode})}\\
    & = - (x_t - y_t) + (T_1 (x_t) - T_1(y_t)) + \left(r_\infty - f(x_t) \right) e \quad \\
    & = - z_t e + z_t e + \left(r_\infty - f(x_t) \right) e\\
    & = - u z_t e + u z_t e + \left (r_\infty - f(x_t) \right) e\\
    & = - u z_t e + \left (r_\infty - f(y_t) \right) e\\
    \implies \dot z_t &= - u z_t + \left(r_\infty - f(y_t) \right).
\end{align*}
\end{proof}

\begin{lemma} \label{lemma: General RVI Q globally asymptotically stable equilibrium}
$q_\infty$ is the globally asymptotically stable equilibrium for \eqref{eq: General RVI Q original ode}.
\end{lemma}

\begin{proof}
We have shown that $q_\infty$ is the unique equilibrium in Lemma \ref{lemma: General RVI Q unique equilibrium}.

With that result, we first prove Lyapunov stability. That is, we need to show that given any $\epsilon > 0$, we can find a $\delta > 0$ such that $\norm{q_\infty - x_0}_\infty \leq \delta$ implies $\norm{q_\infty - x_t}_\infty \leq \epsilon$ for $t \geq 0$. 

First, from Lemma \ref{lemma: General RVI Q connection between original and aux ode} we have $\dot z_t= - u z_t + (r_\infty - f(y_t))$. By variation of parameters and $z_0 = 0$, we have
\begin{align*}
    z_t = \int_0^t \exp(u (\tau - t)) \left (r_\infty - f(y_\tau) \right) d\tau .
\end{align*}
Then
\begin{align}
    \norm{q_\infty - x_t}_\infty & = \norm{q_\infty - y_t - z_t ue}_\infty \nonumber\\ 
    & \leq \norm{q_\infty - y_t}_\infty + u \abs{z_t}  \nonumber\\
    & \leq \norm{q_\infty - y_0}_\infty + u \int_0^t \exp(u (\tau - t)) \abs{r_\infty - f(y_\tau)} d\tau \nonumber\\
    & = \norm{q_\infty - x_0}_\infty + u \int_0^t \exp(u (\tau - t)) \abs{f(q_\infty) - f(y_\tau)} d\tau \quad \text{(from \eqref{eq: General RVI Q relation between f(q_infty) and r_infty})} \label{eq: globally asymptotically stable equilibrium lemma: eq 1}.
\end{align}

Because $f$ is $L$-lipschitz, we have
\begin{align*}
    \abs{ f(q_\infty) - f( y_\tau)} & \leq L \norm{r_\infty - y_\tau}_\infty \\
    & \leq L \norm{r_\infty - y_0}_\infty \quad \text{(from Lemma \ref{lemma: General RVI Q aux ode convergence})} \\
    & = L \norm{r_\infty - x_0}_\infty.
\end{align*}

Therefore
\begin{align*}
    \int_0^t \exp(u (\tau - t)) \abs{ f(q_\infty) - f(y_\tau) } d\tau & \leq \int_0^t \exp(u (\tau - t)) L \norm{q_\infty - x_0 }_\infty d\tau \\
    & = L \norm{q_\infty - x_0 }_\infty \int_0^t \exp(u (\tau - t)) d\tau \\
    & = L \norm{q_\infty - x_0 }_\infty \frac{1}{u }(1 - \exp(-u t)) \\
    & = \frac{L}{u }\norm{q_\infty - x_0 }_\infty (1 - \exp(-u t)).
\end{align*}

Substituting the above equation in \eqref{eq: globally asymptotically stable equilibrium lemma: eq 1}, we have
\begin{align*}
    \norm{q_\infty - x_t}_\infty \leq (1 + L) \norm{q_\infty - x_0}_\infty.
\end{align*}

Lyapunov stability follows.

Now in order to prove the asymptotic stability, in addition to Lyapunov stability, we need to show that there exists $\delta >0$ such that if $\norm{x_0 - q_\infty}_\infty < \delta$ , then $\lim_{t \to \infty } \norm{x_t - q_\infty}_\infty=0$. Note that
\begin{align*}
    \lim_{t \to \infty} z_t & = \lim_{t \to \infty} \int_0^t \exp(u (\tau - t)) \left(r_\infty - f(y_\tau) \right) d\tau\\
    & = \lim_{t \to \infty} \frac{\int_0^t \exp(u \tau) ( r_\infty - f(y_\tau) ) d\tau}{\exp(ut)} \\
    & = \lim_{t \to \infty} \frac{\exp(u t) (r_\infty - f(y_t))}{u \exp(ut)} \quad \text{(by L'Hospital's rule)}\\
    & = \frac{r_\infty - f(y_\infty)}{u} \quad \text{(by Lemma \ref{lemma: General RVI Q aux ode convergence})} .
\end{align*}
Because $x_t = y_t + z_t e$ (Lemma \ref{lemma: General RVI Q connection between original and aux ode}) and $y_t \to y_\infty$ (Lemma \ref{lemma: General RVI Q aux ode convergence}), we have $x_t \to y_\infty + (r_\infty - f(y_\infty)) e / u$, which must coincide with $q_\infty$ because that is the only equilibrium point for \eqref{eq: General RVI Q original ode} (Lemma \ref{lemma: General RVI Q unique equilibrium}). Therefore $\lim_{t \to \infty} \norm{x_t - q_\infty}_\infty = 0$ for any $x_0$. Asymptotic stability is shown and the proof is complete.
\end{proof}

\begin{lemma}\label{lemma: General RVI Q synchronous update convergence}
Equation~\ref{eq: General RVI Q sync update} converges a.s. $Q_{n}$ to $q_\infty$ 
as $n \to \infty$.
\end{lemma}

\begin{proof}
The proof uses Borkar's (2008) Theorem 2 in Section 2 and is essentially the same as Lemma 3.8 by Abounadi et al.~(2001). For completeness, we repeat the proof (with more details) here.

First write the synchronous update \eqref{eq: General RVI Q sync update} as 
\begin{align*}
    Q_{n+1} = Q_{n} + \alpha_n (h(Q_n) + M_{n+1} + \epsilon_n),
\end{align*}
where 
\begin{align*}
    h(Q_n)(i) & \doteq r(i) - f(Q_n) + g(Q_n)(i) - Q_n(i)\\
    & = T_2(Q_n)(i) - Q_n(i) ,\\
    M_{n + 1}(i) & \doteq R_n(i) - F_n(Q_n)(i) + G_n(Q_n)(i) - T_2(Q_n)(i).
\end{align*}

It can be shown that $\epsilon_n$ is asymptotically negligible and therefore does not affect the conclusions of Theorem 2 (text after Equation B.66 by Wan et al.~2021). 

Theorem 2 requires verifying following conditions and concludes that $Q_n$ converges
to a (possibly sample path dependent) compact connected internally chain transitive invariant set of ODE $\dot x_t = h(x_t)$. This is exactly the ODE defined in \eqref{eq: General RVI Q original ode}. Lemma \ref{lemma: General RVI Q unique equilibrium} and \ref{lemma: General RVI Q globally asymptotically stable equilibrium} conclude that this ODE has $q_\infty$ as the unique globally asymptotically stable equilibrium. Therefore the (possibly sample path dependent) compact connected internally chain transitive invariant set is a singleton set containing only the unique globally asymptotically stable equilibrium. Thus Theorem 2 concludes that $Q_n \to q_\infty$ a.s. as $n \to \infty$. We now list conditions required by Theorem 2:

\begin{itemize}
    \item \textbf{(A1)} The function $h$ is Lipschitz: $\norm{h(x) - h(y)} \leq L \norm{x - y}$ for some $0 < L < \infty$.
    \item \textbf{(A2)} The sequence $\{ \alpha_n\}$ satisfies $ \alpha_n > 0$, and $\sum \alpha_n = \infty$, $\sum \alpha_n^2 < \infty$.
    \item \textbf{(A3)} $\{M_n\}$ is a martingale difference sequence with respect to the increasing family of $\sigma$-fields 
        \begin{align*}
            \calF_n \doteq \sigma(Q_i, M_i,i \leq n), n \geq 0.
        \end{align*}
        That is
        \begin{align*}
            \bbE[M_{n+1} \mid \calF_n] = 0 \text{ \ \ a.s., } n \geq 0.
        \end{align*}
        Furthermore, $\{M_n\}$ are square-integrable 
        \begin{align*}
            \bbE[\norm{M_{n+1}}^2 \mid \calF_n] \leq K(1 + \norm{Q_n}^2) \text{ \ \ a.s., \ \ } n \geq 0,
        \end{align*}
        for some constant $K > 0$.
    \item \textbf{(A4)} $\sup_n \norm{Q_n} \leq \infty$ a.s..
\end{itemize}

Let us verify these conditions now.

(A1) is satisfied because $T_2$ is Lipschitz.

(A2) is satisfied by \Cref{assu: stepsize}.

(A3) is also satisfied because for any $i \in \calI$
\begin{align*}
    \bbE[M_{n+1}(i) \mid \calF_n] & = \bbE \left[R_n(i) - F_n(Q_n)(i) + G_n(i) - T_2 (Q_n)(i) \mid \calF_n \right] \\
    & = \bbE \left [R_n(i) - F_n(Q_n)(i) + G_n(Q_n)(i) \mid \calF_n \right] - T_2 (Q_n)(i) \\
    & = 0,
\end{align*}
and $\bbE [\norm{M_{n+1}}^2 \mid \calF_n] \leq \bbE [\norm{R_n - r}^2 \mid \calF_n] + \bbE [\norm{F_n(Q_n) - f(Q_n)e}^2 \mid \calF_n] + \bbE [\norm{G_n(Q_n) - g(Q_n)}^2 \mid \calF_n] \leq K (1 + \norm{Q_n}^2)$ for a suitable constant $K > 0$ can be verified by a simple application of triangle inequality.

To verify (A4), we apply Theorem 7 in Section 3 by Borkar (2008), which shows $\sup_n \norm{Q_n} \leq \infty$ a.s., if (A1), (A2), and (A3) are all satisfied and in addition we have the following condition satisfied:

\textbf{(A5)} The functions $h_d(x) \doteq h(dx)/d$, $d \geq 1, x \in \bbR^{k}$, satisfy $h_d(x) \to h_\infty(x)$ as
$d \to \infty$, uniformly on compacts for some $h_\infty \in C(\bbR^k)$. Furthermore, the ODE $\dot x_t = h_\infty(x_t)$ has the origin as its unique globally asymptotically stable equilibrium.

Note that
\begin{align*}
    h_\infty(x) = \lim_{d \to \infty} h_d(x) = \lim_{d \to \infty} \left (T_2(dx) - dx \right) / d = g(x) - f(x)  e - x,
\end{align*}
because $g(cx) = cg(x)$ and $f(cx) = cf(x)$ by our assumption.

The function $h_\infty$ is clearly continuous in every $x \in \bbR^k$ and therefore $h_\infty \in C(\bbR^k)$.

Now consider the ODE $\dot x_t = h_\infty(x_t) = g (x_t) - f(x_t) e - x_t$. Clearly the origin is an equilibrium. This ODE is a special case of \eqref{eq: General RVI Q original ode}, corresponding to the $r(s, o) \forall s \in \calS, o \in \calO$ being always zero. Therefore Lemma \ref{lemma: General RVI Q unique equilibrium} and \ref{lemma: General RVI Q globally asymptotically stable equilibrium} also apply to this ODE and the origin is the unique globally asymptotically stable equilibrium.

(A1), (A2), (A3), (A4) are all verified and therefore 
\begin{align*}
    Q_n \to q_\infty \text{\ \  a.s. as \ \  } n \to \infty.
\end{align*}
\end{proof}
\end{proof}

\subsection{Theorem~\ref{thm: Inter-option Differential Methods}} \label{app: inter option}

For simplicity, we will only provide formal theorems and proofs for our \emph{control} learning and planning algorithms. The formal theorems and proofs for our prediction algorithms are similar to those for the control algorithms and are thus omitted. To this end, we first provide a general algorithm that includes both learning and planning control algorithms. We call it \emph{General Inter-option Differential Q}. We first formally define it and then explain why both inter-option Differential Q-learning and inter-option Differential Q-planning are special cases of General Inter-option Differential Q. We then provide assumptions and the convergence theorem of the general algorithm. The theorem would lead to convergence of the special cases. Finally, we provide a proof for the theorem.

Given an SMDP $\hat \calM = (\calS, \calO, \hat \calL, \hat \calR, \hat p)$, for each state $s \in \calS $, option $o \in \calO$, and discrete step $n \geq 0$, let $\hat R_n(s, o),  \hat S'_n(s, o), \hat L_n(s, o)  \sim \hat p(\cdot, \cdot, \cdot|s, o)$ denote a sample of resulting state, reward and the length. 
We hypothesize a set-valued process $\{Y_n\}$ taking
values in the set of nonempty subsets of $\calS \times \calO$ with the interpretation: $Y_n = \{(s, o): (s, o)$ component of $Q$ was updated at time $n\}$. Let $\nu(n, s, o) \doteq \sum_{k=0}^n I\{(s, o) \in Y_k\}$, where $I$ is the indicator function. Thus $\nu(n, s, o) =$ the number of times the $(s, o)$ component was updated up to step $n$. The update rules of General Inter-option Differential Q are 
\begin{align}
    Q_{n+1}(s, o) & \doteq Q_n(s, o) + \alpha_{\nu(n, s, o)} \delta_n(s, o) / L_n(s, o) I\{(s, o) \in Y_n\}, \quad \forall s \in \calS, o \in \calO, \label{eq: General Inter-option Diff Q async update Q}\\
    \bar R_{n+1} & \doteq \bar R_n + \eta \sum_{s, o} \alpha_{\nu(n, s, o)} \delta_n(s, o) / L_n(s, o) I\{(s, o) \in Y_n\}, \label{eq: General Inter-option Diff Q async update bar R}\\
    L_{n+1}(s, o) &\doteq L_{n}(s, o) + \beta_n(s, o) (\hat L_n(s, o) - L_{n}(s, o)) I\{(s, o) \in Y_n\},
\end{align}
where
\begin{align}
    \delta_n(s, o) & \doteq \hat R_n(s, o) - \bar R_n L_n(s, o) + \max_{o'} Q_n(\hat S_n'(s, o), o')  - Q_n(s, o)  \label{eq: General Inter-option Diff Q TD error}
\end{align}
is the TD error.

Here $\alpha_{\nu(n, s, o)}$ is the stepsize at step $n$ for state-action pair $(s, o)$. The quantity $\alpha_{\nu(n, s, o)}$ depends on the sequence $\{\alpha_n\}$, which is an algorithmic design choice, and also depends on the visitation of state-option pairs $\nu(n, s, o)$. To obtain the stepsize, the algorithm could maintain a $\vert \calS \times \calO \vert$-size table counting the number of visitations to each state-option pair, which is exactly $\nu(\cdot, \cdot, \cdot)$. Then the stepsize $\alpha_{\nu(n, s, o)}$ can be obtained as long as the sequence $\{\alpha_n\}$ is specified.

$Q_0$ and $R_0$ can be initialized arbitrarily. Note that $L_0$ can not be initialized to $0$ because it is the divisor for both \eqref{eq: General Inter-option Diff Q async update Q} and \eqref{eq: General Inter-option Diff Q async update bar R} for the first update. Because the expected length of all options would be greater than or equal to 1, we choose $L_0$ to be 1. In this way, $L_n$ will never be 0 because it is initialized to 1 and all the sampled option lengths are greater than or equal to 1. Therefore the problem of division by 0 will not happen in the updates.

Now we show inter-option Differential Q-learning and inter-option Differential Q-planning are special cases of General Inter-option Differential Q. Consider a sequence of real experience $\ldots, \hat S_n, \hat O_n, \hat R_{n}, \hat L_n, \hat S_{n+1}, \ldots$. 
\begin{align*}
    Y_n(s, o) &= 1 \text{, if } s = \hat S_n, o = \hat O_n,\\
    Y_n(s, o) &= 0 \text{ otherwise,}
\end{align*}
and $\hat S'_n(\hat S_n, \hat O_n) = \hat S_{n+1}, \hat R_n(\hat S_n, \hat O_n) = \hat R_{n+1}, \hat L_n(\hat S_n, \hat O_n) = \hat L_n$, update rules \eqref{eq: General Inter-option Diff Q async update Q}, \eqref{eq: General Inter-option Diff Q async update bar R}, and \eqref{eq: General Inter-option Diff Q TD error} become
\begin{align*}
    Q_{n+1}(\hat S_n, \hat O_n) & \doteq Q_n (\hat S_n, \hat O_n) + \alpha_{\nu(n, \hat S_n, \hat O_n)} \hat \delta_n / L_n(\hat S_n, \hat O_n) \text{\ , and\ } Q_{n+1}(s, o) \doteq Q_n (s, o), \forall s \neq \hat S_n, o \neq \hat O_n, \\
    \bar R_{n+1} & \doteq \bar R_n + \eta \alpha_{\nu(n, \hat S_n, \hat O_n)}  \hat \delta_n / L_n(\hat S_n, \hat O_n),\\
    \hat \delta_n & \doteq \hat R_n - \bar R_n \hat L_n + \max_{o'} Q_n(\hat S_{n+1}, o') - Q_n(\hat S_n, \hat O_n),\\
    L_{n+1}(\hat S_n, \hat O_n) &\doteq L_{n}(\hat S_n, \hat O_n) + \beta_n(\hat S_n, \hat O_n) (\hat L_n - L_{n}(\hat S_n, \hat O_n))
\end{align*}
which are inter-option Differential Q-learning's update rules (\cref{sec: inter-option}) with stepsize $\alpha$ in the $n$-th update being $\alpha_{\nu(n, \hat S_n, \hat O_n)}$, and the stepsize $\beta$ being $\beta(\hat S_n, \hat O_n)$.

If we consider a sequence of simulated experience $\ldots, \tilde S_n, \tilde O_n, \tilde R_{n}, \tilde L_n, \tilde S_{n}', \ldots$. 
\begin{align*}
    Y_n(s, o) &= 1 \text{, if } s = \tilde S_n, o = \tilde O_n,\\
    Y_n(s, o) &= 0 \text{ otherwise,}
\end{align*}
and $\hat S'_n(s, o) = \tilde S_{n}', \hat R_n(s, o) = \tilde R_{n}, \hat L_n(s, o) = \tilde L_n$, update rules \eqref{eq: General Inter-option Diff Q async update Q}-\eqref{eq: General Inter-option Diff Q TD error} become
\begin{align*}
    Q_{n+1}(\tilde S_n, \tilde O_n) & \doteq Q_n (\tilde S_n, \tilde O_n) + \alpha_{\nu(n, \tilde S_n, \tilde O_n)} \tilde \delta_n / L_n \text{\ , and\ } Q_{n+1}(s, o) \doteq Q_n (s, o), \forall s \neq \tilde S_n, o \neq \tilde O_n, \\
    \bar R_{n+1} & \doteq \bar R_n + \eta \alpha_{\nu(n, \tilde S_n, \tilde O_n)}  \tilde \delta_n / L_n,\\
    \tilde \delta_n & \doteq \tilde R_n - \bar R_n \tilde L_n + \max_{o'} Q_n(\tilde S_{n}', o') - Q_n(\tilde S_n, \tilde O_n),\\
    L_{n+1}(\tilde S_n, \tilde O_n) &\doteq L_{n}(\tilde S_n, \tilde O_n) + \beta_n(\tilde S_n, \tilde O_n) (\tilde L_n - L_{n}(\tilde S_n, \tilde O_n)).
\end{align*}
Now, in the planning setting, the model can produce an expected length, instead of a sampled one. And there estimating the expected length using $L_n$ is no longer needed. The above updates reduce to 
\begin{align*}
    Q_{n+1}(\tilde S_n, \tilde O_n) & \doteq Q_n (\tilde S_n, \tilde O_n) + \alpha_{\nu(n, \tilde S_n, \tilde O_n)} \tilde \delta_n / \tilde L_n \text{\ , and\ } Q_{n+1}(s, o) \doteq Q_n (s, o), \forall s \neq \tilde S_n, o \neq \tilde O_n, \\
    \bar R_{n+1} & \doteq \bar R_n + \eta \alpha_{\nu(n, \tilde S_n, \tilde O_n)}  \tilde \delta_n / \tilde L_n,\\
    \tilde \delta_n & \doteq \tilde R_n - \bar R_n \tilde L_n + \max_{o'} Q_n(\tilde S_{n}', o') - Q_n(\tilde S_n, \tilde O_n).
\end{align*}
The above update rules are our inter-option Differential Q-planning's update rules with stepsize at planning step $n$ being $\alpha_{\nu(n, \tilde S_n, \tilde O_n)}$.

We now provide a theorem, along with its proof, showing the convergence of General Inter-option Differential Q.

\begin{theorem}
Under Assumptions~\ref{assu: unichain}, \ref{assu: stepsize}, \ref{assu: asynchronous stepsize 1}, \ref{assu: asynchronous stepsize 2}, and that $0 \leq \beta_n(s, o) \leq 1$, $\sum_n \beta_n(s, o) = \infty$, and $\sum_n \beta_n^2(s, o) < \infty$, and $\beta_n(s, o) = 0$ unless $s = \hat S_n$, General Inter-option Differential Q (Equations~\ref{eq: General Inter-option Diff Q async update Q}-\ref{eq: General Inter-option Diff Q TD error}) converges, almost surely,
$Q_n$ to $q$ satisfying both \eqref{eq: SMDP Bellman optimality equation}
and
\begin{align*}
    \eta  (\sum q - \sum Q_0) = r_* - \bar R_0,
\end{align*}
$\bar R_n$ to $r_*$, and $r(\mu_n)$ to $r_*$ where $\mu_n$ is a greedy policy w.r.t. $Q_n$.
\end{theorem}

\begin{proof}
At each step, the increment to $\bar R_n$ is $\eta$ times the increment to $Q_n$ and $\sum Q_n$. Therefore, the cumulative increment can be written
\begin{align}
    \bar R_n - \bar R_0 &= \eta  \sum_{i = 0}^{n-1} \sum_{s, o} \alpha_{\nu(i, s, o)} \delta_i(s, o)/L_i(s, o) I\{(s, o) \in Y_i\} \nonumber  \\
    & = \eta \left(\sum Q_n - \sum Q_0 \right) \nonumber \\
    \implies \bar R_n &= \eta \sum Q_n - \eta \sum Q_0 + \bar R_0 = \eta \sum Q_n - c \label{eq: General Inter-option Diff Q relation between bar R and Q} ,\\
    \text{ where } c &\doteq \eta \sum Q_0 - \bar R_0.
\end{align}

Now substituting $\bar R_n$ in \eqref{eq: General Inter-option Diff Q async update Q} with \eqref{eq: General Inter-option Diff Q relation between bar R and Q}, we have $\forall s \in \calS, o \in \calO$: 
\begin{align}
    & Q_{n+1}(s, o) = Q_{n}(s, o) + \alpha_{\nu(n, s, o)} \nonumber \\
    & \frac{\hat R_n(s, o) - L_n(s, o)(\eta \sum Q_n - c) + \max_{o'} Q_n(\hat S_n'(s, o), o') - Q_n(s, o) }{L_n(s, o)} I\{(s, o) \in Y_n\}  \nonumber\\
    & = Q_{n}(s, o) + \alpha_{\nu(n, s, o)} \nonumber \\
    & \left( \frac{\hat R_n(s, o) - l_n(s, o)(\eta \sum Q_n - c) + \max_{o'} Q_n(\hat S_n'(s, o), o') - Q_n(s, o) }{l(s, o)} + \epsilon_n(s, o) \right ) I\{(s, o) \in Y_n\},
    \label{eq: General Inter-option Diff Q transformed async single update}
\end{align}
where $l(s, o)$ is the expected length of option $o$, starting from state $s$, and $\epsilon_n(s, o) \doteq (\hat R_n(s, o) - L_n(s, o)(\eta \sum Q_n - c) + \max_{o'} Q_n(\hat S_n'(s, o), o') - Q_n(s, o) ) / L(s, o) - (\hat R_n(s, o) - l(s, o)(\eta \sum Q_n - c) + \max_{o'} Q_n(\hat S_n'(s, o), o') - Q_n(s, o))/l(s, o)$.

Standard stochastic approximation result can be applied to show that $L_n$ converges to $l$. Further, it can be seen that $\epsilon_n$ satisfies that $\norm{\epsilon_n}_\infty \leq K (1 + \norm{Q_n})$ for some positive $K$ and, by continuous mapping theorem, converges to 0 almost surely (and thus in probability).

We now show that \eqref{eq: General Inter-option Diff Q transformed async single update} is a special case of \eqref{eq: General RVI Q async update}. To see this point, let 
\begin{align*}
    i & = (s, o),\\
    R_n(i) &= \frac{\hat R_n(s, o)}{l(s, o)} + c,\\
    G_n(Q_n)(i) &= \frac{ \max_{o'}Q_n(\hat S_n'(s, o), o')}{l(s, o)} + \frac{l(s, o) - 1}{l(s, o)} Q_n(s, o),\\
    F(Q_n)(i) &= \eta \sum Q_n, \\
    \epsilon_n(i) &= \epsilon_n(s, o).
\end{align*}
We now verify the assumptions of Theorem~\ref{thm: General RVI Q} for Inter-option General Differential Q. \Cref{assu: g} and \Cref{assu: f} can be verified easily.  \Cref{assu: variance of Martingale difference} satisfies because the MDP is finite. \cref{assu: epsilon} is satisfied as shown above. Assumption~\ref{assu: stepsize}-\ref{assu: asynchronous stepsize 2} are satisfied due to assumptions of the theorem being proved. \Cref{assu: General RVI Q unique solution} is satisfied because
\begin{align*}
    &r(i) - \bar r + g(q)(i) - q(i) \\
    & = \bbE[R_n(i) - \bar r + G_n(q)(i) - q(i)] \\
    & = \bbE\left[\frac{\hat R_n(s, o) + c l(s, o) - \bar r l(s, o) + \max_{o'} q(\hat S_n'(s, o), o') +  (l(s, o) - 1)q(s, o) - l(s, o) q(s, o)}{l(s, o)} \right]\\
    & = \frac{\bbE \left[\hat R_n(s, o) + c l(s, o) - \bar r l(s, o) + \max_{o'} q(\hat S_n'(s, o), o') - q(s, o) \right]}{l(s, o)} .
\end{align*}
From \eqref{eq: SMDP Bellman optimality equation} we know if the above equation equals to 0, then under \cref{assu: unichain}, $\bar r = r_* + c$ is the unique solution and the solutions for $q$ form a set $q = q_* + ce$.

All the assumptions are verified and thus from \cref{thm: General RVI Q} we conclude that $Q_n$ converges to a point satisfying $\eta \sum q = r_* + c = r_* + \eta \sum Q_0 - \bar R_0$ and $\bar R_n = \eta \sum Q_n - c$ to $\eta \sum q - c = r_* + c - c = r_*$.

Finally, in order to show $r(\mu_n) \to r_*$, we first extend Theorem 8.5.5 by Puterman (1994) to deal with SMDP.
\begin{lemma}[]\label{lemma: bound of mu Q} Under \cref{assu: unichain}, $\forall Q \in \bbR^{\abs{\calS \times \calO}}$
\begin{align*}
    \min_{s, o} TQ(s, o) \leq r(\mu_{Q}) \leq r_* \leq \max_{s, o} TQ(s, o),
\end{align*}
where $TQ(s, o) \doteq \sum_{s', r, l} \hat p(s', r, l \mid s, o) (r + \max_{o'} Q(s', o'))$ and $\mu_Q$ denotes a greedy policy w.r.t. $Q$.
\end{lemma}
\begin{proof}
Note that
\begin{align*}
    r(\mu_{Q}) = \sum_{s', r, l} \hat p(s', r, l \mid s, o) (r + \sum_{o'} \mu_Q (o' \mid s') Q(s', o') - Q(s, o)).
\end{align*}

Therefore
\begin{align*}
    \min_{s, o} (TQ_n(s, o) - Q_n(s, o)) \leq r(\mu_n) \leq r_* \leq \max_{s, o} (TQ_n(s, o) - Q_n(s, o)) \\
    \implies
    \abs{r_* - r(\mu_n)} \leq sp(TQ_n - Q_n).
\end{align*}
\end{proof}

Because $Q_n \to q_\infty$ a.s., and $sp(TQ_n - Q_n)$ is a continuous function of $Q_n$, by continuous mapping theorem, $sp(TQ_n - Q_n) \to sp(Tq_\infty - q_\infty) = 0$ a.s. Therefore we conclude that $r(\mu_n) \to r_*$.
\end{proof}

The convergence of General Inter-option Differential Q that we showed above implies Theorem~\ref{thm: Inter-option Differential Methods} when there are no transient states ($\calS' = \calS$) and thus all states can be visited for an infinite number of times. When $\calS' \subset \calS$, option values associated states in $\calS - \calS'$ do not converge to a solution of the Bellman equation. However, the option values associated with recurrent states $\calS'$ still converge to a solution of the Bellman equation, the reward rate estimator converges to $r_*$, and the $r(\mu_n)$ converges to $r_*$. The point that option values (associated with recurrent states) converge to depends on the sample trajectory. Specifically, it depends on the transient states visited in the trajectory.

\subsection{Theorem~\ref{thm: intra-option bellman equations}}\label{app: proof of intra-option value equations}

The proof for the intra-option evaluation equation is simple. First note that these equations can be written in the vector form:
\begin{align*}
    0 = r - \bar r e + (P_\mu - I) q,
\end{align*}
where $r(s, o) = \bbE[R_{t+1} \mid S_t = s, O_t = o] $, $P_\mu((s, o), (s', o')) \doteq \Pr(S_{t+1} = s', O_{t+1} = o' | S_t = s, O_t = o, \mu) = \beta(s', o) \mu(o' \mid s') + (1 - \beta(s', o)) \bbI(o = o')$, and $e$ is a all-one vector.
Intuitively, the intra-option evaluation equation can be viewed as the evaluation equation for some average-reward MRP with reward and dynamics being defined as $r$ and $P_\mu$.

By Theorem 8.2.6 and Corollary 8.2.7 in Puterman (1994), the intra-option evaluation equation part in \cref{thm: intra-option bellman equations} is shown as long as the Markov chain associated with $P_\mu$ is unichain. 
Note that by our \cref{assu: unichain}, there is only one recurrent class of states under any policy. Therefore no matter what the start state-option pair is, the agent will enter in the same recurrent class of states. Therefore we have, for every state $\bar s$ in the recurrent class and an option $\bar o$ such that $\mu(\bar o \mid \bar s) > 0$, the MDP visits $(\bar s, \bar o)$ an an infinite number of times. This shows that any two state-option pairs can not be in two separate recurrent sets of state-option pairs. Therefore the Markov chain associated with $P_\mu$ is unichain.

The proof for the Intra-option Optimality Equations is more complicated. First, similar as what we know in the discounted primitive action case, we have the following lemma for the discounted option case.

\begin{lemma}\label{existence of stationary discount optimal policy}
For every $0 < \gamma < 1$, there exists a stationary deterministic discount optimal policy.
\end{lemma}

The proof uses similar arguments as Theorem 6.2.10 and Proposition 4.4.3 by Puterman (1994).

Now choose a sequence of discount factors $\{\gamma_n\}$, $0 \leq \gamma_n < 1$ with  the  property that $\gamma_n \uparrow 1$. By lemma \ref{existence of stationary discount optimal policy}, for each $\gamma_n$, there exists a stationary discount optimal policy. Because the total number of Markov deterministic policies is finite, we can choose a subsequence $\{\gamma'_n\}$ for which the same Markov deterministic policy, $\mu$, is discount optimal for all $\gamma'_n$. Denote this subsequence by $\{\gamma_n\}$. Because $\mu$ is discount optimal for $\gamma_n, \forall n$, $q_*^{\gamma_n} = q_{\mu}^{\gamma_n}, \forall n$. By intra-option optimality equations in the discounted case (Sutton et al., 1999), for all $s \in \calS,  o \in \calO$, 
\begin{align}
    0 & = \sum_{a} \pi(a | s, o)\sum_{s', r} p(s', r | s, a) \left(r + \gamma_n \beta(s', o) q_{\mu}^{\gamma_n}(s', \mu(s')) + \gamma_n (1 - \beta(s', o)) q_{\mu}^{\gamma_n}(s', o)\right) - q_{\mu}^{\gamma_n}(s, o) \nonumber\\
    & = \sum_{a} \pi(a | s, o)\sum_{s', r} p(s', r | s, a) \left (r + \gamma_n \beta(s', o) \max_{o'} q_\mu^{\gamma_n}(s', o') + \gamma_n (1 - \beta(s', o)) q_\mu^{\gamma_n}(s', o) \right ) - q_\mu^{\gamma_n}(s, o). \label{average reward intra-optimality equation proof 2}
\end{align}
By corollary 8.2.4 by Puterman (1994),
\begin{align} \label{laurent series expansion}
    q_{\mu}^{\gamma_n} = (1 - \gamma_n)^{-1} r(\mu) e + q_{\mu} + f(\gamma_n),
\end{align}
where $r(\mu)$ and $q_{\mu}$ are the reward rate and differential value function under policy $\mu$, and $f(\gamma)$ is a function of $\gamma$ that converges to 0 as $\gamma \uparrow 1$.

Substituting \eqref{laurent series expansion} into \eqref{average reward intra-optimality equation proof 2}, we have
\begin{align*}
    0 & = \sum_a \pi(a | s, o) \sum_{s', r} p(s', r | s, a) (r + \gamma_n \beta(s', o) \max_{o'} [(1 - \gamma_n)^{-1} r(\mu) + q_\mu(s', o') + f(\gamma_n, s', o')] \\ 
    & + \gamma_n (1 - \beta(s', o)) [(1 - \gamma_n)^{-1} r(\mu) + q_\mu(s', o) + f(\gamma_n, s', o)]) \\
    & - [(1 - \gamma_n)^{-1} r(\mu) + q_\mu(s, o) + f(\gamma_n, s, o)] \\
    & = \sum_a \pi(a | s, o) \sum_{s', r} p(s', r | s, a) (r - r(\mu) + \gamma_n \beta(s', o) \max_{o'} [q_\mu(s', o') + f(\gamma_n, s', o')] \\ 
    & + \gamma_n (1 - \beta(s', o)) [q_\mu(s', o) + f(\gamma_n, s', o)]) \\
    & - [ q_\mu(s, o) + f(\gamma_n, s, o)] \\
    & = \sum_a \pi(a | s, o) \sum_{s', r} p(s', r | s, a) (r - r(\mu) + \beta(s', o) \max_{o'} [q_\mu(s', o') + f(\gamma_n, s', o')] \\ 
    & + (\gamma_n - 1) \beta(s', o) \max_{o'} [q_\mu(s', o') + f(\gamma_n, s', o')] \\ 
    & + (1 - \beta(s', o)) [q_\mu(s', o) + f(\gamma_n, s', o)] \\
    & + (\gamma_n - 1) (1 - \beta(s', o)) [q_\mu(s', o) + f(\gamma_n, s', o)]\\
    & - [ q_\mu(s, o) + f(\gamma_n, s, o)].
\end{align*}
Note that $(\gamma - 1) \beta(s', o) \max_{o'} [q_\mu(s', o') + f(\gamma, s', o')]$ and $(\gamma - 1) (1 - \beta(s', o)) [q_\mu(s', o) + f(\gamma, s', o)]$ both
converge to 0 as $\gamma \uparrow 1$. 

Now take $n \to \infty$, then $\gamma_n \uparrow 1$, we have
\begin{align*}
    0 = \sum_a \pi(a | s, o) \sum_{s', r} p(s', r | s, a) \left(r - r(\mu) + \beta(s', o) \max_{o'} q_\mu(s', o') + (1 - \beta(s', o)) q_\mu(s', o)\right) - q_\mu(s, o).
\end{align*}

We see that  $\bar r = r(\mu)$ and $q = q_\mu$ is a solution of \eqref{eq: Intra-option Equations}-\eqref{eq: Intra-option Optimality Equations}.

Now we show that the solution for $\bar r$ is unique. Define
\begin{align*}
    B(\bar r, q) \doteq \sum_{a} \pi(a | s, o)\sum_{s', r} p(s', r | s, a) \left (r - \bar r + \beta(s', o) \max_{o'} q(s', o') + (1 - \beta(s', o)) q(s', o) \right ) - q(s, o).
\end{align*}
First we show if $B(\bar r, q) = 0$, then $\bar r \geq r_*$.
\begin{align*}
    0 & = B(\bar r, q)\\
    & = \sum_{a} \pi(a | s, o)\sum_{s', r} p(s', r | s, a)(r - \bar r + \beta(s', o) \max_{o'} q(s', o') + (1 - \beta(s', o)) q(s', o)) - q(s, o) \\
    & \geq \sup_{\mu \in \Pi^{MR}} \sum_{a} \pi(a | s, o)\sum_{s', r} p(s', r | s, a) \\
    & \left (r - \bar r + \beta(s', o) \sum_{o'} \mu(o' | s') q(s', o') + (1 - \beta(s', o)) q(s', o) \right) - q(s, o),
\end{align*}
where $\Pi^{MR}$ denotes the set of all Markov randomized policies. In vector form, the above equation can be written as:
\begin{align*}
    0 \geq \sup_{\mu \in \Pi^{MR}} \{r - \bar r e + (P_\mu - I) q\}.
\end{align*}
Therefore $\forall \mu \in \Pi^{MR}$,
\begin{align*}
    \bar r e \geq r + (P_\mu - I) q.
\end{align*}

Apply $P_\mu$ to both sides,
\begin{align*}
    P_\mu \bar r e & \geq P_\mu r + P_\mu(P_\mu - I) q,\\
    \bar r e & \geq P_\mu r + P_\mu(P_\mu - I) q .
\end{align*}
Repeating this process we have:
\begin{align*}
    \bar r e \geq P_\mu^n r + P_\mu^n(P_\mu - I) q.
\end{align*}
Summing these expressions from $n=0$ to $n=N-1$ we have:
\begin{align*}
    N \bar r e \geq \sum_{n=0}^{N-1} (P_\mu^n r + P_\mu^n(P_\mu - I) q) = \sum_{n=0}^{N-1} P_\mu^n r + (P_\mu^N - P_\mu^{N-1}) q.
\end{align*}
Because $\lim_{N \to \infty} \frac{1}{N} (P_\mu^N - P_\mu^{N-1}) q = 0$,
\begin{align*}
    \bar r e \geq \lim_{N \to \infty} \frac{1}{N} \sum_{n=0}^{N-1} P_\mu^n r = r(\mu) e,
\end{align*}
for all $\mu \in \Pi^{MR}$. Therefore $\bar r \geq r_*$.

Then we show that if $0 = B(\bar r, q)$ then $\bar r \leq r_*$. As we proved above, if $(\bar r,  q)$ satisfies that $0 = B(\bar r, q)$ then there exists a policy $\mu$ such that $\bar r e = r + (P_\mu - I) q$ is true. Therefore,
\begin{align*}
    P_\mu^n \bar r e &= P_\mu^n r + P_\mu^n (P_\mu - I) q,\\
    \lim_{N \to \infty} \frac{1}{N} \sum_{n=0}^{N-1} P_\mu^n \bar r e &=  \lim_{N \to \infty} \frac{1}{N} \sum_{n=0}^{N-1} (P_\mu^n r + P_\mu^n (P_\mu - I) q),\\
    \bar r e &= \lim_{N \to \infty} \sum_{n=0}^{N-1} P_\mu^n r = r(\mu) e \leq r_* e.
\end{align*}
Therefore $\bar r \leq r_*$. Combining $\bar r \geq r_*$ and $\bar r \leq r_*$ we have $\bar r = r_*$.

Finally, we show that the solution for $q$ is unique only up to a constant. Note that one could iteratively replace $q$ in the r.h.s. of the intra-option Optimality equation \eqref{eq: Intra-option Equations}-\eqref{eq: Intra-option Optimality Equations} by the entire r.h.s. of the intra-option Optimality equation, resulting to the inter-option Optimality equation \eqref{eq: SMDP Bellman optimality equation}. Therefore any solution of \eqref{eq: Intra-option Equations}-\eqref{eq: Intra-option Optimality Equations} must be a solution of \eqref{eq: SMDP Bellman optimality equation}. But we know that the solutions for $q$ in \eqref{eq: SMDP Bellman optimality equation} is unique only up to a constant. Therefore the solutions for $q$ in \eqref{eq: Intra-option Equations}-\eqref{eq: Intra-option Optimality Equations} can not differ by a non-constant. Further, it is easy to see that if $q$ is a solution, then $q + ce, \forall c$ is also a solution. The theorem is proved.

$\Box$

\subsection{Theorem~\ref{thm: Intra-option Differential Methods}} \label{app: intra-option value}

For simplicity, we will only provide formal theorems and proofs for our \emph{control} learning and planning algorithms. The formal theorems and proofs for our prediction algorithms are similar to those for the control algorithms and are thus omitted.  To this end, we first provide a general algorithm that includes both learning and planning control algorithms. We call it \emph{General Intra-option Differential Q}. We first formally define it and then explain why both Intra-option Differential Q-learning and Intra-option Differential Q-planning are special cases of General Intra-option Differential Algorithm. We then provide assumptions and the convergence theorem of the general algorithm. The theorem would lead to convergence of the special cases. Finally, we provide a proof for the theorem.

Given an MDP $\calM \doteq (\calS, \calA, \calR, p)$ and a set of options $\calO$, for each state $s \in \calS$, option $o \in \calO$, a reference option $\bar o$, and discrete step $n \geq 0$, let $A_n(s, \bar o) \sim \pi(\cdot \mid s, \bar o)$, $R_n(s, A_n(s, \bar o)),  S'_n(s, A_n(s, \bar o)) \sim p(\cdot, \cdot \mid s, A_n(s, \bar o))$ denote, given state-option pair $(s, \bar o)$, a sample of the chosen action and the resulting state and reward. 
We hypothesize a set-valued process $\{Y_n\}$ taking
values in the set of nonempty subsets of $\calS \times \calO$ with the interpretation: $Y_n = \{(s, o): (s, o)$ component of $Q$ was updated at time $n\}$. Let $\nu(n, s, o) \doteq \sum_{k=0}^n I\{(s, o) \in Y_k\}$, where $I$ is the indicator function. Thus $\nu(n, s, o) =$ the number of times the $(s, o)$ component was updated up to step $n$. In addition, we hypothesize a set-valued process $\{Z_n\}$ taking
values in the set of nonempty subsets of $\calO$ with the interpretation: $Z_n = \{\bar o: \bar o$ component was visited at time $n\}$. $\sum_{\bar o} I\{\bar o \in Z_n\}$ means the number of reference options used at update step $n$. For simplicity, we assume this number is always 1. 
\begin{assumption} \label{assu: num of reference options > 1}
$\sum_{\bar o} I\{\bar o \in Z_n\} = 1$ for all discrete $n \geq 0$.
\end{assumption}

The update rules of General Intra-option Differential Q are 
\begin{align}
Q_{n+1}(s, o) & \doteq Q_n(s, o) + \alpha_{\nu(n, s, o)} \sum_{\bar o} \rho_n(s, o, \bar o) \delta_n(s, o, \bar o)  I\{(s, o) \in Y_n\} I\{\bar o \in Z_n\}, \quad \forall s \in \calS, \text{ and } o \in \calO \label{eq: General Intra-option Diff Q async update Q}\\
\bar R_{n+1} & \doteq \bar R_n + \eta \sum_{s, o} \alpha_{\nu(n, s, o)} \sum_{\bar o} \rho_n(s, o, \bar o) \delta_n(s, o, \bar o) I\{(s, o) \in Y_n\} I\{\bar o \in Z_n\}, \label{eq: General Intra-option Diff Q async update bar R}
\end{align}
where $\rho_n(s, o, \bar o) \doteq \pi(A_n(s, \bar o) \mid s, o) / \pi(A_n(s, \bar o) \mid s, \bar o)$ and
\begin{align}\label{eq: General Intra-option Diff Q TD error}
    \delta_n(s, o, \bar o) & \doteq R_{n}(s, A_n(s, \bar o)) - \bar R_n + \beta(S'_n(s, A_n(s, \bar o)), o) \max_{o'} Q_n(S'_n(s, A_n(s, \bar o)), o') \nonumber\\
    & + (1 - \beta(S'_n(s, A_n(s, \bar o)), o)) Q_n(S'_n(s, A_n(s, \bar o)), o) - Q_n(s, o)
\end{align}
is the TD error.

Here $\alpha_{\nu(n, s, o)}$ is the stepsize at step $n$ for state-option-option triple $(s, o)$. The quantity $\alpha_{\nu(n, s, o)}$ depends on the sequence $\{\alpha_n\}$, which is an algorithmic design choice, and also depends on the visitation of state-option pairs $\nu(n, s, o)$. To obtain the stepsize, the algorithm could maintain a $\vert \calS \times \calO \vert$-size table counting the number of visitations to each state-option pair, which is exactly $\nu(\cdot, \cdot, \cdot)$. Then the stepsize $\alpha_{\nu(n, s, o)}$ can be obtained as long as the sequence $\{\alpha_n\}$ is specified.

Now we show Intra-option Differential Q-learning and Intra-option Differential Q-planning are special cases of General Intra-option Differential Q. Consider a sequence of real experience $\ldots, S_t, O_t, A_t, R_{t+1}, S_{t+1}, \ldots$. By choosing step $n$ = time step $t$,
\begin{align*}
    Y_n(s, o) &= 1 \text{, if } s = S_t \\
    Y_n(s, o) &= 0 \text{ otherwise,}\\
    Z_n(\bar o) & = 1 \text{, if } \bar o = O_t \\
    Z_n(\bar o) & = 0 \text{ otherwise,}
\end{align*}
and $A_n(S_t, O_t) = A_{t}$, $S'_n(S_t, A_n(S_t, O_t)) = S_{t+1}, R_n(S_t, A_n(S_t, O_t)) = R_{t+1}$, update rules \eqref{eq: General Intra-option Diff Q async update Q}, \eqref{eq: General Intra-option Diff Q async update bar R}, and \eqref{eq: General Intra-option Diff Q TD error} become
\begin{align*}
    Q_{t+1}(S_t, o) & \doteq Q_t (S_t, o) + \alpha_{\nu(t, S_t, o)} \rho_t(o) \delta_t(o),  \forall o \in \calO \text{\ , and\ } Q_{t+1}(s, o) \doteq Q_t (s, o), \forall o \in \calO \text{ and } \forall s \neq S_t, \\
    \bar R_{t+1} & \doteq \bar R_t + \eta \sum_o \alpha_{\nu(t, S_t, o)} \rho_t(o) \delta_t(o) ,\\
    \delta_t(o) & \doteq R_{t+1} - \bar R_t + \beta(S_{t+1}, o) \max_{o'} Q_t(S_{t+1}, o') + (1 - \beta(S_{t+1}, o)) Q_t(S_{t+1}, o) - Q_t(S_t, o),
\end{align*}
where $\rho_t(o) \doteq \pi(A_t \mid S_t, o) / \pi(A_t \mid S_t, O_t)$. The above equations are Intra-option Differential Q-learning's update rules (Equations~\ref{eq: Intra-option Differential TD-learning Q}, \ref{eq: Intra-option Differential TD-learning R bar}, \ref{eq: Intra-option Differential Q-learning TD error}) with stepsize at time $t$ being $\alpha_{\nu(t, S_t, o)}$ for each option $o$.

If we consider a sequence of simulated experience $\ldots, \tilde S_n, \tilde O_n, \tilde A_n, \tilde R_{n}, \tilde S_{n}', \ldots$, by choosing step $n =$ planning step $n$,
\begin{align*}
    Y_n(s, o) &= 1 \text{, if } s = \tilde S_n \\
    Y_n(s, o) &= 0 \text{ otherwise,}\\
    Z_n(\bar o) & = 1 \text{, if } \bar o = \tilde O_n \\
    Z_n(\bar o) & = 0 \text{ otherwise,}
\end{align*}
and $A_n(\tilde S_n, \tilde O_n) = \tilde A_{n}$, $S'_n(\tilde S_n, A_n(\tilde S_n, \tilde O_n)) = \tilde S_{n}', R_n(\tilde S_n, A_n(\tilde S_n, \tilde O_n)) = \tilde R_{n}$, update rules \eqref{eq: General Intra-option Diff Q async update Q}, \eqref{eq: General Intra-option Diff Q async update bar R}, and \eqref{eq: General Intra-option Diff Q TD error} become
\begin{align*}
    Q_{n+1}(\tilde S_n, o) & \doteq Q_n (\tilde S_n, o) + \alpha_{\nu(n, \tilde S_n, o)} \rho_n(o) \delta_n(o) , \forall o \in \calO \text{\ , and\ } Q_{n+1}(s, o) \doteq Q_n (s, o), \forall s \neq \tilde S_n, \forall o \in \calO\\
    \bar R_{n+1} & \doteq \bar R_n + \eta \sum_o \alpha_{\nu(n, \tilde S_n, o)} \rho_n(o) \delta_n(o) ,\\
    \delta_n(o) & \doteq \tilde R_{n} - \bar R_n + \beta(\tilde S_{n}', o) \max_{o'} Q_n(\tilde S_{n}', o') + (1 - \beta(\tilde S_{n}', o)) Q_n(\tilde S_{n}', o) - Q_n(\tilde S_n, o),
\end{align*}
where $\rho_n(o) \doteq \pi(A_n \mid S_n, o) / \pi(A_n \mid S_n, O_n)$. The above equations are Intra-option Differential Q-planning's update rules (Equations~\ref{eq: Intra-option Differential TD-learning Q}, \ref{eq: Intra-option Differential TD-learning R bar}, \ref{eq: Intra-option Differential Q-learning TD error}) with stepsize at planning step $n$ being $\alpha_{\nu(n, S_n, o)}$ for each option $o$.

Finally, note that for both Intra-option Differential Q-learning and Q-planning algorithms, because for each time step $t$ or update step $n$, there is only one option which is actually chosen to generate data, \cref{assu: num of reference options > 1} is satisfied.

\begin{theorem}
Under Assumptions~\ref{assu: unichain}, \ref{assu: stepsize}, \ref{assu: asynchronous stepsize 1}, \ref{assu: asynchronous stepsize 2}, \ref{assu: num of reference options > 1}, General Intra-option Differential Q (Equations~\ref{eq: General Intra-option Diff Q async update Q}-\ref{eq: General Intra-option Diff Q TD error}) converges, almost surely,
$Q_n$ to $q$ satisfying both \eqref{eq: Intra-option Equations}-\eqref{eq: Intra-option Optimality Equations}
and
\begin{align}
    \eta  (\sum q - \sum Q_0) = r_* - \bar R_0,
\end{align}
$\bar R_n$ to $r_*$, and $r(\mu_n)$ to $r_*$ where $\mu_n$ is a greedy policy w.r.t. $Q_n$.
\end{theorem}

\begin{proof}
At each step, the increment to $\bar R_n$ is $\eta$ times the increment to $Q_n$ and $\sum Q_n$. Therefore, the cumulative increment can be written as:
\begin{align}
    \bar R_n - \bar R_0 &= \eta  \sum_{i = 0}^{n-1} \sum_{s, o} \alpha_{\nu(i, s, o)} \sum_{\bar o} \rho_i(s, o, \bar o)\delta_i(s, o, \bar o) I\{(s, o) \in Y_i\} I\{\bar o \in Z_i\} \nonumber  \\
    & = \eta \left(\sum Q_n - \sum Q_0 \right) \nonumber \\
    \implies \bar R_n &= \eta \sum Q_n - \eta \sum Q_0 + \bar R_0 = \eta \sum Q_n - c \label{eq: General Intra-option Diff Q relation between bar R and Q} ,\\
    \text{ where } c &\doteq \eta \sum Q_0 - \bar R_0. \nonumber
\end{align}
Now substituting $\bar R_n$ in \eqref{eq: General Intra-option Diff Q async update Q} with \eqref{eq: General Intra-option Diff Q relation between bar R and Q}, we have $\forall s \in \calS, o \in \calO$: 
\begin{align}
    & Q_{n+1}(s, o) = Q_{n}(s, o) + \alpha_{\nu(n, s, o)} \sum_{\bar o}  \frac{\pi(A_n(s, \bar o) \mid s, o)}{\pi(A_n(s, \bar o) \mid s, \bar o)} \nonumber \\
    & \bigg(R_{n}(s, A_n(s, \bar o)) - \eta \sum Q_n + c + \beta(S_{n}'(s, A_n(s, \bar o)), o) \max_{o'} Q_n(S_{n}'(s, A_n(s, \bar o)), o') \nonumber \\
    & + (1 - \beta(S_{n}'(s, A_n(s, \bar o)), o)) Q_n(S_{n}'(s, A_n(s, \bar o)), o) - Q_n(s, o) \bigg) \nonumber\\
    & I\{(s, o) \in Y_n\} I\{\bar o \in Z_n\} .\label{eq: General Intra-option Diff Q transformed async single update}
\end{align}
We now show that \eqref{eq: General Intra-option Diff Q transformed async single update} is a special case of \eqref{eq: General RVI Q async update}. To see this point, let $i = (s, o)$, 
\begin{align*}
    R_n(i) & = \sum_{\bar o}  \frac{\pi(A_n(s, \bar o) \mid s, o)}{\pi(A_n(s, \bar o) \mid s, \bar o)} I\{\bar o \in Z_n\} (R_{n}(s, A_n(s, \bar o)) + c),  \\
    F_n(Q_n)(i) & = \sum_{\bar o}  \frac{\pi(A_n(s, \bar o) \mid s, o)}{\pi(A_n(s, \bar o) \mid s, \bar o)} I\{\bar o \in Z_n\} \eta \sum Q_n, \\
    G_n(Q_n)(i) & = \sum_{\bar o}  \frac{\pi(A_n(s, \bar o) \mid s, o)}{\pi(A_n(s, \bar o) \mid s, \bar o)} I\{\bar o \in Z_n\} \big ( \beta(S_{n}'(s, A_n(s, \bar o)), o) \max_{o'} Q_n(S_{n}'(s, A_n(s, \bar o)), o') \\
    & + (1 - \beta(S_{n}'(s, A_n(s, \bar o)), o)) Q_n(S_{n}'(s, A_n(s, \bar o)), o) - Q_n(s, o) \big ),\\
    \epsilon_n(i) &= 0.
\end{align*}
Then we have:
\begin{align*}
    r(i) & = \bbE[R_n(i)] \\
    & = \bbE\left[\sum_{\bar o}  \frac{\pi(A_n(s, \bar o) \mid s, o)}{\pi(A_n(s, \bar o) \mid s, \bar o)} I\{\bar o \in Z_n\} (R_{n}(s, A_n(s, \bar o)) + c) \right] \\
    & = \sum_{\bar o}   \bbE\left[\frac{\pi(A_n(s, \bar o) \mid s, o)}{\pi(A_n(s, \bar o) \mid s, \bar o)} I\{\bar o \in Z_n\} (R_{n}(s, A_n(s, \bar o)) + c) \right] \\
    & = \sum_{\bar o} I\{\bar o \in Z_n\} \sum_a \pi(a \mid s, o) \bbE[R_n(s, a) + c]\\
    & = \sum_a \pi(a \mid s, o) \sum_{r, s'} p(r, s' \mid s, a) (r + c), \quad \quad \quad \text{By \cref{assu: num of reference options > 1}},\\
    f(q) &= \bbE[F(q)(i)] = \eta \sum q,\\
    g(q)(i) & = \bbE[G_n(q)(i)] \\
    & = \bbE\Bigg[\sum_{\bar o}  \frac{\pi(A_n(s, \bar o) \mid s, o)}{\pi(A_n(s, \bar o) \mid s, \bar o)} I\{\bar o \in Z_n\} \big ( \beta(S_{n}'(s, A_n(s, \bar o)), o) \max_{o'} q(S_{n}'(s, A_n(s, \bar o)), o') \\
    & + (1 - \beta(S_{n}'(s, A_n(s, \bar o)), o)) q(S_{n}'(s, A_n(s, \bar o)), o) - q(s, o) \big ) \Bigg] \\
    & = \sum_{\bar o} I\{\bar o \in Z_n\} \sum_a \pi(a \mid s, o) \\
    & \bbE[(\beta(S_{n}'(s, a), o) \max_{o'} q(S_{n}'(s, a), o') + (1 - \beta(S_{n}'(s, a), o)) q(S_{n}'(s, a), o) - q(s, o))]\\
    & = \sum_a \pi(a \mid s, o) \sum_{s', r} p(s', r \mid s, a) (\beta(s', o) \max_{o'} q(s', o') + (1 - \beta(s', o)) q(s', o) - q(s, o))],
\end{align*}
for any $i \in \calI$.

We now verify the assumptions of Theorem~\ref{thm: General RVI Q} for Intra-option General Differential Q. 
\Cref{assu: g} can be verified for $g(q)(s, o) = \sum_{a} \pi(a \mid s, o) \sum_{s', r} p(s', r \mid s, a) (\beta(s', o) \max_{o'} q(s', o') + (1 - \beta(s', o)) q(s', o))$ easily.
\Cref{assu: f} is satisfied for $f(q) = \eta \sum q$.
\Cref{assu: variance of Martingale difference} satisfies because the MDP is finite.  
\cref{assu: epsilon} is satisfied for $\epsilon_n = 0$.
Assumption~\ref{assu: stepsize}-\ref{assu: asynchronous stepsize 2} are satisfied due to assumptions of the theorem being proved. \cref{assu: General RVI Q unique solution} is satisfied because 
\begin{align*}
    & r(i) - \bar r + g(q)(i) - q(i) \\
    & = \sum_{a} \pi(a | s, o)\sum_{s', r} p(s', r | s, a) (r - \bar r + \beta(s', o) \max_{o'} q(s', o') + (1 - \beta(s', o)) q(s', o)).
\end{align*}
By \cref{thm: intra-option bellman equations}, we know that if the above equation equals to 0, then under \cref{assu: unichain}, $\bar r = r_* + c$ is the unique solution and the solutions for $q$ form a set $q = q_* + ke$ for all $k \in \bbR$.

Therefore \cref{thm: General RVI Q} applies and we conclude that $Q_n$ converges to a point satisfying $\eta \sum q = r_* + c = r_* + \eta \sum Q_0 - \bar R_0$ and $\bar R_n = \eta \sum Q_n - c$ to $\eta \sum q - c = r_* + c - c = r_*$. Finally, by \cref{lemma: bound of mu Q}, we have $r(\mu_n) \to r_*$.

Applying a similar argument as one presented in the last paragraph of Section~\ref{app: inter option} finishes the proof of Theorem~\ref{thm: Intra-option Differential Methods}.
\end{proof}

\subsection{Theorem~\ref{thm: model equations}} \label{app: Proof of thm: model equations}
\begin{proof}
We will show that there exists a unique solution for \eqref{eq: Intra-option Dynamics Model Equation}. Results for \eqref{eq: Intra-option Reward Model Equation} and \eqref{eq: Intra-option Duration Model Equation} can be shown in a similar way. To show that \eqref{eq: Intra-option Dynamics Model Equation} has a unique solution, we apply a generalized version of the Banach fixed point theorem (see, e.g., Theorem 2.4 by Almezel, Ansari, and Khamsi 2014). Once the unique existence of the solution is shown, we easily verify that $\dynamicsmodeltarget$ is the unique solution by showing that it is one solution to \eqref{eq: Intra-option Dynamics Model Equation} as follows. With a little abuse of notation, let $\hat p (s', r \mid s, o) \doteq \sum_{r, l} \hat p(x, r, l \mid s, o)$, we have
\begin{align*}
    m^p(x | s, o) &= \sum_{r, l} \hat p(x, r, l | s, o) \\
    & = \sum_{l = 1}^\infty \hat p(x, l | s, o) = \sum_{a} \pi(a | s, o)\sum_{r} p(s', r | s, a) \beta(s', o) \mathbb{I}(x = s') + \sum_{l = 2}^\infty \hat p(x, l | s, o)\\
    & = \sum_{a} \pi(a | s, o)\sum_{r} p(s', r | s, a) \big(\beta(s', o) \mathbb{I}(x = s') + (1 - \beta(s', o)) \sum_{l = 1}^\infty \hat p(x, l | s', o) \big)\\
    & = \sum_{a} \pi(a | s, o)\sum_{r} p(s', r | s, a) \big(\beta(s', o) \mathbb{I}(x = s') + (1 - \beta(s', o)) m^p(x | s', o)\big).
\end{align*}
To apply the generalized version of the Banach fixed point theorem to show the unique existence of the solution, we first define operator $T: \bbR^{\cardS \times \cardS \times \cardO} \to \bbR^{\cardS \times \cardS \times \cardO}$ such that for any $m \in \bbR^{\cardS \times \cardS \times \cardO}$ and any $x, s \in \calS, o \in \calO$, $T m (x \mid s, o) \doteq \sum_{a} \pi(a | s, o)\sum_{s', r} p(s', r | s, a) ( \beta(s', o) \mathbb{I}(x = s') + (1 - \beta(s', o)) m(x | s', o)))$. We further define $T^n m \doteq T(T^{n-1} m)$ for any $n \geq 2$ and any $m \in \bbR^{\cardS \times \cardS \times \cardO}$. The generalized Banach fixed point theorem shows that if $T^n$ is a contraction mapping for any integer $n \geq 1$ (this is called a $n$-stage contraction), then $Tm = m$ has a unique fixed point. The unique fixed point immediately leads to the existence of the unique solution of \eqref{eq: Intra-option Dynamics Model Equation}. The existence of the unique solution and that $m^p$ is a solution imply that $m^p$ is the unique solution.

The only work left is to verify the following contraction property:
\begin{align}\label{eq: n-stage contraction}
    \norm{T^{\cardS} m_1 - T^{\cardS} m_2}_\infty \leq \gamma \norm{m_1 - m_2}_\infty,
\end{align}
where $m_1$ and  $m_2$ are arbitrary members in $\bbR^{\cardS \times \cardS \times \cardO}$, and $\gamma < 1$ is some constant.

Consider the difference between $T^{\cardS} m_1$ and $T^{\cardS} m_2$ for arbitrary $m_1, m_2 \in \bbR^{|\calS \times \calS \times \calO|}$. For any $x, s \in \calS, o \in \calO$, we have
\begin{align*}
    & T^{\cardS} m_1 (x \mid s, o) - T^{\cardS} m_2 (x \mid s, o) \\
    & = \sum_{a} \pi(a \mid s, o)\sum_{s', r} p(s', r | s, a) (1 - \beta(s', o)) (T^{\cardS-1}m_1(x \mid s', o) - T^{\cardS-1}m_2(x \mid s', o))\\
    & = \sum_{s_1} \Pr(S_{t+1} = s_1 \mid S_t = s, O_t = o) (1 - \beta(s_1, o)) (T^{\cardS-1}m_1(x \mid s_1, o) - T^{\cardS-1}m_2(x \mid s_1, o))\\
    & = \sum_{s_1} \Pr(S_{t+1} = s_1 \mid S_t = s, O_t = o) (1 - \beta(s_1, o)) \sum_{s_2} \Pr(S_{t+2} = s_2 \mid S_{t+1} = s_1, O_{t+1} = o) (1 - \beta(s_2, o)) \\
    & (T^{\cardS-2}m_1(x \mid s_2, o) - T^{\cardS-2}m_2(x | s_2, o))\\
    & \vdots\\
    & = \sum_{s_1, \cdots, s_{\cardS}} \Pr(S_{t+1} = s_1, \cdots, S_{t+\cardS} = s_{\cardS} \mid S_t = s, O_t = o) \prod_{i =1}^{\cardS} (1 - \beta(s_i, o)) (m_1(x \mid s_{\cardS}, o) - m_2(x | s_{\cardS}, o))\\
    & \leq \sum_{s_1, \cdots, s_{\cardS}} \Pr(S_{t+1} = s_1, \cdots, S_{t+\cardS} = s_{\cardS} \mid S_t = s, O_t = o) \prod_{i =1}^{\cardS} (1 - \beta(s_i, o)) \norm{m_1 - m_2}_\infty.
\end{align*}
Here $\tilde p(s, o) \doteq \sum_{s_1, \cdots, s_{\cardS}} \Pr(S_{t+1} = s_1 \cdots, S_{t+\cardS} = s_{\cardS} \mid S_t = s, O_t = o) \prod_{i =1}^{\cardS} (1 - \beta(s_i, o))$ is the probability of executing option $o$ for $\cardS$ steps starting from $s$ without termination. If $\tilde p(s, o) = 0$, then option $o$ will surely terminate within the first $\abs{\calS}$ steps and if $\tilde p(s, o) = 1$, then option $o$ will surely not terminate within the first $\abs{\calS}$ steps.

If option $o$ would surely not terminate within the first $\cardS$ steps ($\tilde p(s, o) = 1$), then it would surely not terminate forever. This is because there are only $\cardS$ number of states, and thus an option could visit all states that are possible to be visited by the option within the first $\cardS$ steps. $\tilde p(s, o) = 1$ means that option $o$ has a zero probability of terminating in all states that are possible to be visited by option $o$. This non-termination of a state-option pair implies that the expected option length is infinite, which is contradict to our assumption of finite expected option lengths (\cref{sec: background}). Therefore $\tilde p(s, o) = 1$ is not allowed by our assumption and thus $\tilde p(s, o) < 1$. So there must exist some $\gamma(s, o) < 1$ such that $\tilde p(s, o) \leq \gamma(s, o)$. With $\gamma \doteq \max_{s, o} \gamma (s, o)$, we obtain \eqref{eq: n-stage contraction}.
\end{proof}

\subsection{Theorem~\ref{thm: intra-option model learning}}\label{app: thm: intra-option model learning}

We first provide a formal statement of \cref{thm: intra-option model learning}. The formal theorem statement needs stepsizes to be specific for each state-option pair. We rewrite (\ref{eq: Intra-option Dynamics Model learning}–\ref{eq: Intra-option Duration Model learning}) to incorporate such stepsizes:
\fontsize{9.5}{6}
\begin{align}
    \dynamicsmodel_{t+1}(x \mid S_t, o) &\doteq \dynamicsmodel_{t}(x \mid S_t, o) + \alpha_t(S_t, o) \rho_t(o) \Big( \beta(S_{t+1}, o) \mathbb{I}(S_{t+1} = x) \nonumber \\
    &\quad + \big( 1 - \beta(S_{t+1}, o) \big) \dynamicsmodel_t(x \mid S_{t+1}, o) - \dynamicsmodel_t(x \mid S_t, o) \Big), \quad \forall\ x \in \statespace, \label{eq: Intra-option Dynamics Model learning, formal}\\
    \rewardmodel_{t+1}(S_t, o) &\doteq \rewardmodel_{t}(S_t, o) + \alpha_t(S_t, o) \rho_t(o) \Big( R_{t+1} + \big( 1 - \beta(S_{t+1}, o) \big) \rewardmodel_t(S_{t+1}, o) - \rewardmodel_{t}(S_t, o) \Big) \label{eq: Intra-option Reward Model Learning, formal}\\
    \durationmodel_{t+1}(S_t, o) &\doteq \durationmodel_{t}(S_t, o) + \alpha_t(S_t, o) \rho_t(o) \Big( 1 + \big( 1 - \beta(S_{t+1}, o) \big) \durationmodel_{t}(S_{t+1}, o) - \durationmodel_{t}(S_t, o) \Big) \label{eq: Intra-option Duration Model learning, formal}.
\end{align}
\normalsize
\begin{theorem}[Convergence of the intra-option model learning algorithm, formal]
If $0 \leq \alpha_t(s, o) \leq 1$, $\sum_t \alpha_t(s, o) = \infty$ and $\sum_t \alpha_t^2(s, o) < \infty$, and $\alpha_t(s, o) = 0$ unless $s = S_t$, then the intra-option model-learning algorithm (\ref{eq: Intra-option Dynamics Model learning, formal}–\ref{eq: Intra-option Duration Model learning, formal}) converges almost surely, $\dynamicsmodel_t$ to $\dynamicsmodeltarget$, $\rewardmodel_t$ to $\rewardmodeltarget$, and $\durationmodel_t$ to $\durationmodeltarget$.
\end{theorem} 
Here the assumptions on $\alpha_t$ guarantee that each state-option pair is updated for an infinite number of times. Because the three update rules are independent, we only show convergence of the first update rule; the other two can be shown in the same way. 

\begin{proof}
We apply a slight generalization of Theorem 3 by Tsitsiklis (1994) to show the above theorem. The generalization replaces Assumption 5 (an assumption for Theorem 3) by:
\begin{assumption}
There exists a vector $x^* \in \bbR^n$, a positive vector $v$, a positive integer $m$, and a scalar $\beta \in [0, 1)$, such that 
\begin{align*}
    \norm{F^m(x) - x^*}_v \leq \beta \norm{x - x^*}_v, \quad \forall x \in \bbR^n.
\end{align*}
\end{assumption}
That is, we replace the one-stage contraction assumption by a $m$-stage contraction assumption. The proof of Tsitsiklis' Theorem 3 also applies to its generalized form and is thus omitted here.

Notice that our update rule \eqref{eq: Intra-option Dynamics Model learning, formal} is a special case of the general update rule considered by Theorem 3 (equations 1-3), and is thus a special case of its generalized version. Therefore we only need to verify the above $m-$stage contraction assumption, as well as Assumption 1, 2, and 3 required by Theorem 3. According to the proof in \cref{app: Proof of thm: model equations}, the operator $T$ associated with the update rule \eqref{eq: Intra-option Dynamics Model learning} is a $\abs{\calS}-$stage contraction (and thus is a $\abs{\calS}-$stage pseudo-contraction). Other assumptions (Assumptions 1, 2, 3) required by Theorem 3 are also satisfied given our step-size, and finite MDP assumptions.
\end{proof}

\subsection{\cref{thm: interruption}}\label{app: interruption}

\begin{proof}
We first show that 
\begin{align}
    & \sum_{o'} \mu'(o' \mid s) \sum_{s', r, l} \hat p(s', r, l \mid s, o') (r - l r(\mu) + v_\mu(s')) \nonumber\\
    & \geq \sum_{o} \mu(o \mid s) \sum_{s', r, l} \hat p(s', r, l \mid s, o) (r - l r(\mu) + v_\mu(s')) = v_\mu(s) \label{eq: proof of interruption 1}.
\end{align}

Note that for all $s$, $o$ and its corresponding $o'$, $\mu(o \mid s) = \mu'(o' \mid s)$. In order to show \eqref{eq: proof of interruption 1}, we show $\sum_{s', r, l} \hat p(s', r, l \mid s, o') (r - l r(\mu) + v_\mu(s')) \geq \sum_{s', r, l} \hat p(s', r, l \mid s, o) (r - l r(\mu) + v_\mu(s'))$ for all $s, o$ and corresponding $o'$.
\begin{align*}
    & \sum_{s', r, l} \hat p(s', r, l \mid s, o') (r - l r(\mu) + v_\mu(s')) \\
    & = \bbE [\hat R_n - \hat L_n r(\mu) + v_\mu(\hat S_{n+1}) \mid S_n = s, O_n = o'] \\
    & = \bbE [\hat R_n - \hat L_n r(\mu) + v_\mu(\hat S_{n+1}) \mid S_n = s, O_n = o', \text{Not encountering an interruption}] \\
    & + \bbE [\hat R_n - \hat L_n r(\mu) + v_\mu(\hat S_{n+1}) \mid S_n = s, O_n = o', \text{Encountering an interruption}] \\
    & \geq \bbE [\hat R_n - \hat L_n r(\mu) + v_\mu(\hat S_{n+1}) \mid S_n = s, O_n = o', \text{Not encountering an interruption}] \\
    & + \bbE [\beta(s')(\hat R_n - \hat L_n r(\mu) + v_\mu(\hat S_{n+1})) + (1 - \beta(s')) (\hat R_n - \hat L_n r(\mu) + q_\mu(\hat S_{n+1}, o)) \\
    & \mid S_n = s, O_n = o', \text{Encountering an interruption}] \\
    & = \sum_{s', r, l} \hat p(s', r, l \mid s, o) (r - l r(\mu) + v_\mu(s')).
\end{align*}
The above inequality holds because $\hat S_{n+1}$ is the state where termination happens and thus $q_\mu(\hat S_{n+1}, o) \leq v_\mu(\hat S_{n+1})$. The last equality holds because $\bbE [\beta(s')(\hat R_n - \hat L_n r(\mu) + v_\mu(\hat S_{n+1})) + (1 - \beta(s')) (\hat R_n - \hat L_n r(\mu) + q_\mu(\hat S_{n+1}, o)) \mid S_n = s, O_n = o', \text{Encountering an interruption}]$ is the expected differential return when the agent could interrupt its old option but chooses to stick on the old option. \eqref{eq: proof of interruption 1} is shown.

Now write the l.h.s. of \eqref{eq: proof of interruption 1} in the matrix form
\begin{align*}
    \sum_{o'} \mu'(o'\mid s) \sum_{s', r, l} \hat p(s', r, l | s, o') (r - l r(\mu) + v_\mu(s')) = r_{\mu'} (s) - l_{\mu'}(s) r(\mu) + (P_{\mu'} v_\mu)(s),
\end{align*}
where $r_{\mu'}(s) \doteq \sum_{o'} {\mu'}(o'\mid s) \sum_{s', r, l} \hat p(s', r, l | s, o') r$ is the expected one option-transition reward, $l_{\mu'}(s) \doteq \sum_{o'} {\mu'}(o'\mid s) \sum_{s', r, l} \hat p(s', r, l | s, o') l$ is the expected one option-transition length, and $P_{\mu'}(s, s') \doteq \sum_{o'} {\mu'}(o'\mid s) \sum_{r, l} \hat p(s', r, l | s, o')$ is the probability of terminating at $s'$.

Combined with the r.h.s. of \eqref{eq: proof of interruption 1}, we have 
\begin{align*}
    r_{\mu'} (s) - l_{\mu'}(s) r(\mu) + (P_{\mu'} v_\mu)(s) \geq v_\mu(s).
\end{align*}

Iterating the above inequality for $K-1$ times, we have
\begin{align*}
    \sum_{k=0}^{K-1}(P_{\mu'}^k r_{\mu'} (s) - P_{\mu'}^k l_{\mu'}(s) r(\mu)) + P_{\mu'}^K v_\mu(s) \geq v_\mu (s)\\
    \sum_{k=0}^{K-1}(P_{\mu'}^k r_{\mu'} (s) - P_{\mu'}^k l_{\mu'}(s) r(\mu)) \geq v_\mu (s) - P_{\mu'}^K v_\mu (s).
\end{align*}

Divide both sides by $\sum_{k=0}^{K-1} P_{\mu'}^k l_{\mu'}(s) $ and take $K \to \infty$:
\begin{align*}
    \lim_{K \to \infty} \frac{1}{\sum_{k=0}^{K-1} P_{\mu'}^k l_{\mu'}(s) } \sum_{k=0}^{K-1}(P_{\mu'}^k r_{\mu'} (s) - P_{\mu'}^k l_{\mu'}(s) r(\mu)) \geq \lim_{K \to \infty} \frac{1}{\sum_{k=0}^{K-1} P_{\mu'}^k l_{\mu'}(s) } (v_\mu(s) - P_{\mu'}^K v_\mu(s)).
\end{align*}

For the l.h.s.:
\begin{align*}
    \lim_{K \to \infty} \frac{1}{\sum_{k=0}^{K-1} P_{\mu'}^k l_{\mu'}(s) } \sum_{k=0}^{K-1}(P_{\mu'}^k r_{\mu'} (s) - P_{\mu'}^k l_{\mu'}(s) r(\mu))) = \lim_{K \to \infty} \frac{\sum_{k=0}^{K-1} P_{\mu'}^k r_{\mu'} (s)}{\sum_{k=0}^{K-1} P_{\mu'}^k l_{\mu'}(s) } -  r(\mu) = r(\mu') - r(\mu).
\end{align*}

For the r.h.s.:
\begin{align*}
    \lim_{K \to \infty} \frac{1}{\sum_{k=0}^{K-1} P_{\mu'}^k l_{\mu'}(s) } (v_\mu(s) - P_{\mu'}^K v_\mu(s)) = 0.
\end{align*}

Therefore $r(\mu') - r(\mu) \geq 0$. 

Finally, note that a strict inequality holds if the probability of interruption when following policy $\mu'$ is non-zero.
\end{proof}

\newpage
\section{Additional Empirical Results} \label{app: additional empirical results}

\subsection{Inter-option Learning}
\label{app: additional empirical results: Inter-option Learning}

\begin{figure*}[h!]
\centering
\begin{subfigure}{.5\textwidth}
    \centering
    \includegraphics[width=0.9\textwidth]{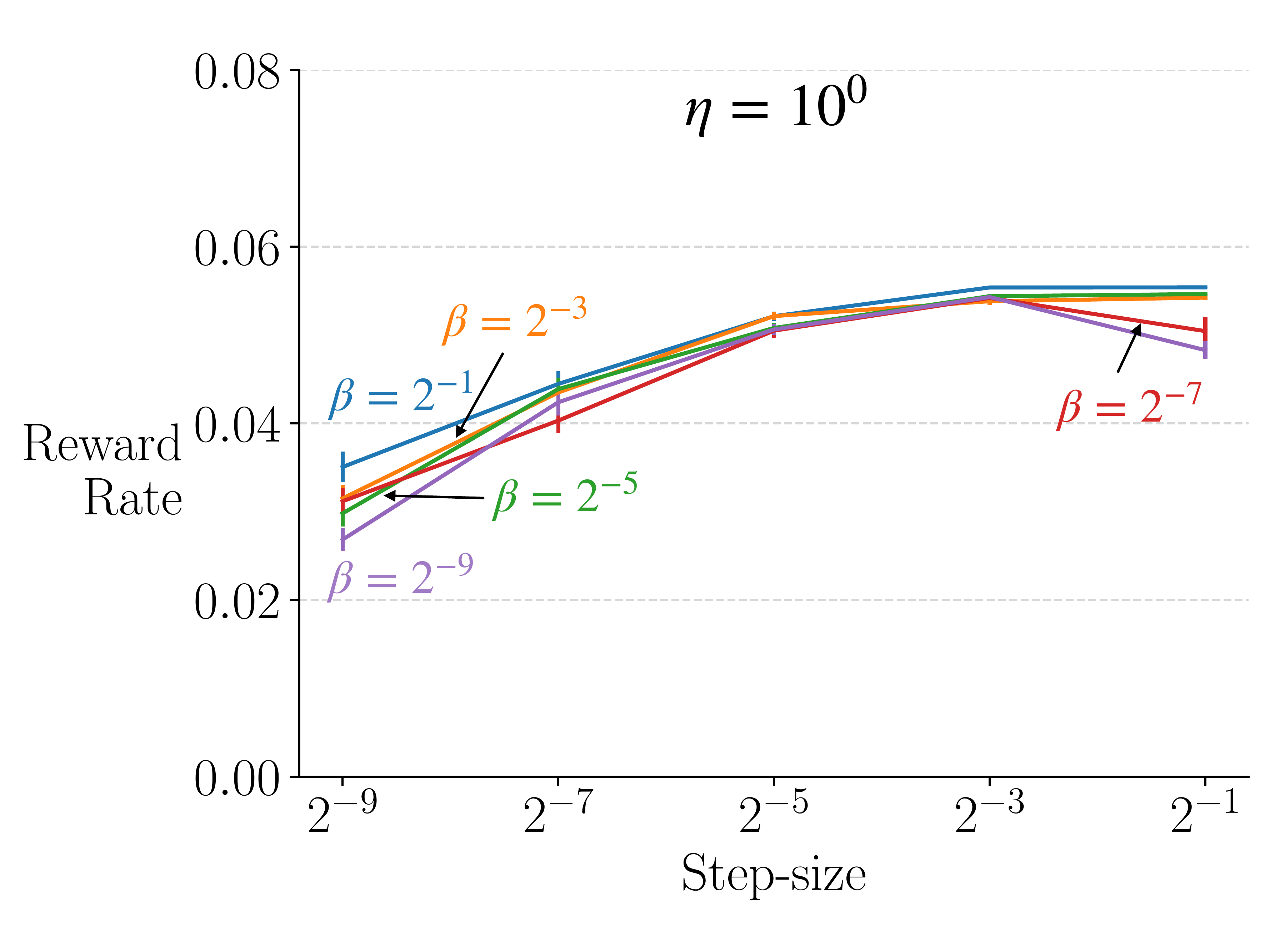}
\end{subfigure}%
\begin{subfigure}{.5\textwidth}
    \centering
    \includegraphics[width=0.9\textwidth]{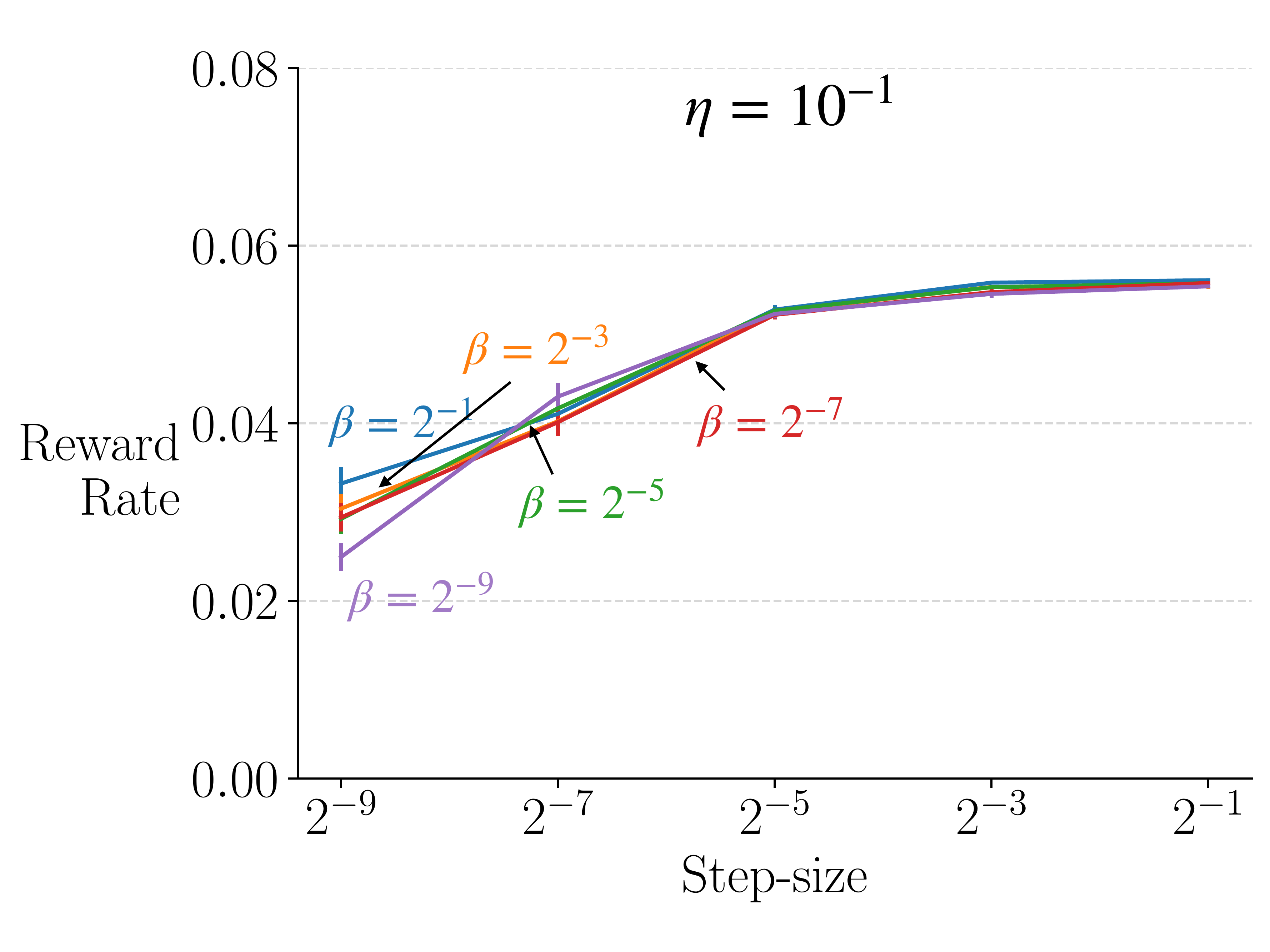}
\end{subfigure}
\begin{subfigure}{.5\textwidth}
    \centering
    \includegraphics[width=0.9\textwidth]{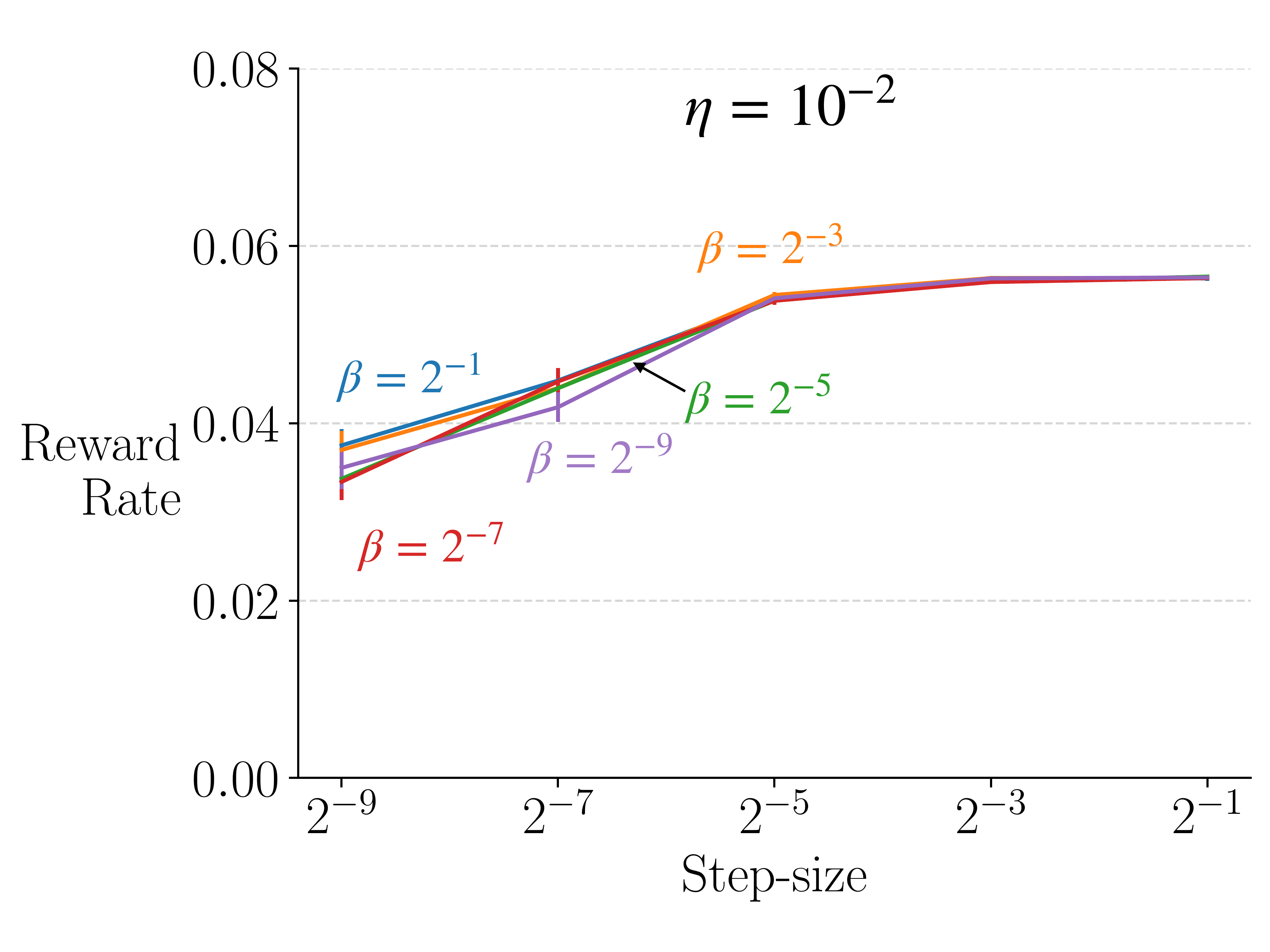}
\end{subfigure}%
\begin{subfigure}{.5\textwidth}
    \centering
    \includegraphics[width=0.9\textwidth]{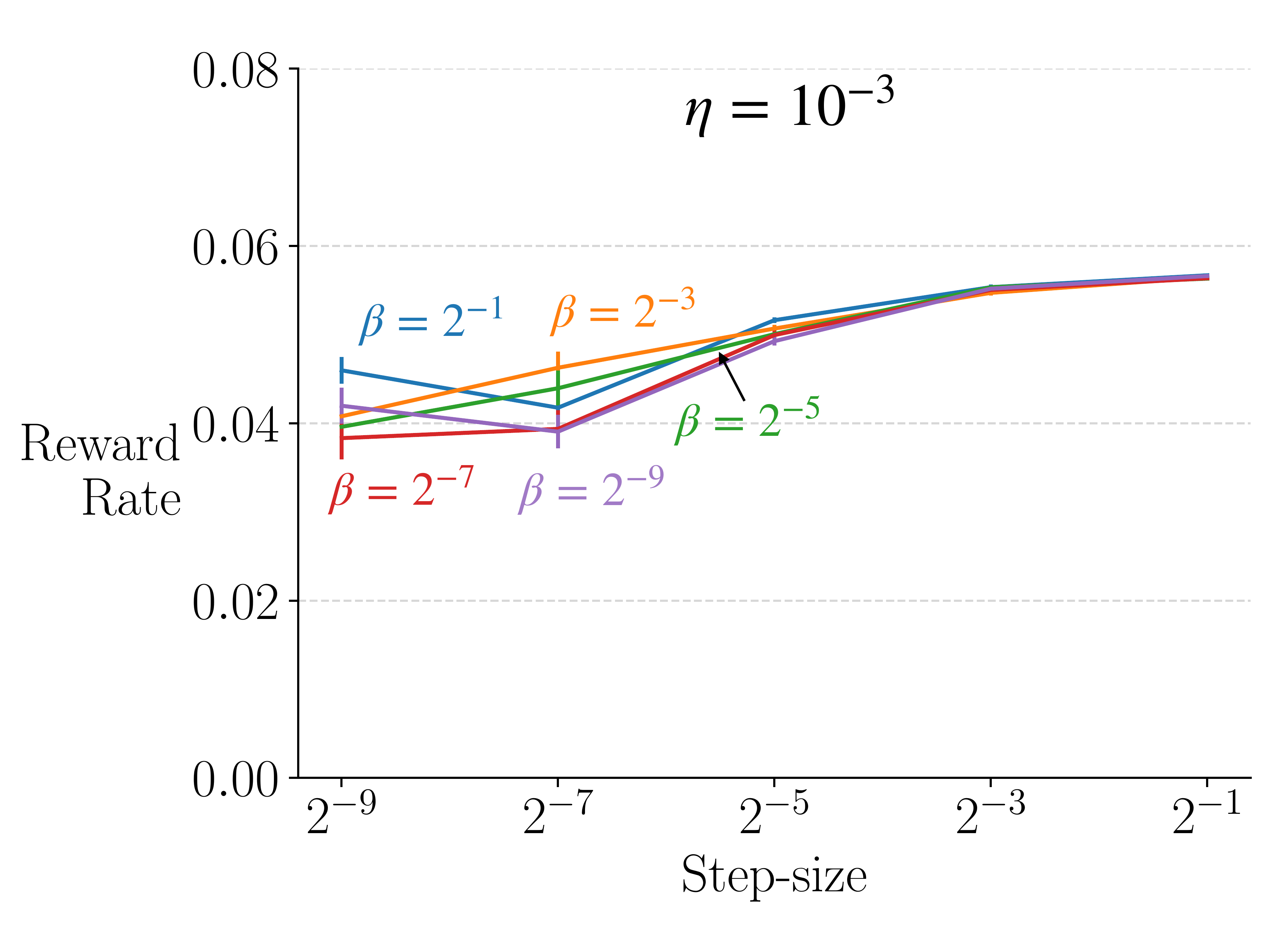}
\end{subfigure}
\begin{subfigure}{.5\textwidth}
    \centering
    \includegraphics[width=0.9\textwidth]{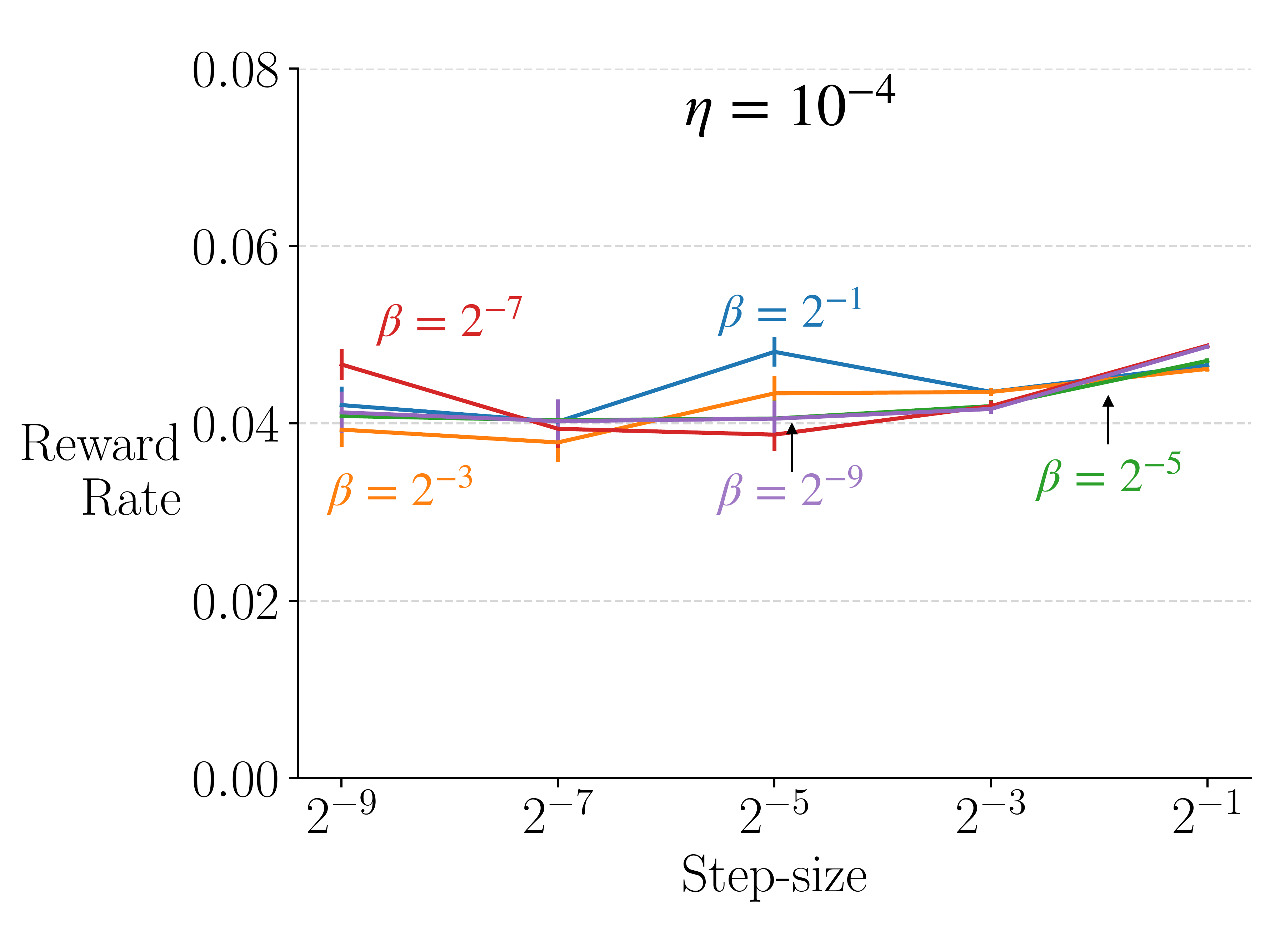}
\end{subfigure}%
    \caption{Plots showing a parameter study for inter-option Differential Q-learning and the set of options $ \calO = \calH + \calA$ in the continuing Four-Room domain when the goal was to go to \texttt{G1}. 
    Same experimental setups are used as what was described in \cref{sec: inter-option}.
    The x-axis denotes step size $\alpha$; the y-axis denotes the rate of the rewards averaged over all 200,000 steps of training, reflecting the rate of learning. The error bars denote one standard error. The algorithm's rate of learning varied little over a broad range of its parameters $\alpha, \beta$ and $\eta$. Small standard error bars show that the algorithm's performance varied little over multiple runs.}
    \label{fig: inter-option_learning_para_study}
\end{figure*}

\begin{figure*}[h!]
\centering
\begin{subfigure}{.5\textwidth}
    \centering
    \includegraphics[width=0.9\textwidth]{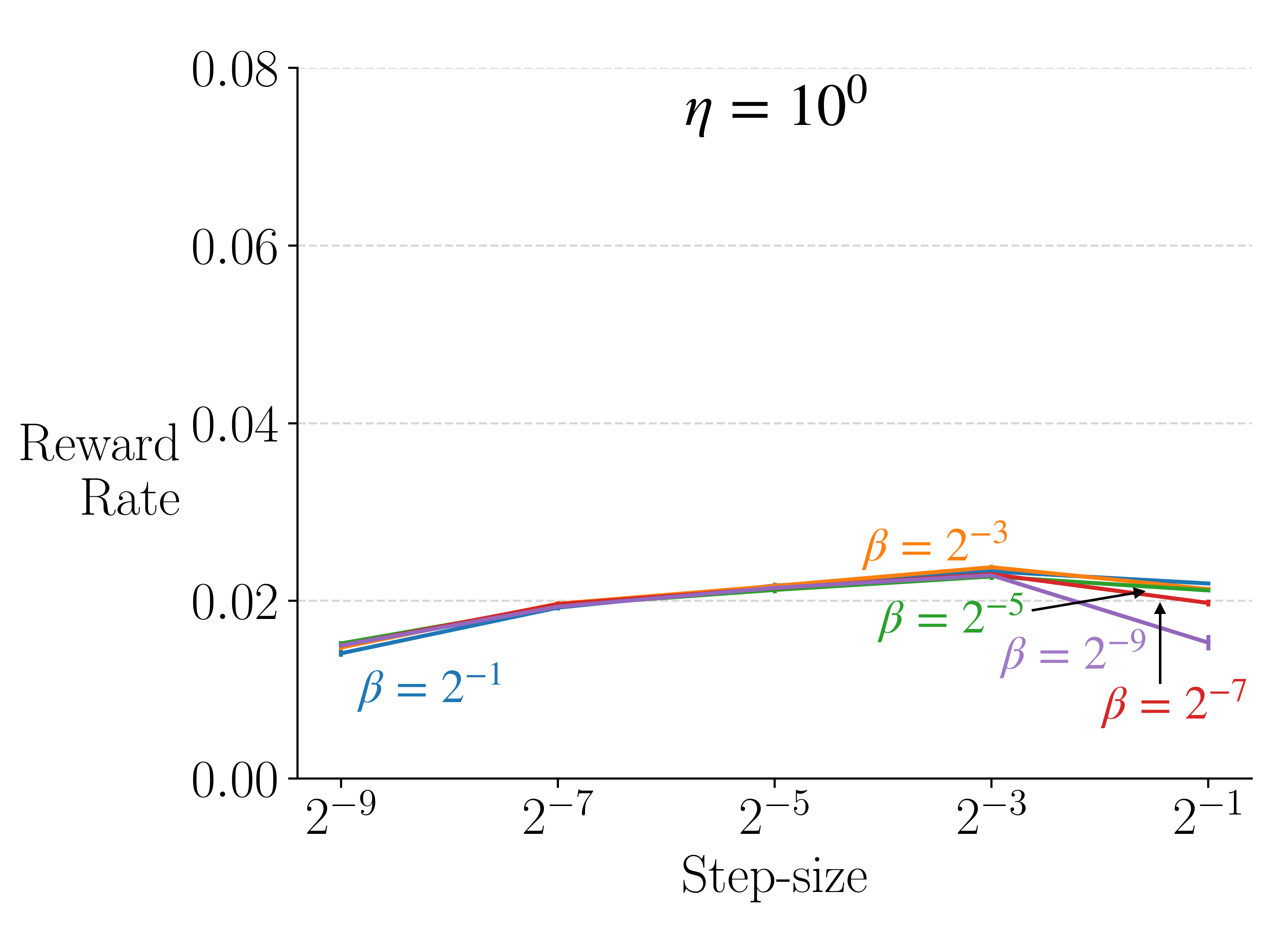}
\end{subfigure}%
\begin{subfigure}{.5\textwidth}
    \centering
    \includegraphics[width=0.9\textwidth]{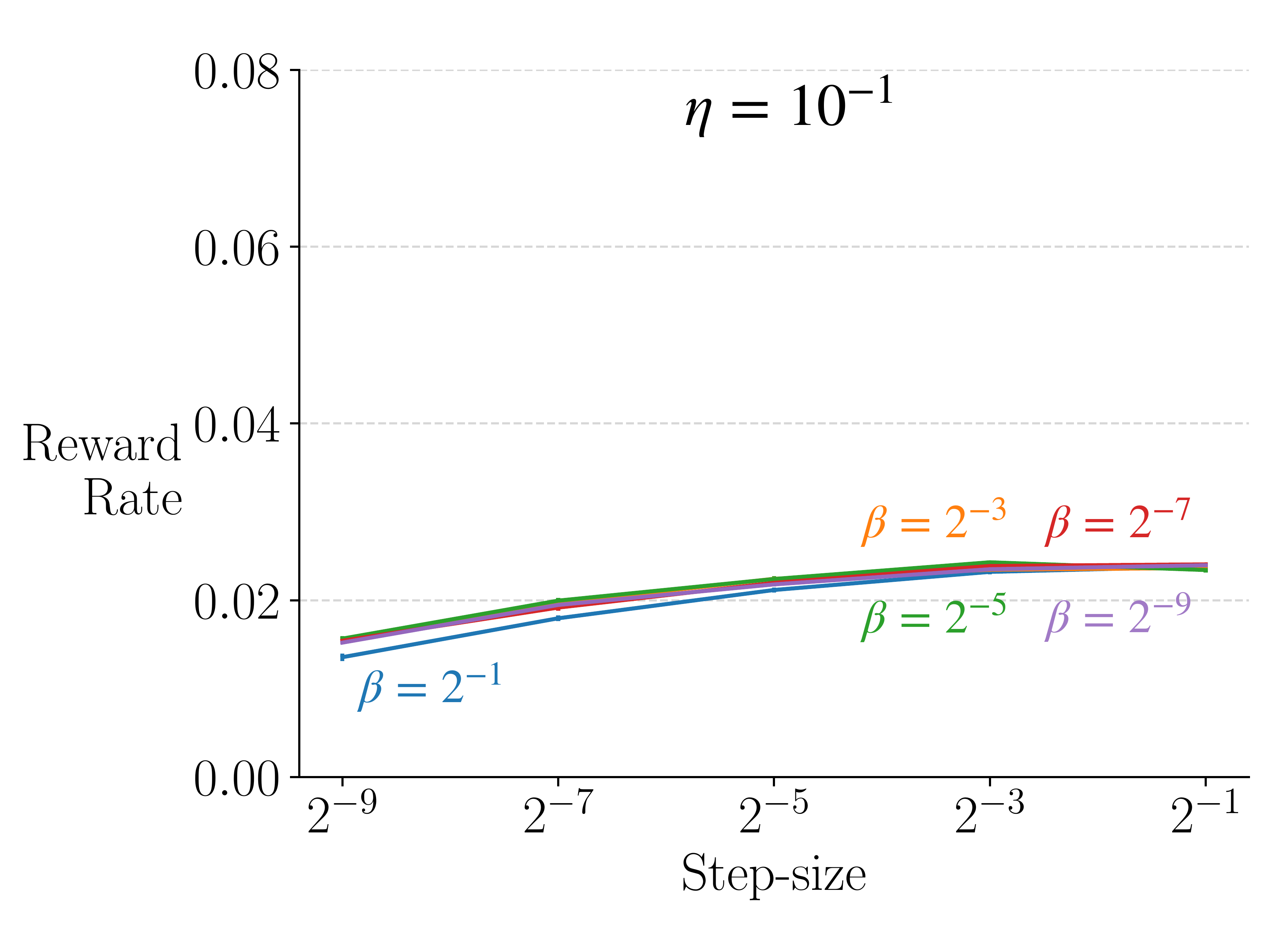}
\end{subfigure}
\begin{subfigure}{.5\textwidth}
    \centering
    \includegraphics[width=0.9\textwidth]{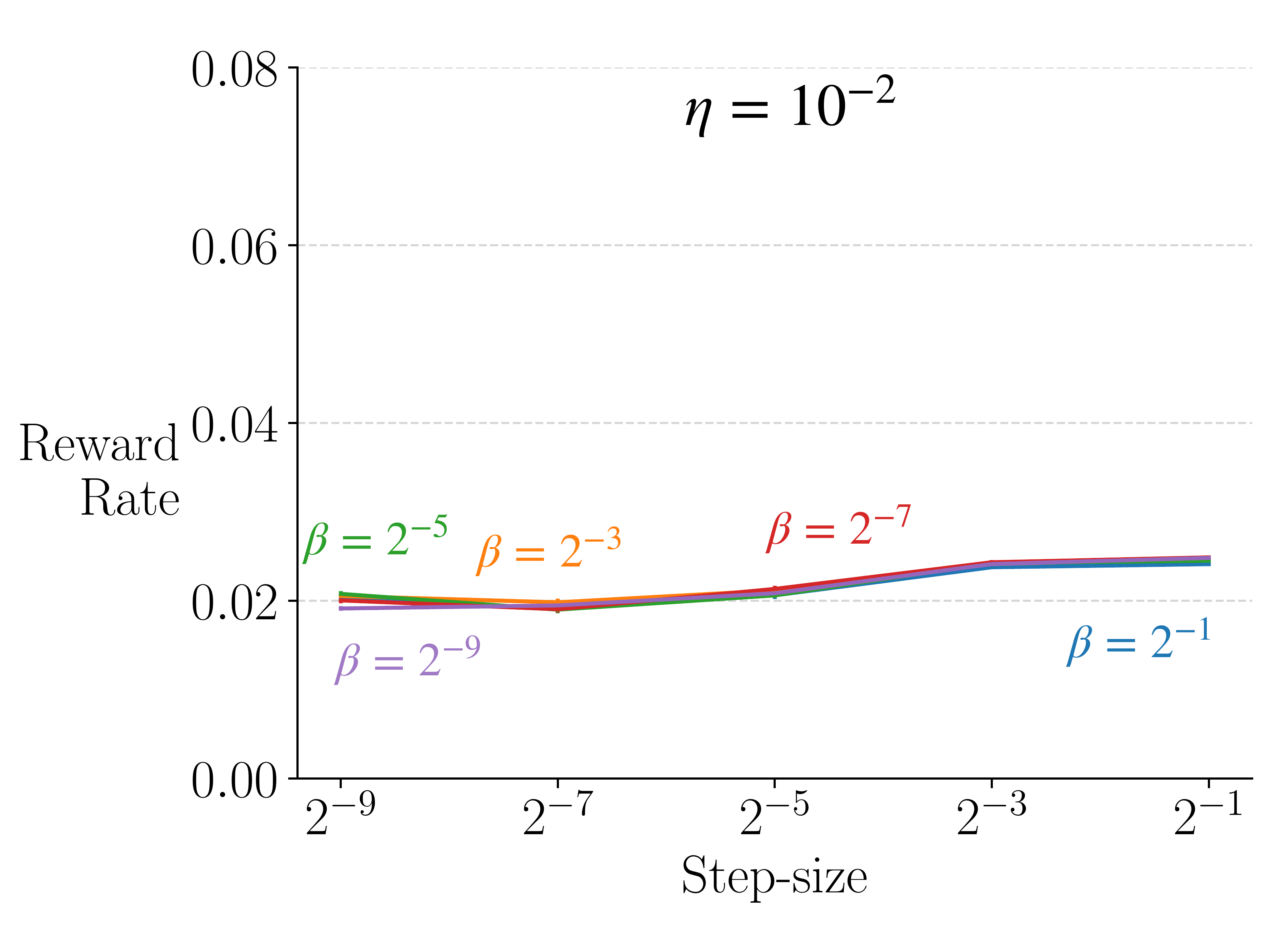}
\end{subfigure}%
\begin{subfigure}{.5\textwidth}
    \centering
    \includegraphics[width=0.9\textwidth]{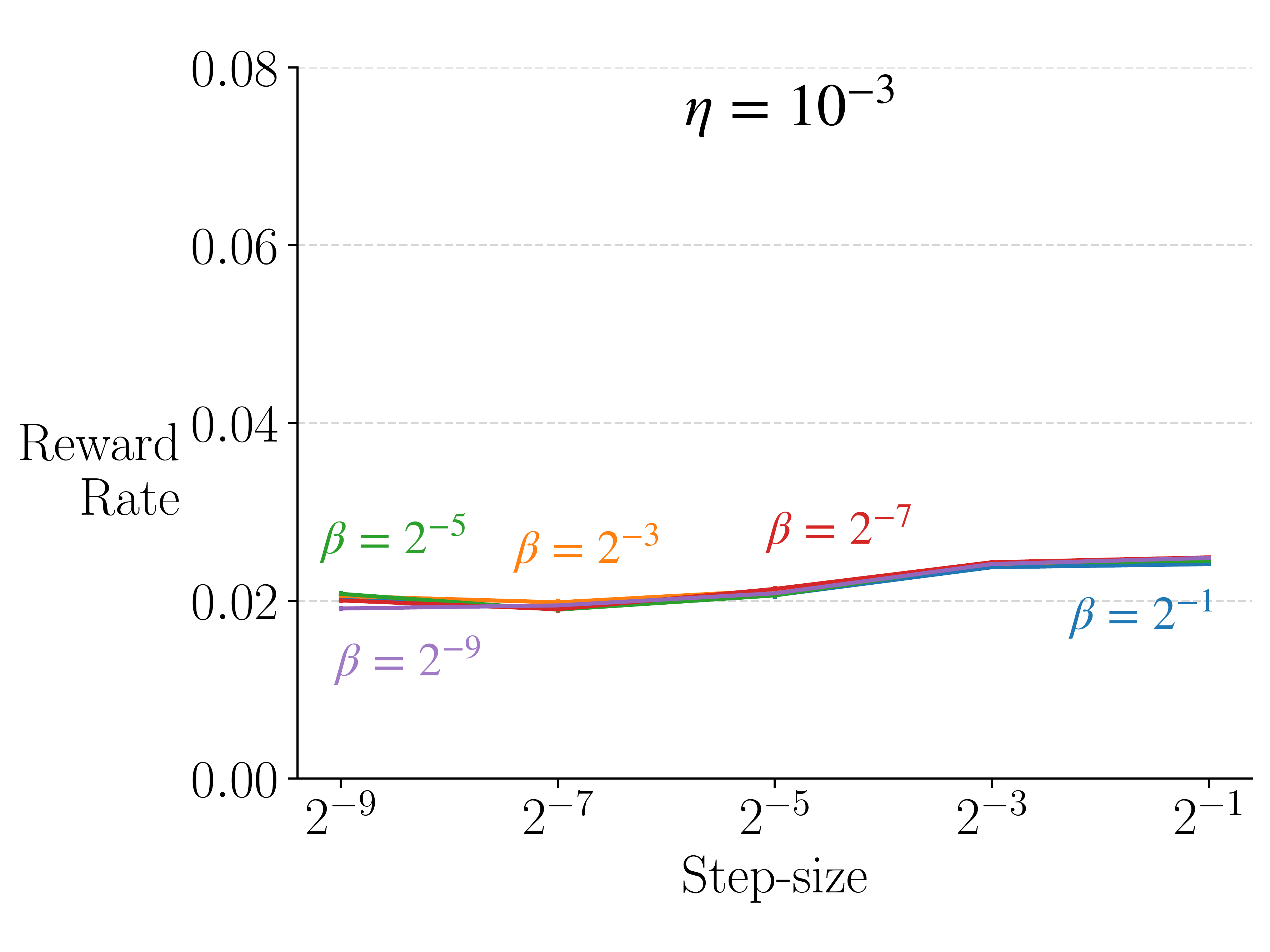}
\end{subfigure}
\begin{subfigure}{.5\textwidth}
    \centering
    \includegraphics[width=0.9\textwidth]{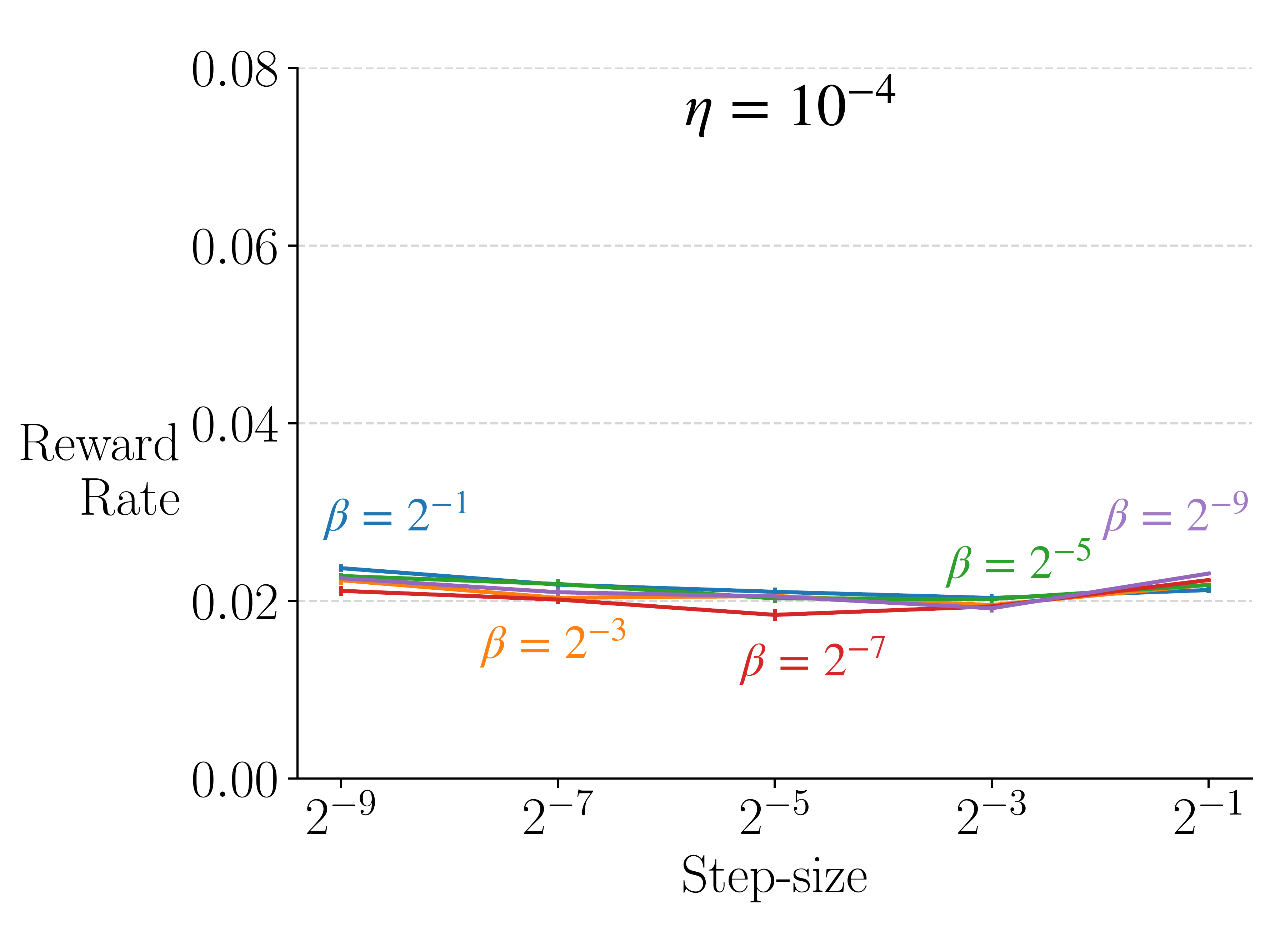}
\end{subfigure}%
    \caption{Plots showing a parameter study for inter-option Differential Q-learning and the set of options $ \calO = \calH$ in the continuing Four-Room domain when the goal was to go to \texttt{G1}. The experimental setting and the plot axes are the same as mentioned in \cref{fig: inter-option_learning_para_study}. Compared with \cref{fig: inter-option_learning_para_study}, it can be seen that the algorithm's rate of learning with $\calO = \calH$ was worse than it with $\calO = \calH + \calA$. This is because there is no hallway option from $\calH$ can takes the agent to \texttt{G1}. The algorithm's rate of learning varied little over a broad range of its parameters $\alpha, \beta$ and $\eta$, and also varied little over multiple runs.}
    \label{fig: inter-option_learning_options_para_study}
\end{figure*}

\begin{figure*}[h!]
\centering
\begin{subfigure}{.5\textwidth}
    \centering
    \includegraphics[width=0.9\textwidth]{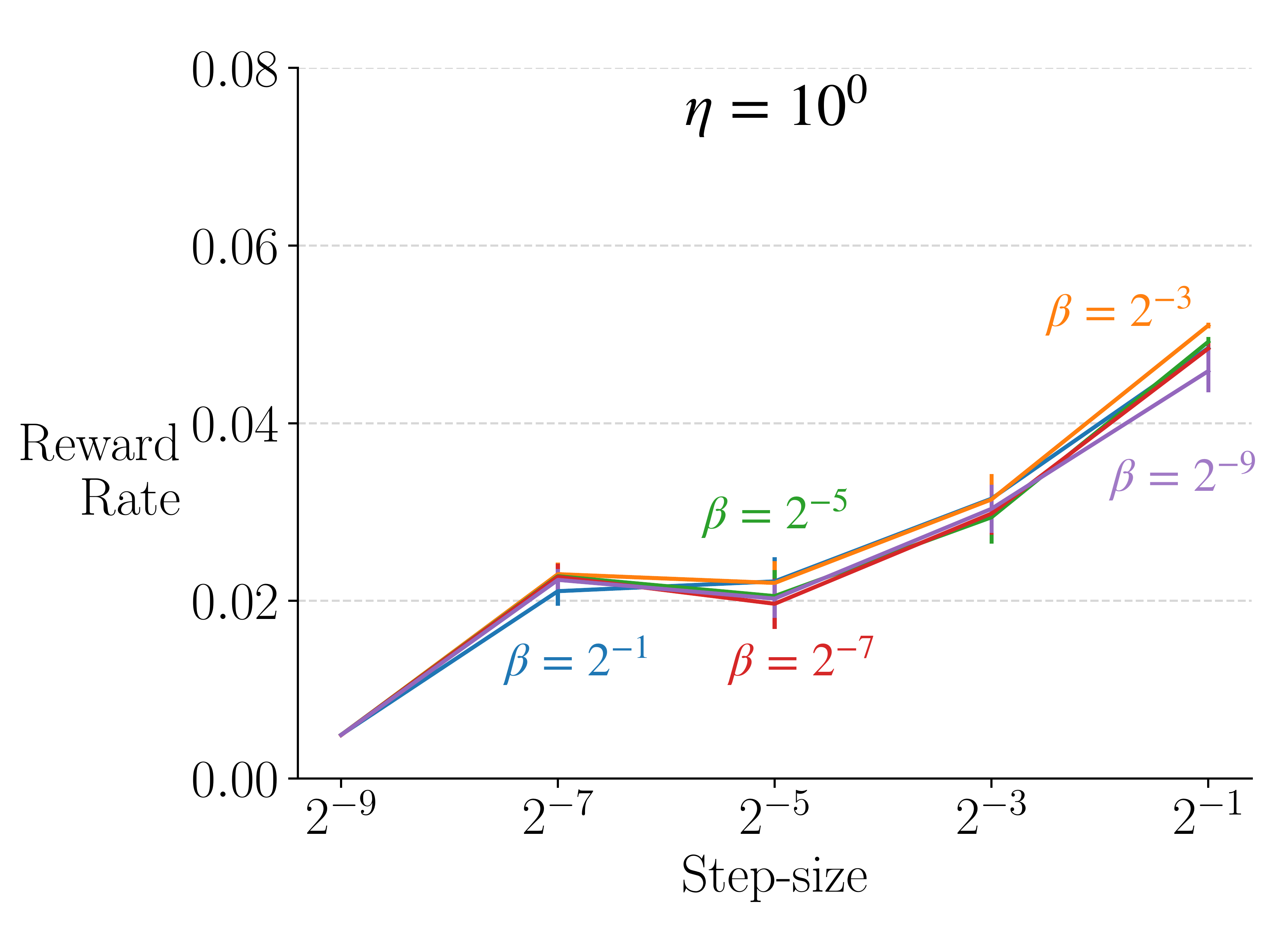}
\end{subfigure}%
\begin{subfigure}{.5\textwidth}
    \centering
    \includegraphics[width=0.9\textwidth]{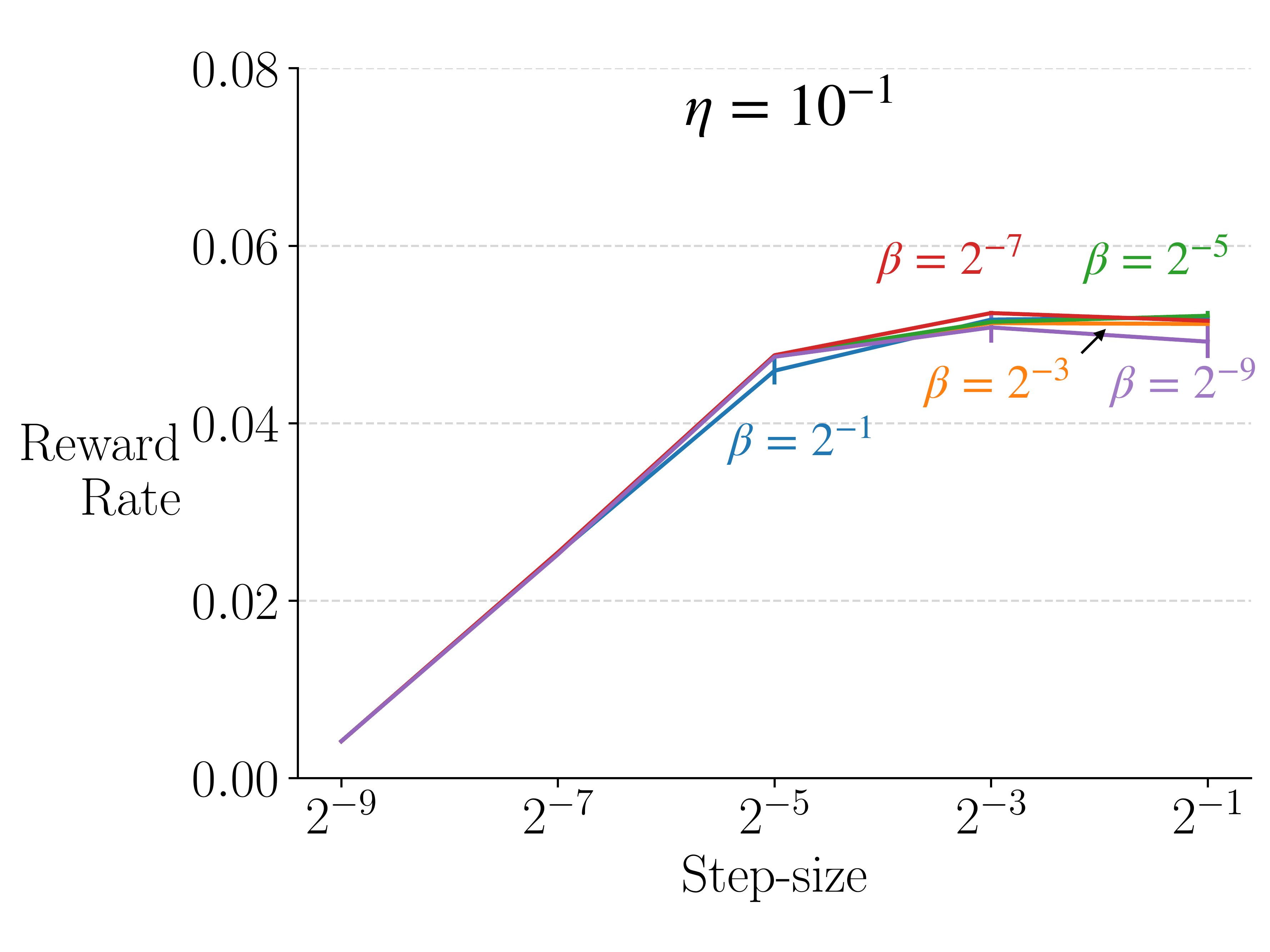}
\end{subfigure}
\begin{subfigure}{.5\textwidth}
    \centering
    \includegraphics[width=0.9\textwidth]{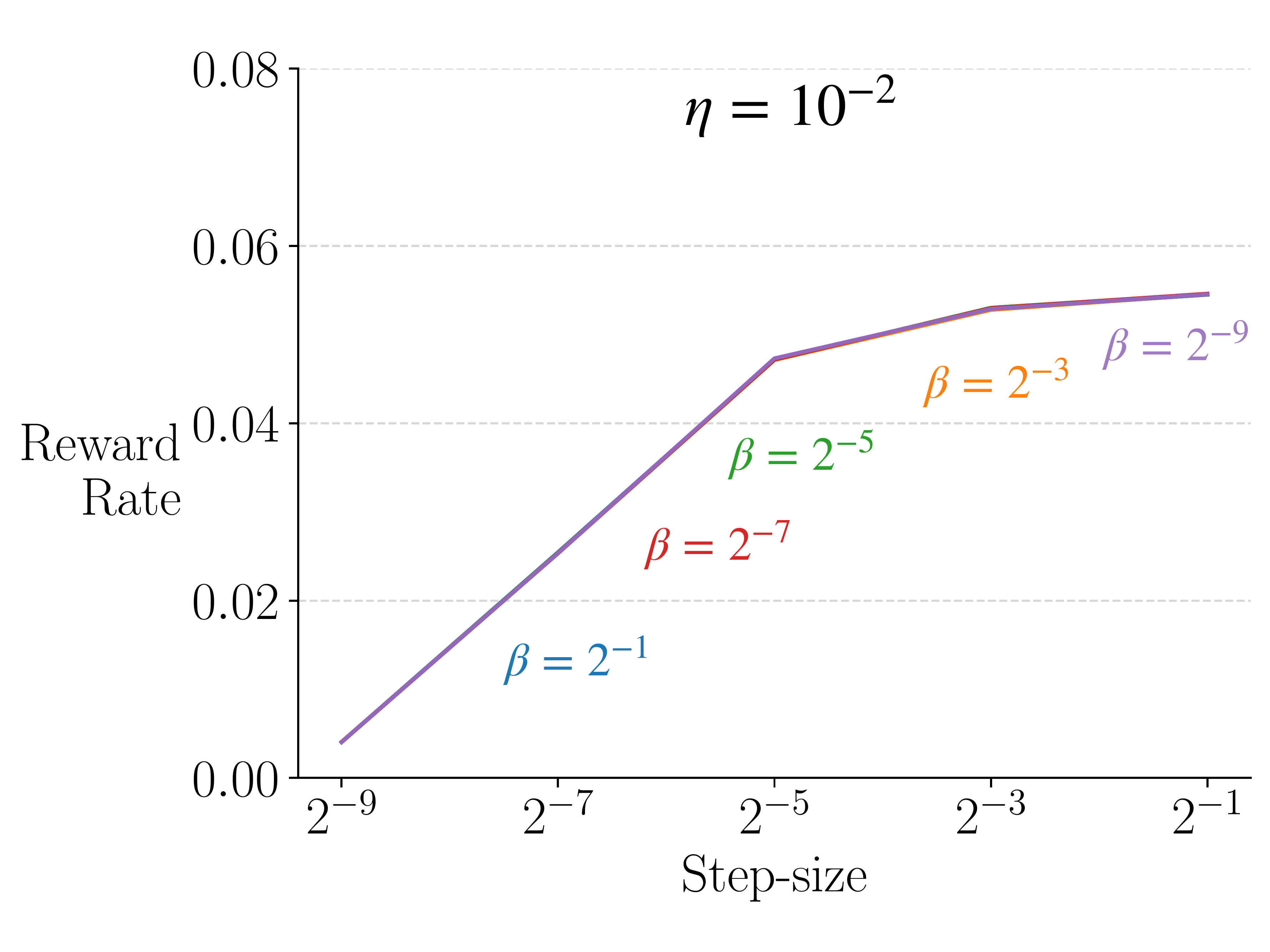}
\end{subfigure}%
\begin{subfigure}{.5\textwidth}
    \centering
    \includegraphics[width=0.9\textwidth]{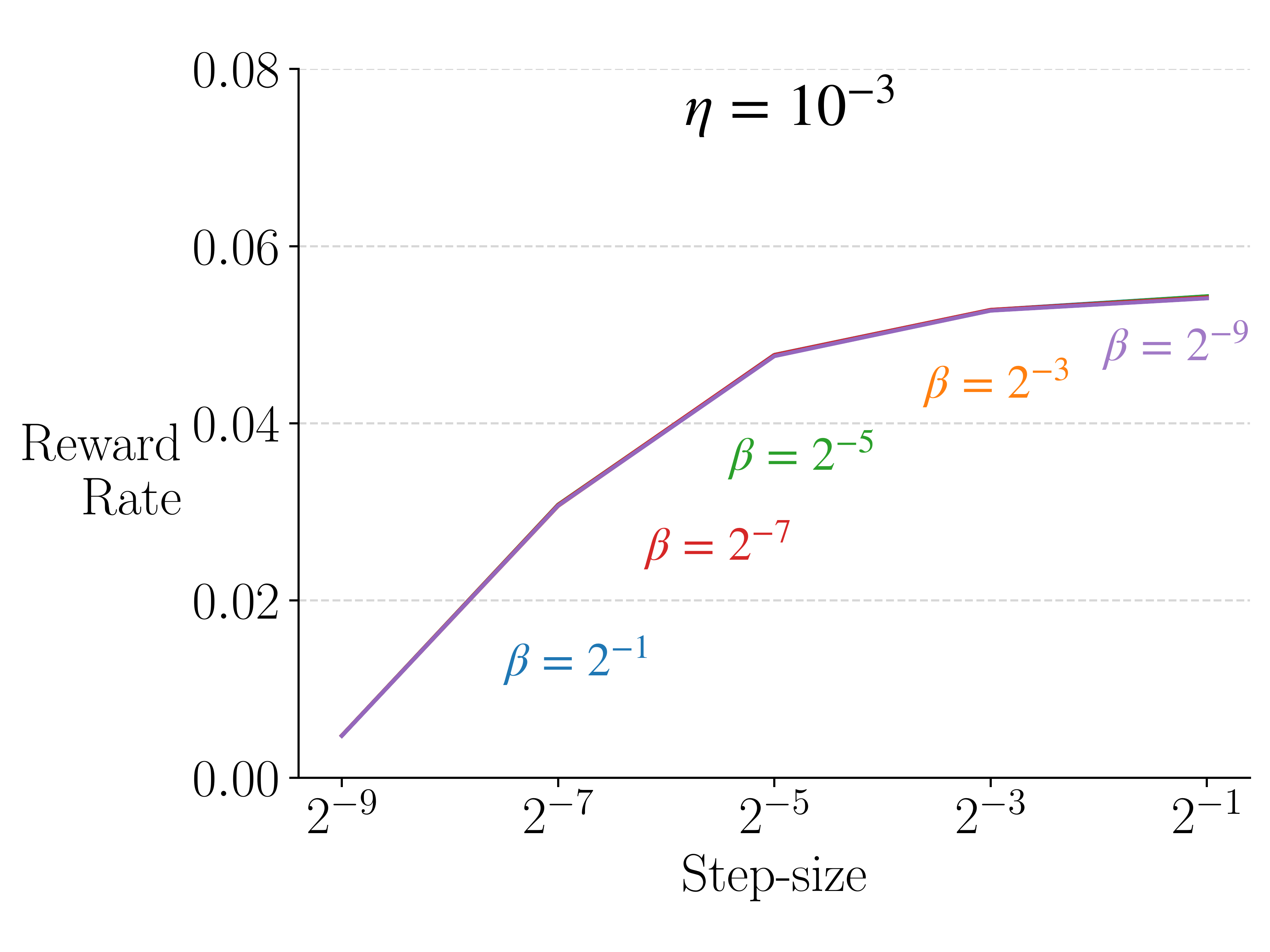}
\end{subfigure}
\begin{subfigure}{.5\textwidth}
    \centering
    \includegraphics[width=0.9\textwidth]{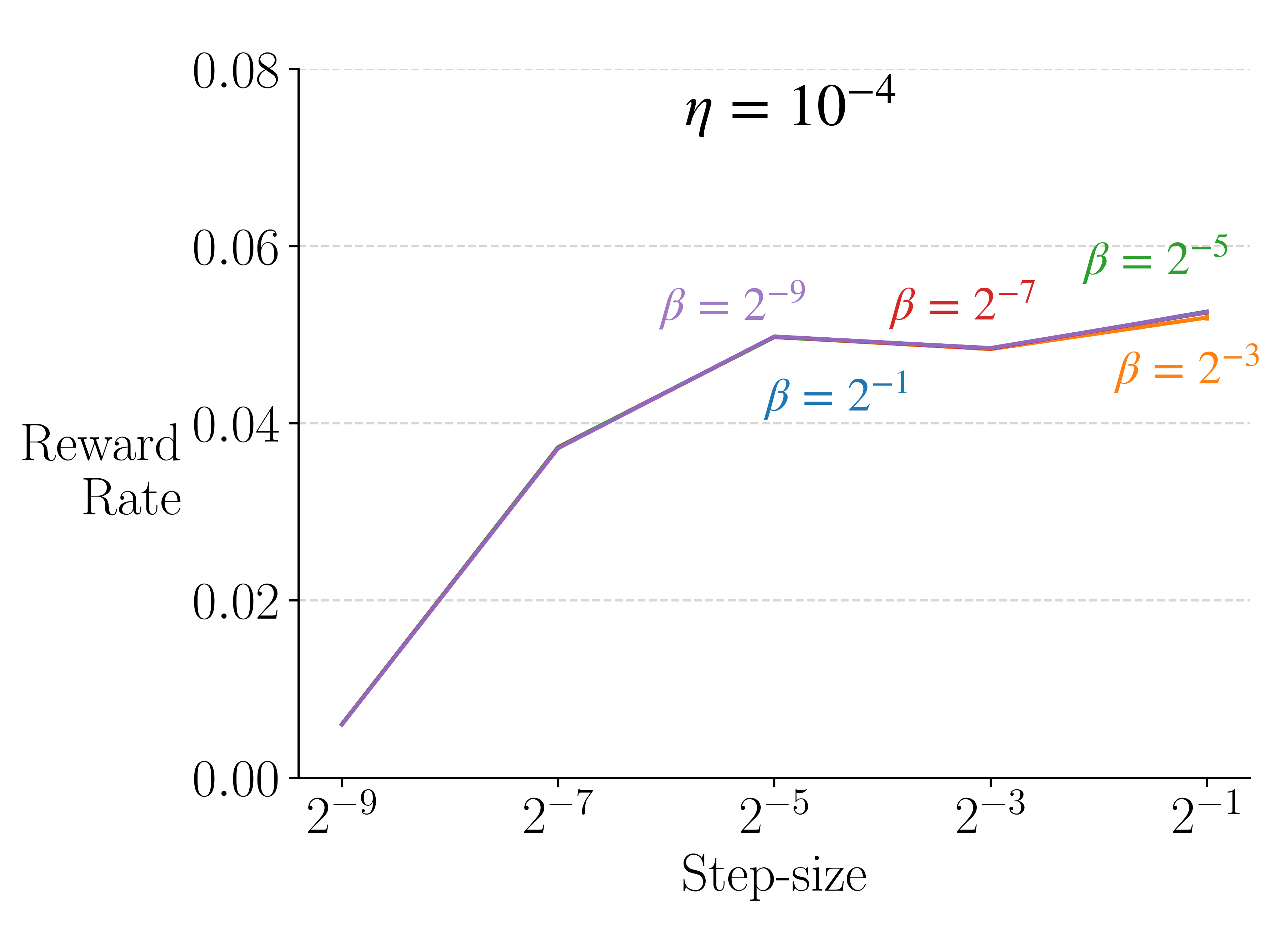}
\end{subfigure}%
    \caption{Plots showing a parameter study for inter-option Differential Q-learning and the set of options $\calO = \calA$ in the continuing Four-Room domain when the goal was to go to \texttt{G1}. Note that with options being primitive actions, the algorithm becomes exactly the same as Differential Q-learning by Wan et al.\ (2021). The experimental setting and the plot axes are the same as mentioned in \cref{fig: inter-option_learning_para_study}. Compared with \cref{fig: inter-option_learning_para_study}, it can be seen that the algorithm's rate of learning with $\calO = \calA$ was worse than it with $\calO = \calH + \calA$, particularly for small $\alpha$. The algorithm's rate of learning did not vary too much over a broad range of its parameters $\beta$ and $\eta$, and also varied little over multiple runs. The algorithm's performance is more sensitive to the choice of $\alpha$.}
    \label{fig: inter-option_learning_actions_para_study}
\end{figure*}

\newpage\phantom{blabla}

\newpage\phantom{blabla}

\newpage
\subsection{Intra-option Q-learning}
\label{app: additional empirical results: Intra-option Learning}

\begin{figure*}[h!]
\centering
\begin{subfigure}{.5\textwidth}
    \centering
    \includegraphics[width=0.9\textwidth]{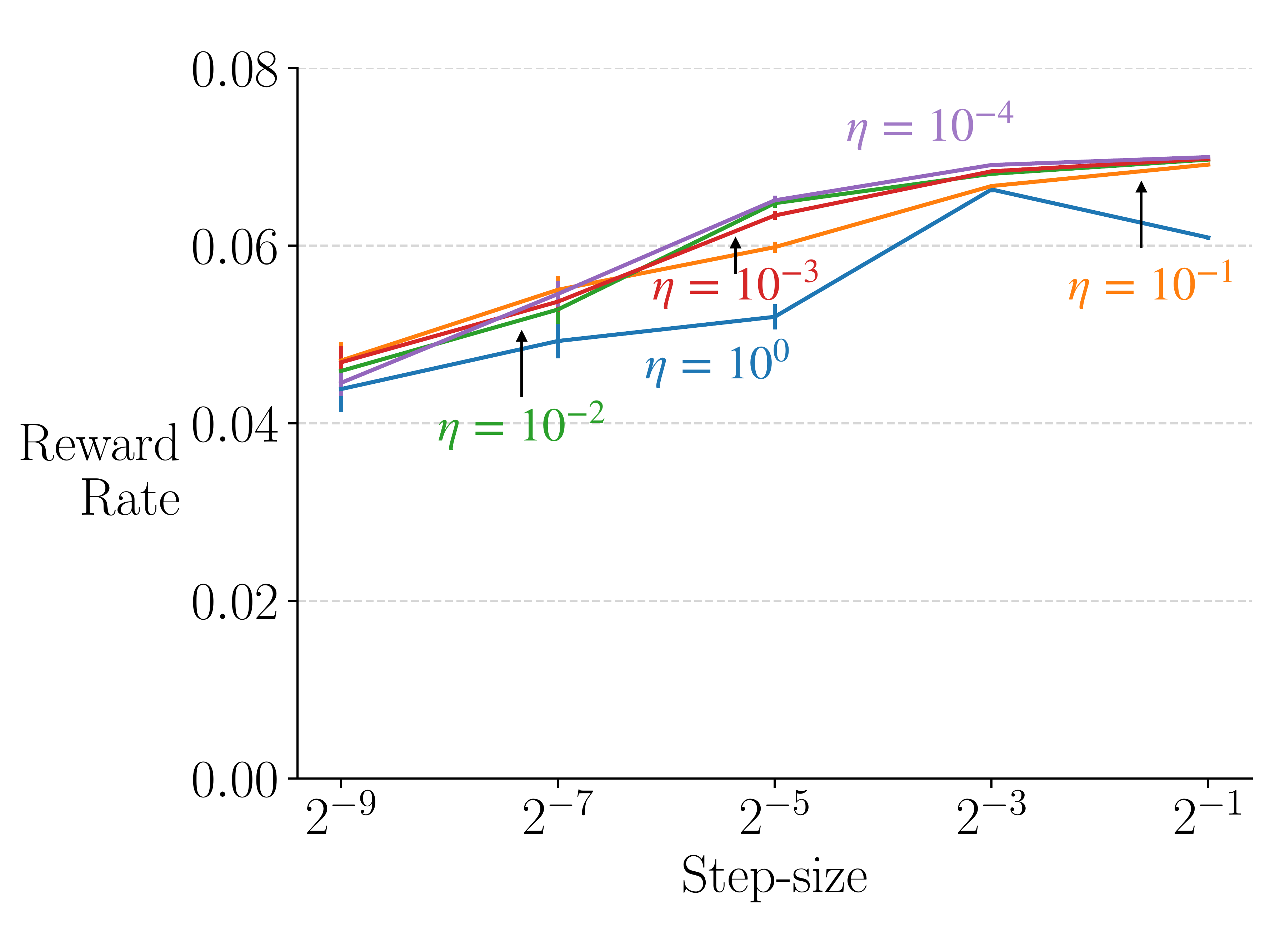}
\end{subfigure}
    \caption{Plots showing a parameter study for intra-option Differential Q-learning with the set of options $\calO = \calH$ in the continuing Four-Room domain when the goal was to go to \texttt{G2}. The algorithm used a behavior policy consisting only of primitive actions. The hallway options were never executed.. The experimental setting and the plot axes are the same as mentioned in \cref{sec: intra-option value}.
    The algorithm's rate of learning varied little over a broad range of its parameters $\alpha$ and $\eta$, and also varied little over multiple runs.}
    \label{fig: intra-option random actions para study}
\end{figure*}

\subsection{Interruption}
\label{app: additional empirical results: Interruption}
\begin{figure*}[h!]
\centering
\begin{subfigure}{.5\textwidth}
    \centering
    \includegraphics[width=0.9\textwidth]{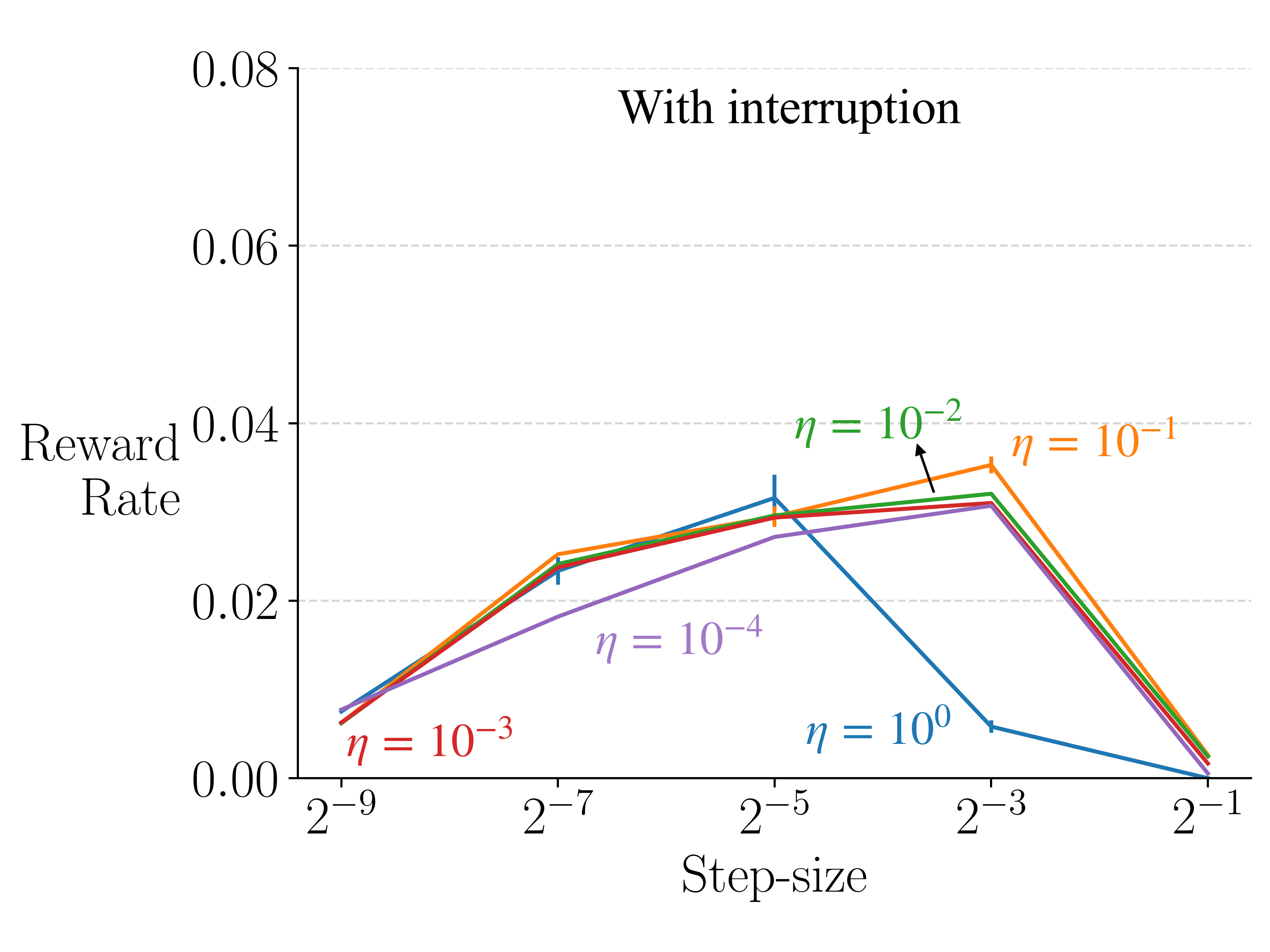}
\end{subfigure}%
\begin{subfigure}{.5\textwidth}
    \centering
    \includegraphics[width=0.9\textwidth]{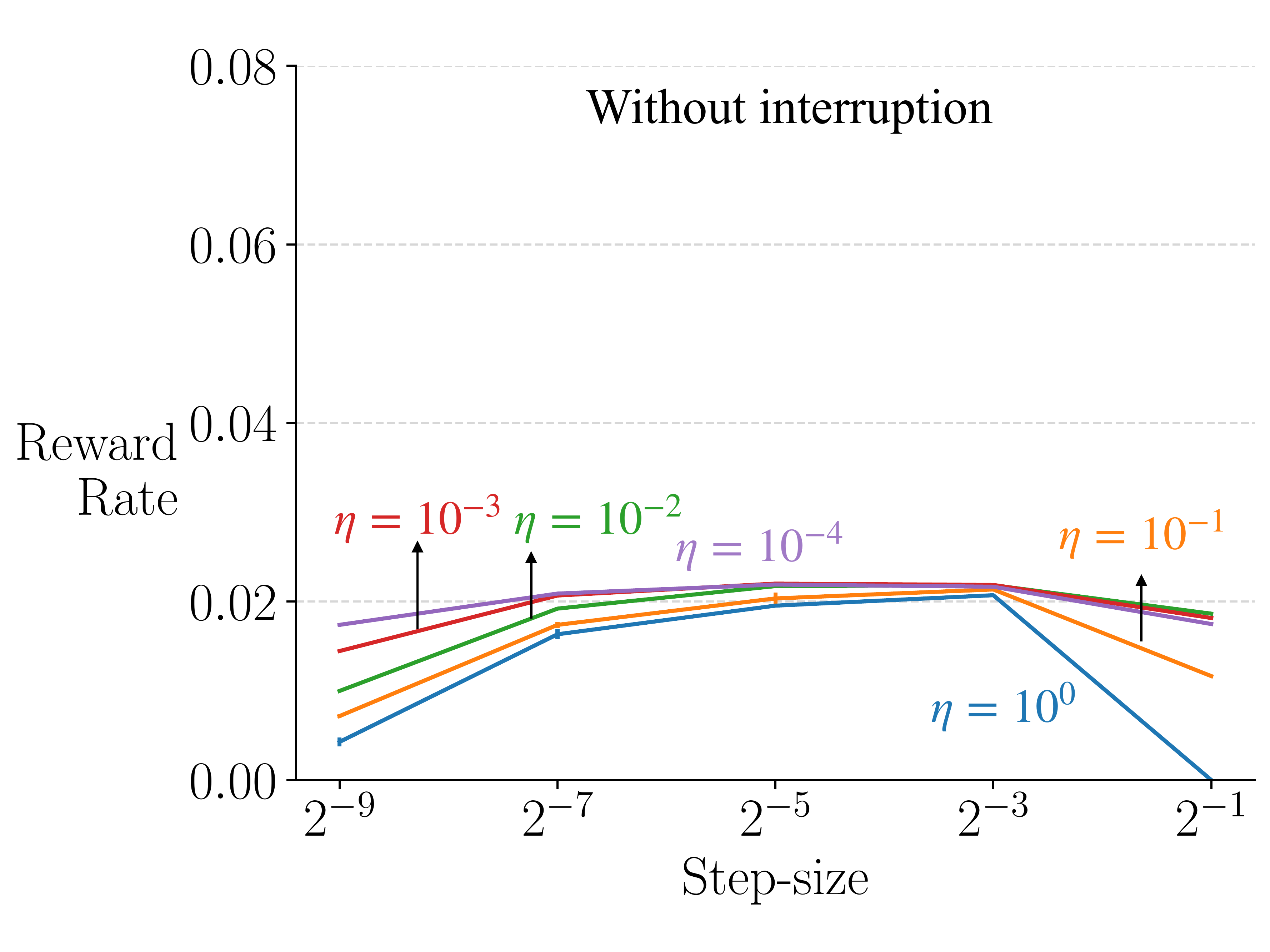}
\end{subfigure}
    \caption{
    Plots showing parameter studies for intra-option Differential Q-learning with and without interruption in the continuing Four-Room domain when the goal was to go to \texttt{G3}. The algorithm used the set of hallway options $\calO = \calH$. The experimental setting and the plot axes are the same as mentioned in \cref{sec: interruption}.
    The algorithm's rate of learning with interruption was higher than it without interruption for medium sized choices of $\alpha$. When a large or small $\alpha$ was used, interruption produced a worse rate of learning.
    The algorithm's rate of learning varied not too much over a broad range of its parameters $\eta$ and varied little over multiple runs, regardless of interruption. The algorithm's rate of learning was more sensitive to $\alpha$ when interruption is used.}
    \label{fig: interruption parameter study}
\end{figure*}

\subsection{Prediction Experiments}
\label{app: additional empirical results: Prediction}
We also performed a set of experiments to show that both inter- and intra-option Differential Q-evaluation can learn the reward rate well. The tested environment is the same as the one used to test inter-option Differential Q-learning (with \texttt{G1}). The set of options consists of 4 primitive actions and 8 hallway options. For each state, the behavior policy randomly picks an option. The target policy is an optimal policy, which induces a reward rate 0.0625. We ran both inter- and intra option Differential Q-evaluation in this problem. The parameters used are the same with those used in inter- and intra-option Differential Q-learning experiments. The sensitivity of the two algorithms w.r.t. the parameters is shown in Figure~\ref{fig: inter-option Q-evaluation parameter study} and Figure~\ref{fig: intra-option Q-evaluation parameter study}. Inter-option algorithm's reward rate error is quite robust to $\beta$. Intra-option algorithm's reward rate error is generally better than Inter-option algorithm's reward rate error unless a large stepsize like $\alpha = 0.5$ is used. 

\begin{figure*}[h!]
\centering
\begin{subfigure}{.5\textwidth}
    \centering
    \includegraphics[width=0.9\textwidth]{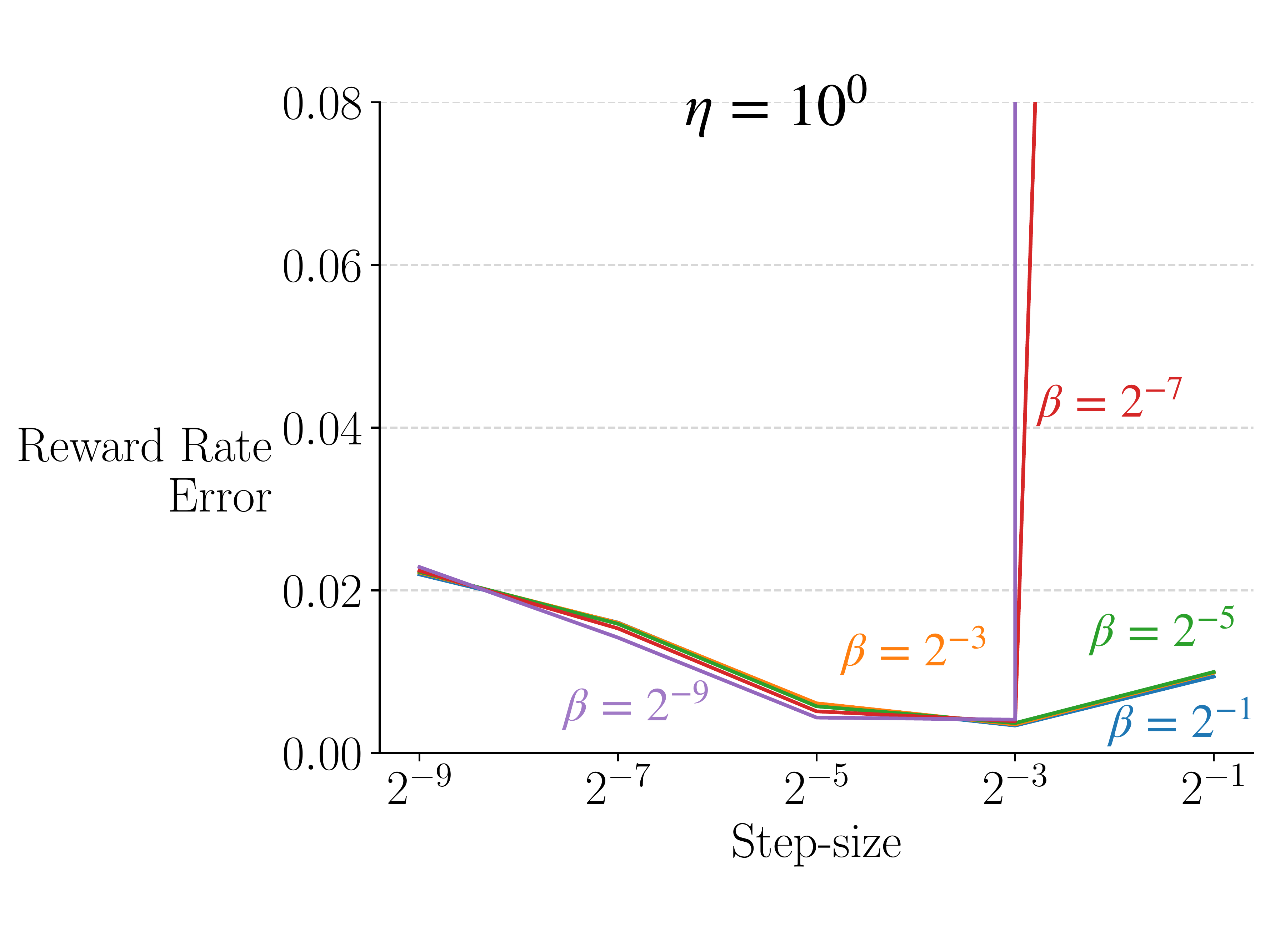}
\end{subfigure}%
\begin{subfigure}{.5\textwidth}
    \centering
    \includegraphics[width=0.9\textwidth]{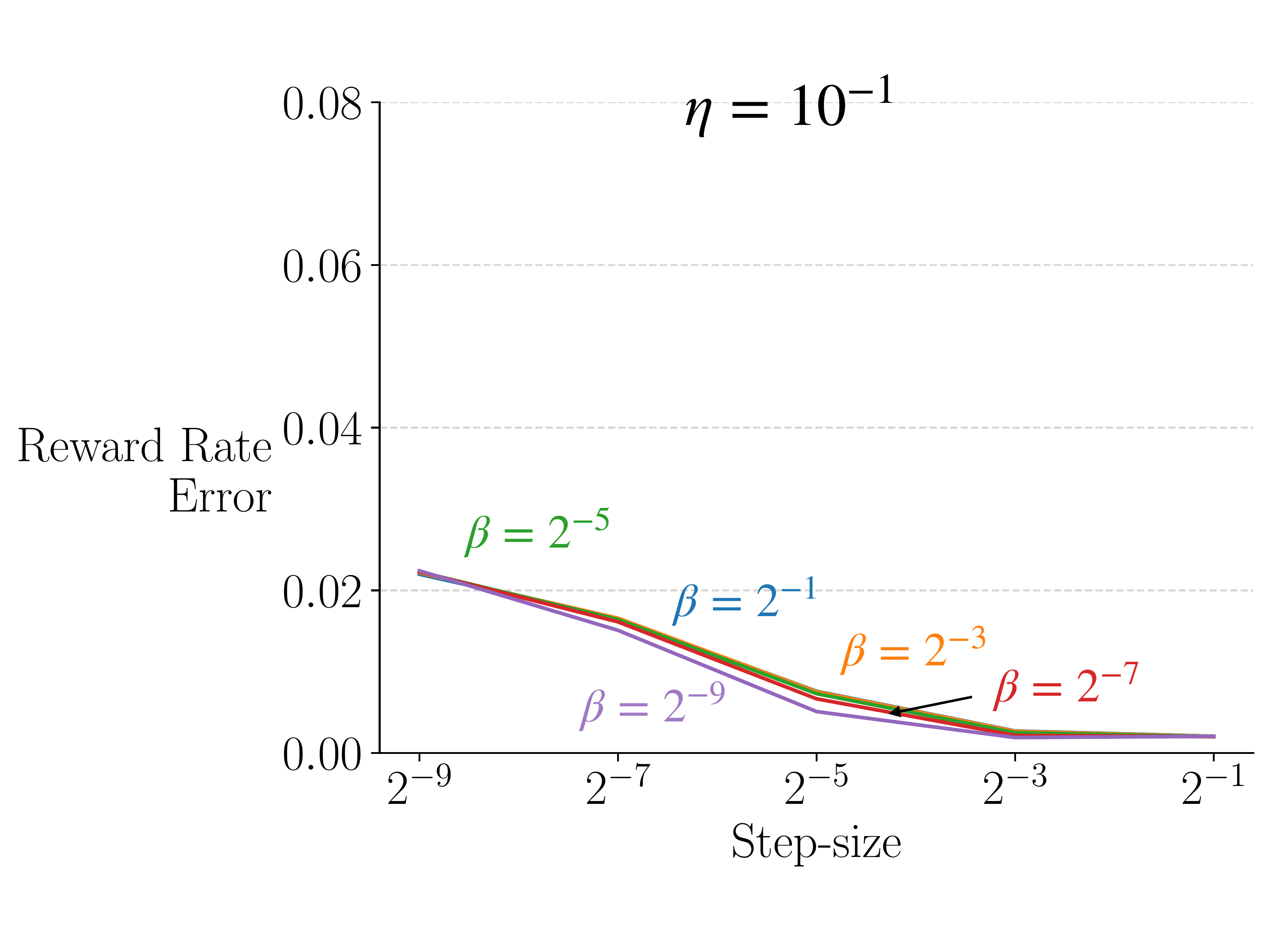}
\end{subfigure}
\begin{subfigure}{.5\textwidth}
    \centering
    \includegraphics[width=0.9\textwidth]{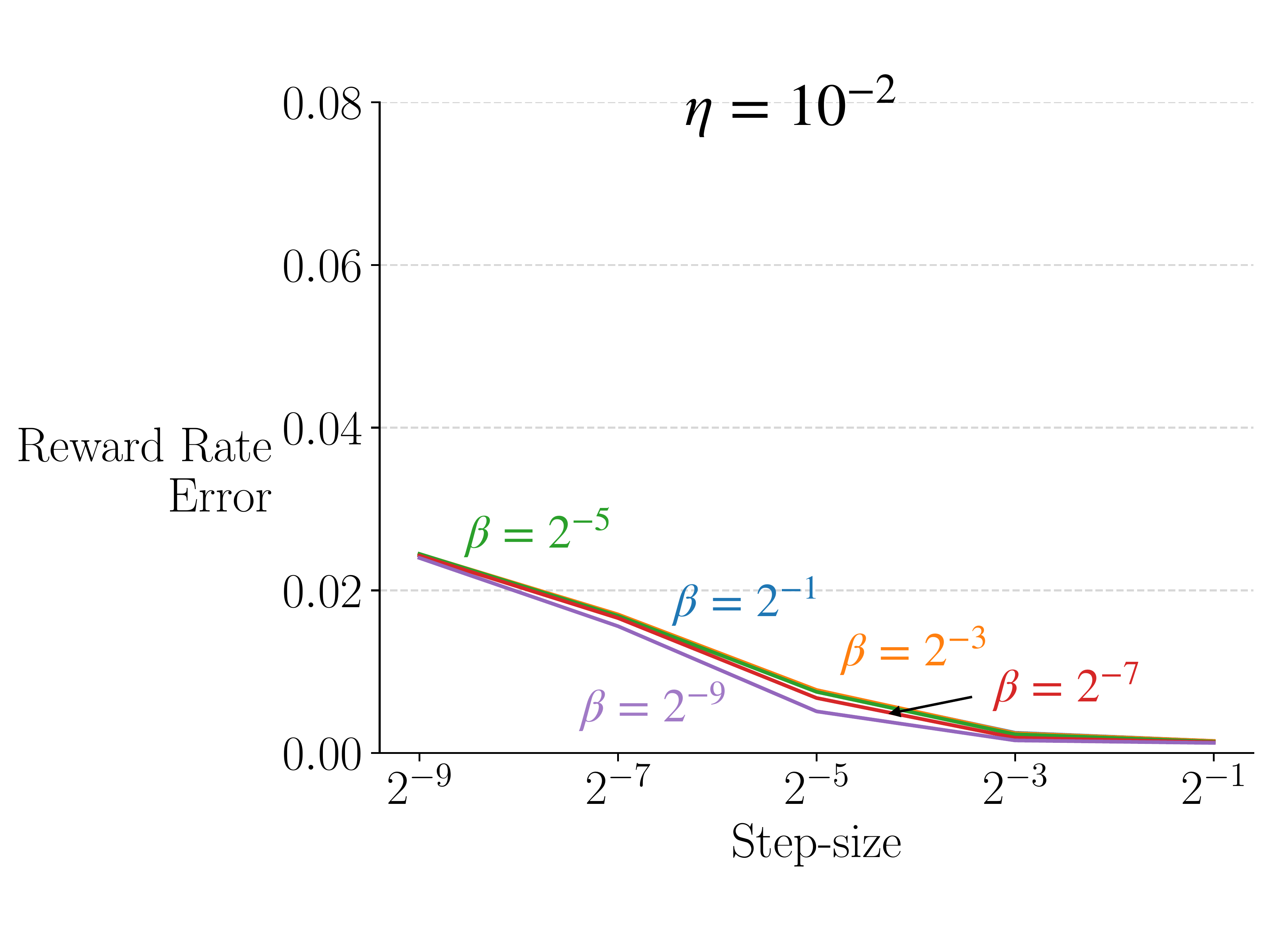}
\end{subfigure}%
\begin{subfigure}{.5\textwidth}
    \centering
    \includegraphics[width=0.9\textwidth]{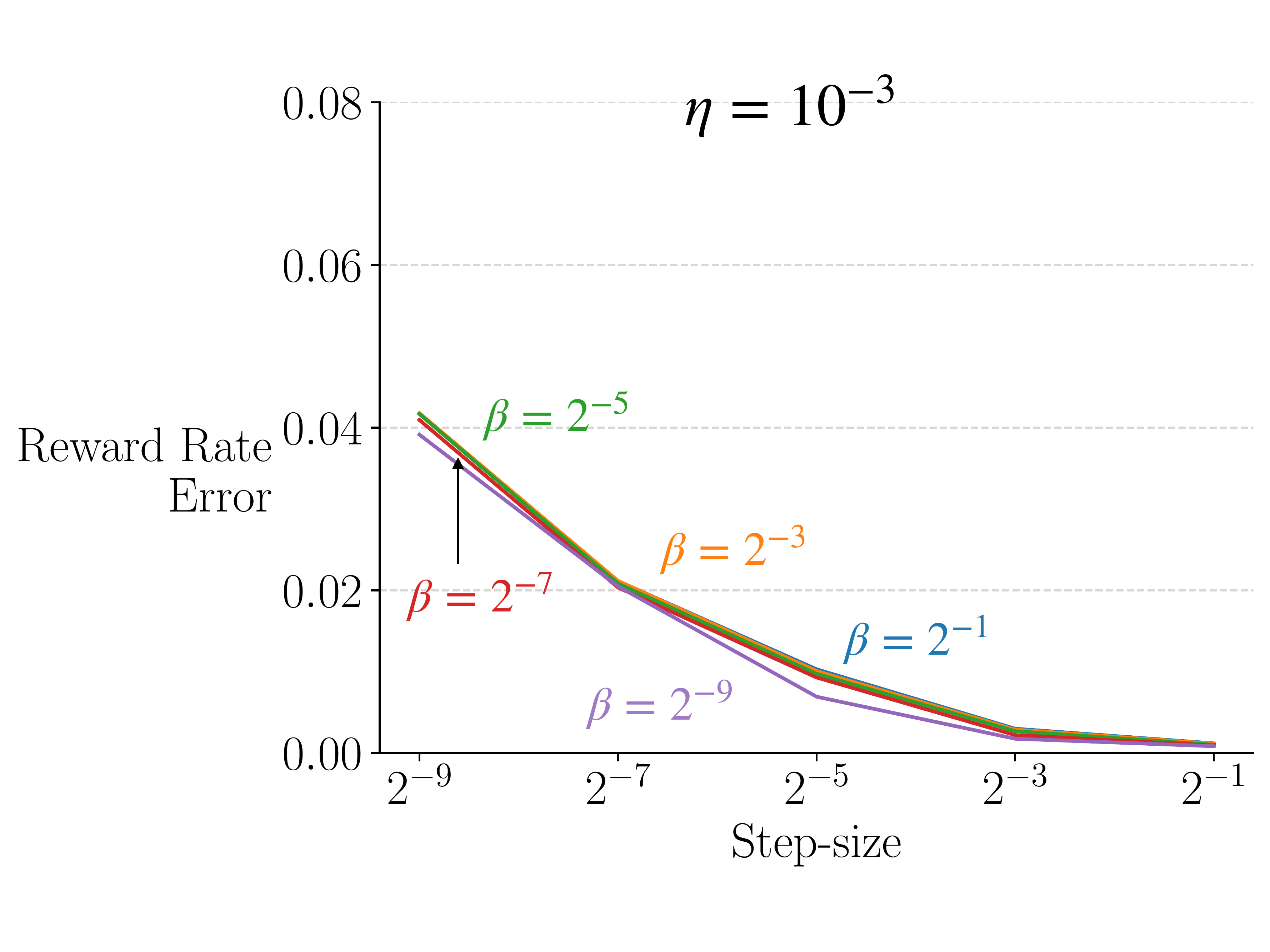}
\end{subfigure}
\begin{subfigure}{.5\textwidth}
    \centering
    \includegraphics[width=0.9\textwidth]{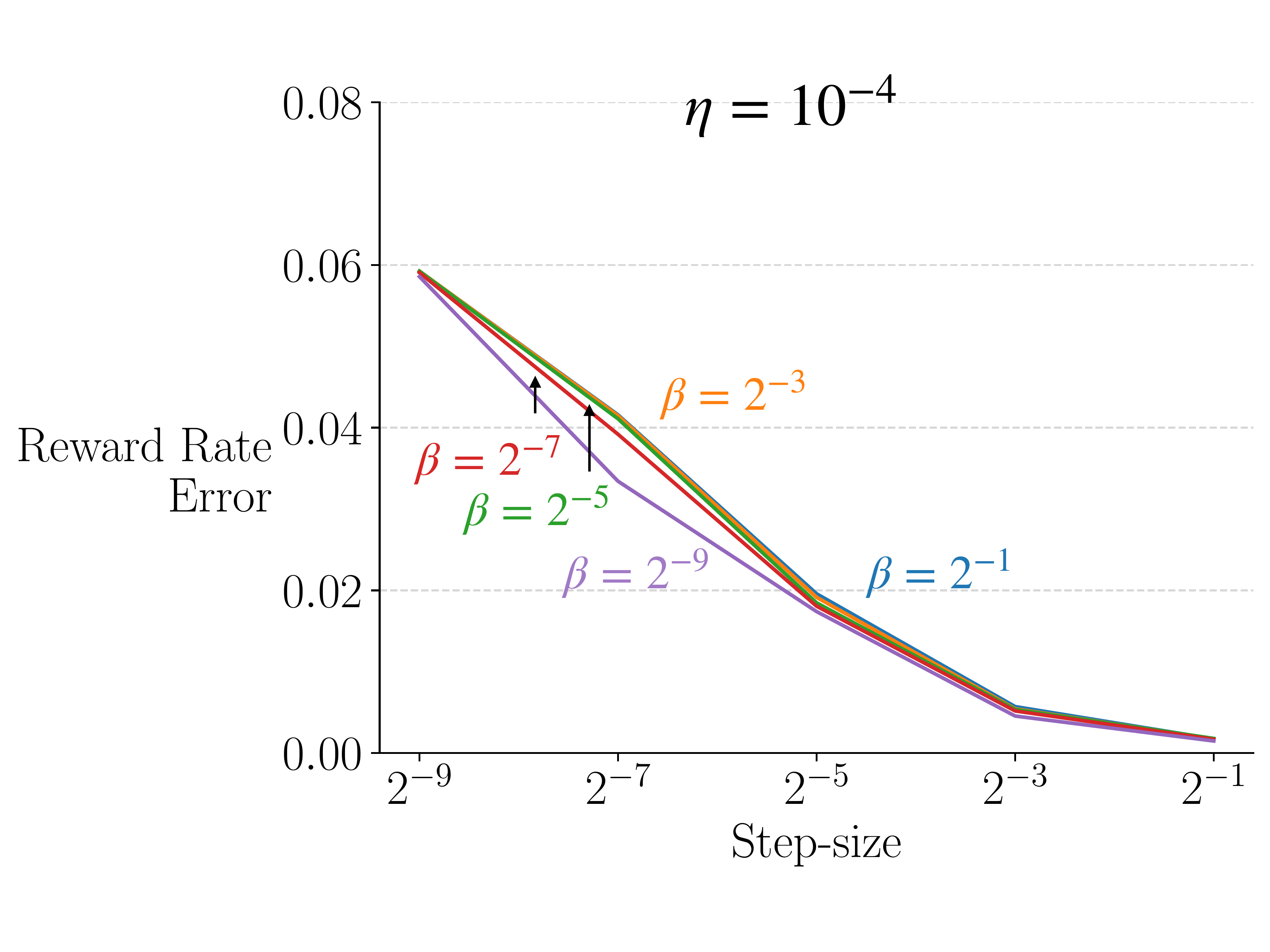}
\end{subfigure}%
\caption{
Plots showing parameter studies for inter-option Differential Q-evaluation in the continuing Four-Room domain when the goal was to go to \texttt{G1}. The algorithm used the set of primitive actions and the set of hallway options $\calO = \calA + \calH$. The y-axis is the absolute difference between the optimal reward rate 0.0625 and the estimated reward rate, averaged over all 200,000 steps.}
\label{fig: inter-option Q-evaluation parameter study}
\end{figure*}

\begin{figure*}[h!]
\centering
\begin{subfigure}{.5\textwidth}
    \centering
    \includegraphics[width=0.9\textwidth]{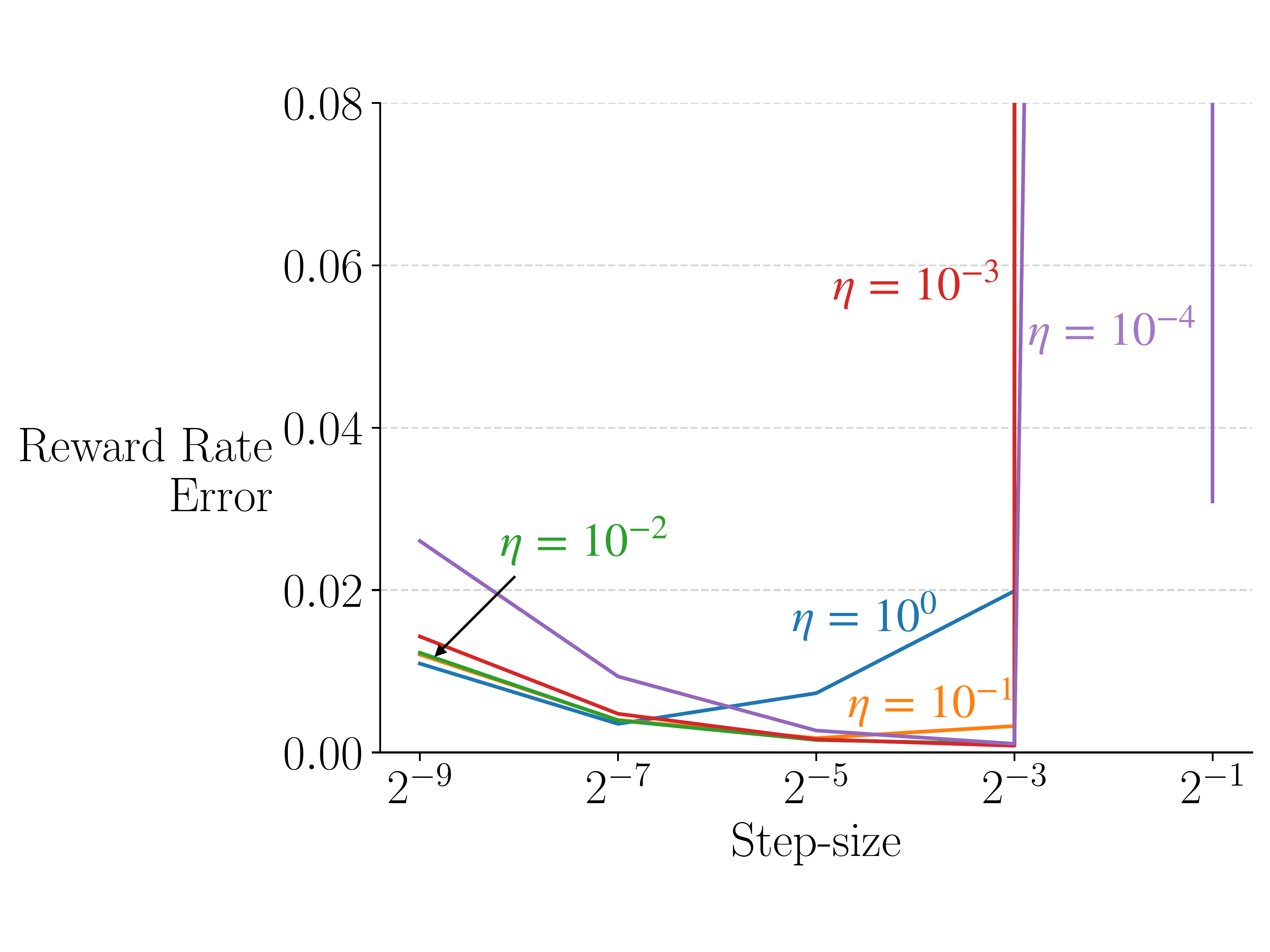}
\end{subfigure}%
\caption{
Plots showing parameter studies for intra-option Differential Q-evaluation in the continuing Four-Room domain when the goal was to go to \texttt{G1}. The setting is the same as the one used for intra-option Differential Q-evaluation.
}
\label{fig: intra-option Q-evaluation parameter study}
\end{figure*}

\section{Additional Discussion}\label{app: additional discussion}

\subsection{Two Failed Attempts on Extending Differential Q-learning to an Inter-option Algorithm}\label{app: additional discussion: Other Attempts on Extending Differential Q-learning}
The authors have tried two other ways of extending Differential Q-learning to an Inter-option Algorithm (cf.\ \cref{sec: inter-option}). While these two ways appear to work properly at the first glance, they do not actually. We now show these two approaches and explain why they do not work properly.

The first extension uses, for each option, the average-reward rate per-step instead of the total reward as the reward of the option. In particular, such an extension use update rules \eqref{eq: possible extension of Diff Q} and \eqref{eq: possible extension of Diff bar R}, but with TD error defined as:
\begin{align} \label{eq: Inter-option Differential Q-learning wrong TD error}
    \delta_n' & \doteq \hat R_n / \hat L_n - \bar R_n + \max_o Q_n (\hat S_{n+1}, o) - Q_n(\hat S_{n}, \hat O_n)
\end{align}
Unfortunately, such an extension can not guarantee convergence to a desired point. Specifically, the extension, if converges, will converge to a solution of $\mathbb{E}[\delta_n'] = 0$, which is not necessarily a solution of the Bellman equation $\mathbb{E}[\delta_n] = 0$ (Equation \ref{eq: SMDP Bellman optimality equation}). 

An alternative approach to avoid the instability issue is to shrink the entire update, not the option's cumulative reward, by the sample length:
\begin{align}
    Q_{n+1}(\hat S_n, \hat O_n) & \doteq Q_{n}(\hat S_n, \hat O_n) + \alpha_{n} \delta_{n} / \hat L_n, \label{eq: possible extension of Diff Q 2}\\
    \bar R_{n+1} & \doteq \bar R_{n} + \eta \alpha_{n} \delta_{n} / \hat L_n. \label{eq: possible extension of Diff bar R 2}
\end{align}
Still, the above two updates can not guarantee convergence to the desired values because, again, $\bbE[\delta_n / \hat L_n] = 0$ does not imply that the Bellman equation $\mathbb{E}[\delta_n] = 0$ is satisfied. 


\subsection{Pseudocodes}
\label{app: additional discussion: pseudocodes}

\begin{algorithm}[h]
\DontPrintSemicolon
\SetAlgoLined
\KwIn{Behavioral policy $\mu_b$'s parameters (e.g., $\epsilon$ for $\epsilon$-greedy)}
\SetKwInput{AP}{Algorithm parameters}
\SetKwRepeat{Do}{do}{while}
\AP{step-size parameters $\alpha, \eta, \beta$}
Initialize $Q(s,o)\ \forall\,s\in\calS,o\in\calO, \bar{R}$ arbitrarily (e.g., to zero); $L(s,o) \gets 1\ \forall\ s\in\calS,o\in\calO$ \\
Obtain initial $S$ \\
 \While{still time to train}
 {
    Initialize $\hat{L}\gets 0, \hat R \gets 0, S_{tmp} \gets S$ \\
    $O \gets$ option sampled from $\mu_b(\cdot \mid S)$ \\
    \Do{$O$ doesn't terminate in $S'$}
    {
        Sample primitive action $A \sim \pi(\cdot \mid S,O)$ \\ 
        Take action $A$, observe $R, S'$ \\
        $\hat{L} \gets \hat{L} + 1$ \\
        $\hat R \gets \hat R + R$ \\
        $S \gets S'$ \\
    }
    $S \gets  S_{tmp}$ \\
    $L(S,O) \gets  L(S,O) + \beta \big( \hat{L} - L(S,O) \big)$ \\
    $\delta \gets  \hat R - \bar{R} \cdot L(S,O) + \max_{o}Q(S',o) - Q(S,O)$ \\
    $Q(S,O) \gets  Q(S,O) + \alpha \delta / L(S,O)$ \\
    $\bar{R} \gets  \bar{R} + \eta \alpha \delta / L(S,O)$ \\
    $S \gets S'$ \\
 }
 return $Q$
 \caption{Inter-option Differential Q-learning}
 \label{algo:inter-option-diff-q}
\end{algorithm}

\begin{algorithm}[h]
\DontPrintSemicolon
\SetAlgoLined
\KwIn{Behavioral policy $\mu_b$, target policy $\mu$}
\SetKwInput{AP}{Algorithm parameters}
\SetKwRepeat{Do}{do}{while}
\AP{step-size parameters $\alpha, \eta, \beta$}
Initialize $Q(s,o)\ \forall\,s\in\calS,o\in\calO, \bar{R}$ arbitrarily (e.g., to zero); $L(s,o) \gets 1\ \forall\ s\in\calS,o\in\calO$ \\
Obtain initial $S$ \\
 \While{still time to train}
 {
    Initialize $\hat{L}\gets 0, \hat R \gets 0, S_{tmp} \gets S$ \\
    $O \gets$ option sampled from $\mu_b(\cdot \mid S)$ \\
    \Do{$O$ doesn't terminate in $S'$}
    {
        Sample primitive action $A \sim \pi(\cdot \mid S,O)$ \\ 
        Take action $A$, observe $R, S'$ \\
        $\hat{L} \gets \hat{L} + 1$ \\
        $\hat R \gets \hat R + R$ \\
        $S \gets S'$ \\
    }
    $S \gets  S_{tmp}$ \\
    $L(S,O) \gets  L(S,O) + \beta \big( \hat{L} - L(S,O) \big)$ \\
    $\delta \gets  \hat R - \bar{R} \cdot L(S,O) + \sum_o \mu(o \mid S') Q(S',o) - Q(S,O)$ \\
    $Q(S,O) \gets  Q(S,O) + \alpha \delta / L(S,O)$ \\
    $\bar{R} \gets  \bar{R} + \eta \alpha \delta / L(S,O)$ \\
    $S \gets S'$ \\
 }
 return $Q$
 \caption{Inter-option Differential Q-evaluation (learning)}
 \label{algo:inter-option-diff-q}
\end{algorithm}

\begin{algorithm}[h]
\DontPrintSemicolon
\SetAlgoLined
\KwIn{Behavioral policy $\mu_b$'s parameters (e.g., $\epsilon$ for $\epsilon$-greedy)}
\SetKwInput{AP}{Algorithm parameters}
\SetKwRepeat{Do}{do}{while}
\AP{step-size parameters $\alpha, \eta$}
Initialize $Q(s,o)\ \forall\,s\in\calS,o\in\calO, \bar{R}$ arbitrarily (e.g., to zero) \\
Obtain initial $S$ \\
 \While{still time to train}
 {
    $O \gets$ option sampled from $\mu_b(\cdot \mid S)$ \\
    \Do{$O$ doesn't terminate in $S$}
    {
        Sample primitive action $A \sim \pi(\cdot \mid S,O)$ \\ 
        Take action $A$, observe $R, S'$ \\
        $\Delta = 0$\\
        \For{all options $o$}
        {
            $\rho \gets \pi(A \mid S,o) / \pi(A \mid S,O)$ \\
            $\delta \gets  R - \bar{R} + \Big( \big( 1-\beta(S',o) \big) Q(S',o) + \beta(S',o) \max_{o'} Q(S',o') \Big) - Q(S,o)$ \\
            $Q(S,o) \gets Q(S,o) + \alpha \rho \delta$ \\
            $\Delta \gets \Delta + \eta \alpha \rho \delta$
        }
        $\bar{R} \gets \bar{R} + \Delta$ \\
        $S \gets S'$ \\
    }
 }
 return $Q$
 \caption{Intra-option Differential Q-learning}
 \label{algo:intra-option-diff-q}
\end{algorithm}

\begin{algorithm}[h]
\DontPrintSemicolon
\SetAlgoLined
\KwIn{Behavioral policy $\mu_b$'s parameters (e.g., $\epsilon$ for $\epsilon$-greedy)}
\SetKwInput{AP}{Algorithm parameters}
\SetKwRepeat{Do}{do}{while}
\AP{step-size parameters $\alpha, \eta$}
Initialize $Q(s,o)\ \forall\,s\in\calS,o\in\calO, \bar{R}$ arbitrarily (e.g., to zero) \\
Obtain initial $S$ \\
$O \gets$ option sampled from $\mu_b(\cdot|S)$ \\
 \While{still time to train}
 {
    {
        \If{$O \notin \argmax Q(S, \cdot)$ }
        {
            $O \gets$ option sampled from $\mu_b(\cdot|S)$ \\
        }
        Sample primitive action $A \sim \pi(\cdot|S,O)$ \\ 
        Take action $A$, observe $R, S'$ \\
        $\Delta = 0$\\
        \For{all options $o$}
        {
            $\rho \gets \pi(A|S,o) / \pi(A|S,O)$ \\
            $\delta \gets R - \bar{R} + \Big( \big( 1-\beta(S',o) \big) Q(S',o) + \beta(S',o) \max_{o'} Q(S',o') \Big) - Q(S,o)$ \\
            $Q(S,o) \gets Q(S,o) + \alpha \rho \delta$ \\
            $\Delta \gets \Delta + \eta \alpha \rho \delta$
        }
        $\bar{R} \gets \bar{R} + \Delta$ \\
    }
    $S = S'$ \\
 }
 return $Q$
 \caption{Intra-option Differential Q-learning with interruption}
 \label{algo:intra-option-diff-q}
\end{algorithm}

\begin{algorithm}[h]
\DontPrintSemicolon
\SetAlgoLined
\KwIn{Behavioral policy $\mu_b$, target policy $\mu$}
\SetKwInput{AP}{Algorithm parameters}
\SetKwRepeat{Do}{do}{while}
\AP{step-size parameters $\alpha, \eta$}
Initialize $Q(s,o)\ \forall\,s\in\calS,o\in\calO, \bar{R}$ arbitrarily (e.g., to zero) \\
Obtain initial $S$ \\
 \While{still time to train}
 {
    $O \gets$ option sampled from $\mu_b(\cdot \mid S)$ \\
    \Do{$O$ doesn't terminate in $S$}
    {
        Sample primitive action $A \sim \pi(\cdot \mid S,O)$ \\ 
        Take action $A$, observe $R, S'$ \\
        $\Delta = 0$\\
        \For{all options $o$}
        {
            $\rho \gets \pi(A \mid S,o) / \pi(A \mid S,O)$ \\
            $\delta \gets  R - \bar{R} + \Big( \big( 1-\beta(S',o) \big) Q(S',o) + \beta(S',o) \sum_{o'} \mu(o' \mid S') Q(S',o') \Big) - Q(S,o)$ \\
            $Q(S,o) \gets Q(S,o) + \alpha \rho \delta$ \\
            $\Delta \gets \Delta + \eta \alpha \rho \delta$
        }
        $\bar{R} \gets \bar{R} + \Delta$ \\
        $S \gets S'$ \\
    }
 }
 return $Q$
 \caption{Intra-option Differential Q-evaluation (learning)}
 \label{algo:intra-option-diff-q}
\end{algorithm}

\begin{algorithm}[h]
\DontPrintSemicolon
\SetAlgoLined
\KwIn{Behavioral policy $\mu_b$'s parameters (e.g., $\epsilon$ for $\epsilon$-greedy)}
\SetKwInput{AP}{Algorithm parameters}
\SetKwRepeat{Do}{do}{while}
\AP{step-size parameters $\alpha, \beta, \eta$; number of planning steps per time step $n$}
Initialize $Q(s,o), P(x\mid s, o), R(s, o)\ \forall\,s,x\in\calS,o\in\calO, \bar{R}$, arbitrarily (e.g., to zero); $L(s,o)=1\ \forall\ s\in\calS,o\in\calO$; $T \gets$ False\\
 \While{still time to train}
 {
    $S \gets$ current state\\
    $O \gets$ option sampled from $\mu_b(\cdot \mid S)$\\
    \While{$T$ is False}
    {
        Sample primitive action $A \sim \pi(\cdot \mid S,O)$ \\ 
        Take action $A$, observe $R', S'$ \\
        \For{all options $o$ such that $\pi(A\mid S, o) > 0$}
        {
            $\rho \gets \pi(A \mid S,o) / \pi(A \mid S,O)$ \\
            \For{all states $x \in \calS$}
            {
                $P(x \mid S, o) \gets P(x \mid S, o) + \beta \rho \Big( \beta(S', o) \bbI(S'=x) + \big( 1 - \beta(S', o) \big) P(x \mid S', o) - P(x \mid S, o) \Big)$
            }
            $R(S, o) \gets R(S, o) + \beta \rho \Big( R' + \big( 1 - \beta(S', o) \big) R(S', o) - R(S, o) \Big) $\\
            $L(S, o) \gets L(S, o) + \beta \rho \Big( 1 + \big( 1 - \beta(S', o) \big) L(S', o) - L(S, o) \Big) $
        }
        $T \gets$ indicator of termination sampled from $\beta(S', O)$\\
        \For{all of the $n$ planning steps}{
            $S \gets$ a random previously observed state\\
            $O \gets$ a random option previously taken in $S$\\
            $S' \gets$ a sampled state from $P(\cdot \mid S, O)$ \\
            $\delta \gets R(S, O) - L(S, O)\bar{R} + \max_{o} Q(S',o) - Q(S,O)$ \\
            $Q(S,O) \gets Q(S,O) + \alpha \rho \delta/L(S, O)$ \\
            $\bar{R} \gets \bar{R} + \eta \alpha \rho \delta/L(S, O)$ \\
        }
    }
 }
 return $Q$
 \caption{Combined Algorithm: Intra-option Model-learning + Inter-option Q-planning}
 \label{algo:inter-option-diff-q}
\end{algorithm}

\end{document}